\documentclass[journal]{IEEEtran}

\usepackage[T1]{fontenc}
\usepackage[noend]{algpseudocode}
\usepackage[ruled, vlined, linesnumbered]{algorithm2e}
\usepackage{amsmath}
\usepackage{amssymb}
\usepackage{amsfonts}
\usepackage{amsthm}
\usepackage[ruled, vlined, linesnumbered]{algorithm2e}
\usepackage[caption=false,font=footnotesize]{subfig}
\usepackage {graphicx}
\usepackage{color}
\usepackage{enumitem}

\newtheorem{theorem}{\bf Theorem}

\newtheorem{remark}{\bf Remark}

\newtheorem{proposition}{\bf Proposition}
\newtheorem{lemma}{\bf Lemma}
\newtheorem{assumption}{\bf Assumption}

\newcommand{\bx}{{\bf x}}
\newcommand{\by}{{\bf y}}
\newcommand{\bv}{{\bf v}}
\newcommand{\bn}{{\bf n}}
\newcommand{\bm}{{\bf m}}

\newcommand{\bw}{{\bf w}}
\newcommand{\bz}{{\bf z}}

\newcommand{\mc}{\mathcal}

\usepackage{bm}
\usepackage{amsmath}
\DeclareMathOperator*{\argmin}{argmin}
\algnewcommand\And{\textbf{and} }

\usepackage{hyperref}
\usepackage{cite}

\ifCLASSINFOpdf

\else

\fi

\usepackage{amsmath}
\usepackage[cmintegrals]{newtxmath}
\usepackage{eso-pic}
\begin{document}
\AddToShipoutPictureBG*{%
  \AtPageUpperLeft{%
    \setlength\unitlength{1in}%
    \hspace*{\dimexpr0.5\paperwidth\relax}
    \makebox(0,-0.50)[c]{\normalsize This paper has been accepted for publication in the IEEE Transactions on Robotics.} 
}}
\AddToShipoutPictureBG*{%
  \AtPageUpperLeft{%
    \setlength\unitlength{1in}%
    \hspace*{\dimexpr0.5\paperwidth\relax}
    \makebox(0,-0.80)[c]{\normalsize Please cite the paper as: A.R. Pedram, R. Funada, and T. Tanaka,} %
}}
\AddToShipoutPictureBG*{%
  \AtPageUpperLeft{%
    \setlength\unitlength{1in}%
    \hspace*{\dimexpr0.5\paperwidth\relax}
    \makebox(0,-1.10)[c]{\normalsize ``Gaussian Belief Space Path Planning for Minimum Sensing Navigation", IEEE Transactions on Robotics (T-RO), 2023.} %
}}

\title{Gaussian Belief Space Path Planning for \\ Minimum Sensing Navigation}

\author{Ali~Reza~Pedram$^{1}$ \qquad 
        Riku~Funada$^{2}$ \qquad
        Takashi~Tanaka$^{3}$ 

\thanks{©2022 IEEE. Personal use of this material is permitted. Permission from IEEE must be obtained for all other uses, in any current or future media, including reprinting/republishing this material for advertising or promotional purposes, creating new collective works, for resale or redistribution to servers or lists, or reuse of any copyrighted component of this work in other works.}
\thanks{*This work is supported by Lockheed Martin Corporation, FOA-AFRL-AFOSR-2019-0003, and JSPS KAKENHI GrandNumber 21K20425.}
\thanks{$^{1}$ Walker Department of Mechanical Engineering, University of Texas at Austin. {\tt\small apedram@utexas.edu}.
        $^2$ Department of Systems and Control, Tokyo Institute of Technology. {\tt\small funada@sc.e.titech.ac.jp}.
         $^{3}$Department of Aerospace Engineering and Engineering Mechanics, University of Texas at Austin.
        {\tt\small ttanaka@utexas.edu}. }%
}

\maketitle

\begin{abstract}
We propose a path planning methodology for a mobile robot navigating through an obstacle-filled environment to generate a reference path that is traceable with moderate sensing efforts.
The desired reference path is characterized as the shortest path in an obstacle-filled Gaussian belief manifold equipped with a novel information-geometric distance function. The distance function we introduce is shown to be an asymmetric quasi-pseudometric and can be interpreted as the minimum information gain required to steer the Gaussian belief.
An RRT*-based numerical solution algorithm is presented to solve the formulated shortest-path problem.  To gain insight into the asymptotic optimality of the proposed algorithm, we show that the considered path length function is continuous with respect to the topology of total variation. 
Simulation results demonstrate that the proposed method is effective in various robot navigation scenarios to reduce sensing costs, such as the required frequency of sensor measurements and the number of sensors that must be operated simultaneously.
\end{abstract}


\IEEEpeerreviewmaketitle

\section{Introduction}
\IEEEPARstart{I}{n} robot motion planning, the reference path generation is often performed independently of the feedback control design for path following.
While such a two-stage procedure leads to a suboptimal policy in general, it simplifies the problem to be solved and the resulting performance loss is acceptable in many applications. 
The two-stage procedure also benefits from powerful geometry-based trajectory generation algorithms (e.g., A* \cite{hart1968formal}, PRM \cite{kavraki1996probabilistic}, RRT \cite{lavalle2001rapidly}), as other factors such as dynamic constraints, stochasticity and uncertainty can often be resolved in the control design stage.

Motion planning is more challenging if the robot's configuration is only partially observable through noisy measurements. 
A common approach to such problems is via the \emph{belief state formalism} \cite{platt2010belief,van2011lqg}. In this approach,  the Bayesian estimate (i.e., a probability distribution) of the robot’s state is considered as a new state, called the \emph{belief state}, whereby the original stochastic optimal control problem with a partially-observable state is converted into an equivalent stochastic optimal control problem with a fully observable state. The belief state formalism makes the aforementioned two-stage motion planning strategy applicable in a similar manner, except that both path generation and tracking are performed in the space of belief states (\emph{belief space}). This approach is powerful especially if the belief space is  representable by a small number of parameters (e.g., Gaussian beliefs, which can be parametrized by mean and covariance only).

We consider the problem of generating a reference path in the Gaussian belief space such that the path length with respect to a particular quasi-pseudometric on the belief manifold is minimized. The quasi-pseudometric we choose is interpreted as the weighted sum of the Euclidean travel distance and the information gain required to steer the belief state. Solving the shortest path problem  therefore means finding a joint sensing and control strategy for a robot to move from a given initial Gaussian belief to a target Gaussian belief while minimizing the weighted sum of the travel distance and the cost of sensing. 

\begin{figure}[t]
\vspace{-0.2cm}
\centering
\includegraphics[trim = 0cm 0cm 0cm 0cm, clip=true, width = 0.65\columnwidth]{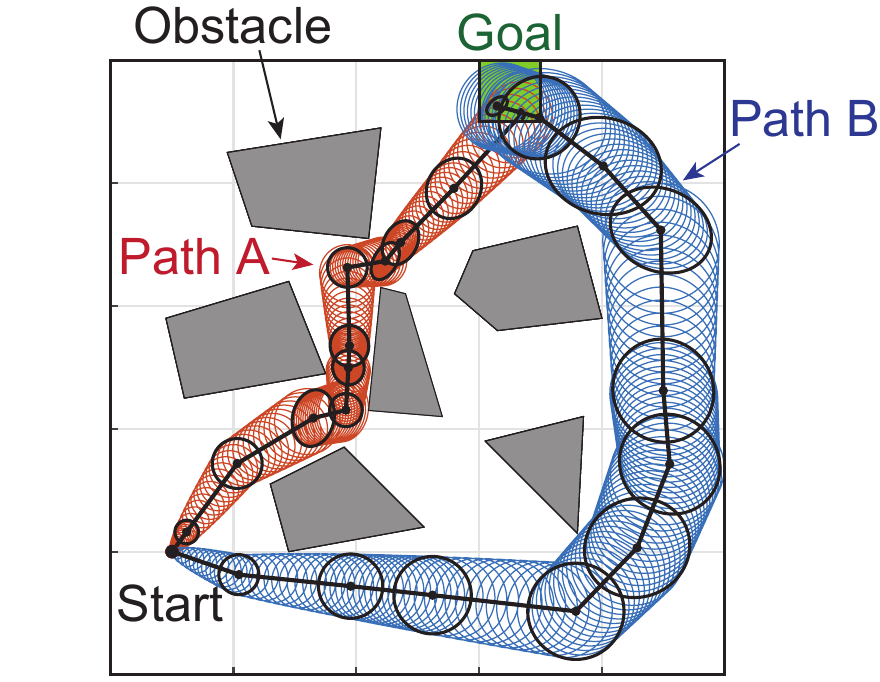}
\caption{Simulation results of the proposed algorithm. Path A prioritizes to minimize the Euclidean travel distance, while path B prioritize to reduce the information gain required to follow the path. 
}
\label{fig:intro_example}
\end{figure}
\vspace{-0.4cm}
\subsection{Motivation}
The shortest path problem we formulate is motivated by the increasing need for simultaneous perception and action planning in  modern, information-rich autonomy.
Due to the wide availability of low-cost and high-performance sensing devices, obtaining a large amount of sensor data has become easier in many applications. Nevertheless, operating a sensor at its full capacity may not be the best strategy for resource-constrained robots, especially if it drains the robot's scarce power or computational resources with little benefit. 
As sensor modalities increase, how to achieve a given task with minimum perceptual resources (e.g., with reduced sensing frequencies or sensor gains) becomes an increasingly relevant question. 
For instance, planetary rovers need to estimate the wheel slippage using visual odometry (VO) when traversing  harsh and unknown terrains \cite{kilic2021slip}.  However, using VO reduces the navigation speed  as the rover needs to stop frequently to capture images \cite{toupet2020terrain} and drive slowly due to its limited computational capability \cite{li2008characterization}. Particularly, the Mars Science Laboratory rover reaches a maximum speed of $140$\,m/h in blind-drive mode (no VO update) and $45$\,m/h in hazard avoidance mode (VO  update every $10$ meters) \cite{grotzinger2012mars}. Another example is the vision-based navigation of micro aerial vehicles (MAV), where the visual data are sent to a ground station (e.g., see \cite{blosch2010vision}) or are processed by computationally-constrained onboard processors (e.g., see \cite{shen2011autonomous}). In navigation of MAVs, the available resources (computation speed, memory, power, and communication bandwidth) are limited and require specific consideration. These examples showcase the importance of perception effort management in autonomous navigation.

In navigation tasks, required sensing effort critically depends on the geometry of the planned paths. For example, Path A in Fig.~\ref{fig:intro_example} offers a shorter travel distance; however, the sensing effort required to trace it is high as the robot's locational uncertainty needs to be kept small. Depending on the cost of perception, taking a longer path (such as Path B) that is traceable with less sensing cost may be preferable. Thus, we aim to develop a path planning methodology that allows a ``minimum sensing'' navigation that can flexibly comply with the robot’s perceptual resource constraint. 
\vspace{-0.4cm}
\subsection{Related Work}
This subsection provides a non-exhaustive list of related works categorized from the perspective of 1) belief space planning, 2) chance-constrained path planning, 3) information theory in path planning, and 4) controlled sensing.  
\subsubsection{Belief space planning} Belief space path planning for uncertain systems  \cite{lavalle2006planning} has been studied in various forms in the literature.
The work \cite{alterovitz2007stochastic} studied path planning for systems with uncertain dynamics within a fully observable and geometrically known environment. 
The work \cite{agha2014firm} generalized the results of \cite{alterovitz2007stochastic} by incorporating the sensing uncertainty and presented a feedback-based information roadmap. 
The belief-space probabilistic roadmap (BRM) is presented in \cite{prentice2009belief}, wherein a factored form of the covariance matrix is used, leading to efficient posterior belief predictions.
The work \cite{roy1999coastal}
presented the advantage of a path plan (the \emph{coastal navigation} strategy) that best assists the robot's perception during the navigation. The framework of safe path-planning \cite{lambert2003safe, pepy2006safe} is also 
established for path planning in the belief space to provide a planned path with a safety guarantee.

The probability distribution of closed-loop trajectories under linear feedback policies has been characterized in \cite{bry2011rapidly, van2011lqg} which allows the evaluation of the probability of collision with obstacles. The work \cite{van2011lqg} uses the probability of collision to  search for ``safe'' trajectories among the ones generated by RRT.  Instead, \cite{bry2011rapidly} incrementally constructs a graph of ``safe'' trajectories with the aid of RRT*. 

In \cite{van2012motion, van2017motion}, belief space path planning for robots with imperfect state information is studied. 
In \cite{van2012motion}, a belief space variant of stochastic dynamic programming is introduced to find the optimal belief trajectory. 
Alternatively, \cite{van2017motion} used the belief state iterative LQG 
which is shown to have a lower computation complexity and a better numerical stability compared with \cite{van2012motion}. The authors of \cite{sun2016stochastic}
proposed the stochastic extended linear quadratic regulator for motion planning in Gaussian belief space, which simultaneously computes the optimal trajectory and the associated linear control policy for following that trajectory.  Belief space path planning in environments with discontinuities in sensing domains is studied in \cite{patil2014gaussian}.

In addition to generating a nominal path, many of the aforementioned works (e.g., \cite{van2012motion,agha2014firm, van2017motion, sun2016stochastic}) provide local controllers to stabilize the system around the nominal path. These controllers eliminate the need for extensive replanning during the execution. A belief space path-following algorithm is also considered in \cite{platt2010belief},
where the nonlinear stochastic dynamics of the belief state is linearized, and a local LQR controller is used. In \cite{platt2010belief}, a replanning strategy is also proposed to update the reference trajectory when divergence from the planned trajectory is too large to be handled by the LQR controller.

\subsubsection{Chance-constrained path planning} Belief space path planning is closely related to the large body of literature on chance-constrained (CC) path planning \cite{blackmore2006probabilistic,  blackmore2011chance,vitus2011closed}, where the focus is on the development of path planning algorithms under probabilistic safety constraints (see e.g., \cite{ono2008iterative, jasour2015semidefinite}). Basic CC methods for linear-Gaussian systems have been extended to address non-linear and non-Gaussian problems \cite{blackmore2010probabilistic, wang2020non} and to handle the joint chance constraints \cite{ono2015chance}.

While CC formulations often suffer from poor scalability \cite{aoude2013probabilistically}, a method to overcome such difficulties is proposed in \cite{dai2019chance}. The computational complexity of CC algorithms is studied in \cite{da2019collision}.  Alternatively,  \cite{luders2010chance} proposed a sampling-based method called CC-RRT method that allows efficient computation of feasible paths. The CC-RRT algorithm is  generalized in \cite{du2011robot, kothari2013probabilistically}  for CC path planning in dynamic environments.
The works \cite{luders2013robust, liu2014incremental} introduced several variants of CC-RRT* algorithm where the convergence to the optimal trajectory is guaranteed.

\subsubsection{Information theory (IT) in path planning}  Information-theoretic concepts have been utilized in path planning problems by several prior works, albeit differently from our approach in this paper. The work \cite{he2008planning} proposed a modification for BRM \cite{prentice2009belief} in which the states where informative measurements are available (i.e., where a large entropy reduction is expected) are sampled more frequently. Informative path planning is investigated in \cite{levine2013information}, where the sensing agents seek to both maximize the information gathered about the target position (or the environment), quantified by Fisher information matrix, and minimize the cost of traversing to its goal state. 
In \cite{folsom2021scalable}, information-theoretic path planning is studied, where a Mars helicopter uses RRT*-IT to explore and obtain information about the surface of Mars, expressed in terms of the reduction in the standard deviation of the belief terrain type distribution, in a shortest time. 
IT is also used to reduce the computational complexity of the path planning in \cite{larsson2020information}, where it is suggested to obtain abstractions of the search space by the aid of IT and perform the path planning in abstracted representation of the search space.

\subsubsection{Controlled sensing} The works mentioned above assume either no or fixed sensor modalities. In modern autonomy, where variable sensor modality is available, it is becoming increasingly meaningful to model the strategic sensing aspect explicitly in the problem formulation. The partially observable Markov decision process (POMDP) framework is widely used for controlled sensing design \cite{krishnamurthy2016partially}.  
The work \cite{carlone2018attention} proposed a greedy algorithm for strategic sensing in vision-based navigation, where during the path execution, the robot chooses only a small number of landmarks that are most relevant to its task.
The authors of \cite{tzoumas2020lqg} incorporated a restriction on sensing budget  into the optimal control problem,  and proposed an algorithm for control and sensing co-design. The optimal estimation through a network of sensors operated under sensing constraints is explored in \cite{hashemi2020randomized}. 
\vspace{-0.3cm}
\subsection{Proposed Approach}
\label{sec:proposed_approach}

Noticing that belief-space path planning and strategic sensing are inseparable problems, we propose to perform perception planning and motion planning simultaneously. To this end, we propose a two-stage procedure similar to the one discussed above: in the first stage, a reference path in the belief space is generated by an off-line path planner, and in the second stage, the reference path is followed by operating sensors and actuators in real-time. 

In the first stage, we solve a shortest path problem in the belief space. In this shortest path problem, we use a non-Euclidean information-geometric distance function which captures not only the travel distance but also the expected perception cost required to follow the path. Specifically, we identify the cost of transitioning from one Gaussian belief state to another with a weighted sum of the Euclidean distance between their means and the information gain (i.e., the entropy reduction) required to update the belief covariance. The ``distance'' notion introduced this way defines a quasi-pseudometric on the belief manifold, making the shortest path problem on the Gaussian belief space well-defined. We then apply an RRT*-based path planner \cite{karaman2011sampling} to find the shortest path numerically.

Note that we do not incorporate any sensor models nor actuator models in the first stage. Although the idea of estimating perception costs without assuming sensor models may sound unconventional, the aforementioned information-geometric distance is chosen to provide a reasonable guess of the actual perception cost regardless of the types of the sensors used in real-time navigation. The proposed “model-free'' approach is beneficial in applications where sensor models are not available or cumbersome to obtain. This benefit can be understood by invoking why ``model-free'' path planners such as A*, RPM, RRT are popular, even though they do not incorporate the robot’s dynamic constraints (e.g., actuator models) to generate a reference path.

Since the belief path generated in the first stage only serves as a reference trajectory to be tracked, an appropriate trajectory tracking algorithm is needed in the second stage. Notice that, in order to follow a belief path, both the actuators and the sensors must be operated in real-time by an appropriate joint sensing/feedback-control policy. Belief LQR \cite{platt2010belief} is a popular method for such an implementation. In this paper, we consider an event-based sensing strategy similar to \cite{miskowicz2018event} along with feedback controllers to achieve belief path tracking. 
\vspace{-0.1cm}
\subsection{Contribution}
The technical contributions of this paper are as follows:
\begin{itemize}
\item[(a)] 
We formulate  a shortest path problem in a Gaussian belief space with respect to a novel ``distance'' function  $\mathcal{D}$ (defined by \eqref{eq:def_D} below) that  characterizes the weighted sum of the Euclidean travel distance and the sensing cost required for  trajectory following assuming  simple mobile robot dynamics (equation \eqref{eq:euler3} below). 
We first show that $\mathcal{D}$ is a quasi-pseudometric, i.e., $\mathcal{D}$ satisfies the triangle inequality (Theorem~\ref{theo:tria}) but fails to satisfy symmetry and the identity of indiscernibles. We then introduce a path length concept using $\mathcal{D}$, for which the shortest path problem is formulated as \eqref{eq:main_problem}.
\item[(b)] We develop an RRT*-based algorithm for the shortest path problem described in part (a). Besides a basic version (Algorithm~\ref{algo:1}), we also develop a modified algorithm with improved computational efficiency (Algorithm~\ref{algo:2}). 
We also show how the sensing constraints can be incorporated into the developed algorithm.
\item[(c)] We prove that the path length function characterized in the problem formulation in part (a) is continuous with respect to the topology of total variation. This result is critical to prove the asymptotic optimality of the proposed RRT*-based algorithm, although a complete proof must be postponed as future work. 
\item[(d)] The practical usefulness of the proposed motion planning approach is demonstrated by simulation studies. We show that the algorithms in part (b) can be easily combined with the existing ideas for belief path following (e.g., belief LQR with event-based sensing or greedy sensor selection) to efficiently reduce sensing efforts (e.g., frequency of sensing or the number of sensors to be used simultaneously) for both the scenario when the robot's dynamics are close to \eqref{eq:euler3}, and  the scenario when the dynamics are significantly different from it. 
\end{itemize}

In our previous work \cite{pedram2021rationally}, we proposed the distance function $\mathcal{D}$ while the analysis on this metric was limited to 1-D spaces. In this paper, the analyses are extended to $N$-dimensional spaces. To establish these results, we provide novel proofs (summarized in Appendix \ref{ap:zero}, \ref{ap:B}, and \ref{ap:C}) which are significantly different from our previous work.  
We also develop two novel algorithms as described in (b). One of the major improvements is the introduction of the concept of \emph{losslessness}. This concept allows us to develop computationally efficient algorithms and to establish a continuity result in Theorem~\ref{theo:continuity2}. Finally, this paper provides comprehensive simulation results demonstrating the effectiveness of the proposed method for mitigating sensing costs during the path following phase, which was not discussed in the previous work.
\vspace{-0.2cm}
\subsection{Outline of the Paper}
The rest of the paper is organized as follows:
In Section~\ref{sec:prelim}, we summarize basic information-geometric concepts that are necessary to formally state the shortest path problem in Section~\ref{sec:formulation}. 
In Section~\ref{sec:algorithm}, we present the proposed RRT*-based algorithms and prove the continuity of the path length function to shed light on its asymptotic optimality. 
Section~\ref{sec:simulation} presents
simulation studies demonstrating the effectiveness of the proposed planning strategy.
We conclude with a list of future work in Section~\ref{sec:conclusion}.

\subsection{Notation and Convention}
\label{sec:notation}
Vectors and matrices are represented by lower-case and upper-case symbols, respectively. Random variables are denoted by bold symbols such as $\mathbf{x}$. The following notation will be used: $\mathbb{S}^d=\big\{P \in \mathbb{R}^{d \times d} : \text{ $P=P^\top$} \big\}$, $\mathbb{S}^d_{++}=\big\{P \in \mathbb{S}^d : P \succ 0 \big\}$, and $\mathbb{S}^d_{\rho}=\big\{P \in \mathbb{S}^d : P \succeq \rho I \}$  for a real value $\rho > 0$. $\bar{\sigma}(M)$ and $\|M\|_F$ represent the maximum singular value and the Frobenius norm of the matrix $M$, respectively. The vector $2$-norm is denoted by $\|\cdot\|$. $\mathcal{N}(x, P)$ represents a Gaussian random variable with  mean $x$ and covariance of $P$. $\chi^2({\textup{Pr}})$ is the $\textup{Pr}$-th quantile of the Chi-squared distribution. For the simplicity, we use $\chi^2$ in the sequel.
\vspace{-0.6cm}
\section{Preliminaries}
\label{sec:prelim}  
In this section, we introduce an appropriate distance notion on a Gaussian belief space which will be needed to formulate a shortest path problem in Section~\ref{sec:formulation}. We also study the mathematical properties of the introduced distance notion.
\vspace{-0.6cm}
\subsection{Assumed Dynamics}
The distance concept we introduce can be interpreted as a navigation cost for a mobile robot whose location uncertainty (covariance matrix) grows linearly with the Euclidean travel distance when no sensor is used. 
Specifically, suppose that the reference trajectory for the robot is given as a sequence of way points $\{x_k\}_{k=0,1, ... , K}$ in the configuration space $\mathbb{R}^d$, and that the robot is commanded with a constant unit velocity input
\[
v_k := \frac{x_{k+1}-x_k}{\|x_{k+1}-x_k\|}
\]
to move from $x_k$ to $x_{k+1}$. Let $t_k$ be the time that the robot is scheduled to visit the $k$-th way point $x_k$, defined sequentially by $t_{k+1}-t_k=\|x_{k+1}-x_k\|$.
We assume that the actual robot motion is subject to stochastic disturbance.
Let $\bx(t_k)$ be the random vector representing the robot's actual position at time $t_k$. In an open-loop control scenario, it is assumed to satisfy
\begin{equation}
 \label{eq:euler3}
\bx(t_{k+1})=\bx(t_k)+(t_{k+1}-t_k)v_k+\bn_k
\end{equation}
where $\bn_k \sim \mathcal{N}(0, \|x_{k+1}-x_k\|W)$ is  a Gaussian disturbance whose covariance matrix is proportional to the commanded travel distance. In feedback control scenarios, the command input $v_k$ is allowed to be dependent on sensor measurements.

We emphasize that the simple dynamics \eqref{eq:euler3} are assumed solely for the purpose of introducing a distance notion on a Gaussian belief space in the sequel. The algorithm we develop in Section~\ref{sec:algorithm} can be used even if the actual robot dynamics are significantly different from \eqref{eq:euler3}.
This choice is similar to the fact that RRT* for Euclidean distance minimization is widely used even in applications where Euclidean distance in the configuration space does not capture the motion cost accurately.\footnote{However, we also note that there are many works that incorporate non-Euclidean metrics in RRT* to better approximate true motion costs.}
To follow this philosophy, we strategically adopt a simple model \eqref{eq:euler3} and leave more realistic dynamic constraints to be addressed in the path following control phase. 

\subsection{Gaussian Belief Space and Quasi-pseudometric}
\label{subsec:metric}
In Gaussian belief space planning, a reference trajectory is given as a sequence of belief way points $b_k=(x_k,P_k), k=0, 1, ... , K$, where $x_k\in\mathbb{R}^d$ and $P_k\in \mathbb{S}_{++}^d$ are planned mean and covariance of the random vector $\bx(t_k)$.
In the sequel, we call $\mathbb{B}:=\mathbb{R}^d \times \mathbb{S}_{++}^d$ the \emph{Gaussian belief space} or simply the \emph{belief space}.
We first introduce an appropriate directed distance function from a point $b_k=(x_k, P_k)$  to another $b_{k+1}=(x_{k+1}, P_{k+1})$. 
The distance function is interpreted as the cost of steering the Gaussian probability density characterized by $b_k$ to the one characterized by $b_{k+1}$. 
We assume that the distance function is a weighted sum of the travel cost $\mathcal{D}_{\text{travel}}(b_k, b_{k+1})$ and the information cost $\mathcal{D}_{\text{info}}(b_k, b_{k+1})$.

\subsubsection{Travel cost}
We assume that the travel cost is simply the commanded travel distance:
\[
\mathcal{D}_{\text{travel}}(b_k, b_{k+1}):=\|x_{k+1}-x_k\|.
\]

\subsubsection{Information cost}
Assuming that no sensor measurement is utilized while the deterministic control input $v_k$ is applied to \eqref{eq:euler3}, the covariance at time step $k+1$ is computed as
\begin{equation}
\label{eq:p_prior}
\hat{P}_{k+1}:=P_k+\|x_{k+1}-x_k\|W.
\end{equation}
We refer to $\hat{P}_{k+1}$ as the prior covariance at time step $k+1$.
Suppose that the prior covariance is updated to the posterior $P_{k+1}(\preceq \hat{P}_{k+1})$ by a sensor measurement $\by_{k+1}$ at time step $k+1$. (See Section~\ref{sec:formulation_interpret} for a discussion on sensing actions enabling this transition).
The minimum information gain required for this transition is given by the entropy reduction:
\begin{align}
\mathcal{D}_{\text{info}}(b_k, b_{k+1})&= h(\bx_{k+1}|\by_0, \cdots ,\by_k)-h(\bx_{k+1}|\by_0, \cdots, \by_{k+1}) \nonumber \\
&=\frac{1}{2}\log\det \hat{P}_{k+1} - \frac{1}{2}\log\det P_{k+1}. \label{eq:info_gain1}
\end{align} 
Here,  $h(\cdot|\cdot)$ denotes conditional differential entropy. Intuitively, $\mathcal{D}_{\text{info}}(b_k, b_{k+1})$ represents the minimum bits of information required to reduce the uncertainty from $\hat{P}_{k+1}$ to $ P_{k+1}$.

Note that for any physically ``meaningful'' belief update, the inequality $P_{k+1}\preceq \hat{P}_{k+1}$ should be satisfied, as the posterior uncertainty $P_{k+1}$ should be ``smaller'' than the prior  uncertainty $\hat{P}_{k+1}$. 
The posterior uncertainty is smaller because incorporating the information of a measurement $y_{k+1}$ never increases the uncertainty. In the sequel, we say that a transition from $b_k$ to $b_{k+1}$ is \emph{lossless} if the inequality $P_{k+1}\preceq \hat{P}_{k+1}$ is satisfied.
If the transition from $b_k$ to $b_{k+1}$ is lossless, the formula \eqref{eq:info_gain1} takes a non-negative value and hence it can be used in the definition of a (directed) distance from $b_k$ to $b_{k+1}$.
However, in order for the shortest path problem on a Gaussian belief space $\mathbb{B}$ to be well-defined, the distance function must be well-defined for arbitrary pairs $(b_k, b_{k+1})$. To generalize \eqref{eq:info_gain1} to pairs $(b_k, b_{k+1})$ that are not necessarily lossless, we adopt the following definition:
\begin{subequations}
\label{eq:d_info_general0}
\begin{align}
\mathcal{D}_{\text{info}}(b_k, b_{k+1})=\min_{Q_{k+1}\succeq 0} & \  \frac{1}{2}\log\det \hat{P}_{k+1}\!-\!\frac{1}{2}\log\det Q_{k+1} \label{eq:d_info_general}\\
\text{s.t.\ \ } &\quad Q_{k+1} \preceq P_{k+1}, \;\; Q_{k+1} \preceq \hat{P}_{k+1}.\label{eq:d_info_general1}
\end{align}
\end{subequations}
Notice that for any given pair $(P_k, P_{k+1})$, \eqref{eq:d_info_general0} takes a non-negative value, and \eqref{eq:d_info_general0} coincides with \eqref{eq:info_gain1} if  the transition from $b_k$ to $b_{k+1}$ is lossless.
To see why \eqref{eq:d_info_general0} is a natural generalization of \eqref{eq:info_gain1}, consider a two-step procedure $\hat{P}_{k+1}\rightarrow Q_{k+1}\rightarrow P_{k+1}$ to update the prior covariance $\hat{P}_{k+1}$ to the posterior covariance $P_{k+1}$. In the first step, the uncertainty is ``reduced'' from $\hat{P}_{k+1}$ to $Q_{k+1}(\preceq \hat{P}_{k+1})$. The associated information gain is $\frac{1}{2}\log\det \hat{P}_{k+1}-\frac{1}{2}\log\det Q_{k+1}$.
In the second step, the covariance $Q_{k+1}$ is ``increased'' to $P_{k+1}(\succeq Q_{k+1})$.
This step incurs no information cost, since the location uncertainty can be increased simply by ``forgetting'' the prior knowledge. The optimization problem \eqref{eq:d_info_general0} is interpreted as finding the optimal intermediate step $Q_{k+1}$ to minimize the information gain in the first step.

\begin{remark}
The expression \eqref{eq:d_info_general0} characterizes $\mathcal{D}_{\text{info}}(b_k, b_{k+1})$ as a value of convex program (more precisely, the max-det program \cite{vandenberghe1998determinant}).
Lemma~\ref{lemma:explicit} in Appendix~\ref{ap:zero} provides a method to solve \eqref{eq:d_info_general0} directly using the singular value decomposition.
\end{remark}

Since only lossless transitions are physically meaningful, the path planning algorithms we develop in Section~\ref{sec:algorithm} below are designed to produce a sequence of lossless transitions as an output. In Theorem~\ref{theo:loss-lessmod} in Section~\ref{sec:formulation}, we will formally prove that the optimal solution to the shortest path problem can be assumed lossless without loss of generality.

\subsubsection{Total cost}
The total cost to steer the belief state from $b_k=(x_k,P_k)$ to $b_{k+1}=(x_{k+1},P_{k+1})$
is a weighted sum of $\mathcal{D}_{\text{travel}}(b_k, b_{k+1})$ and $\mathcal{D}_{\text{info}}(b_k, b_{k+1})$.
Introducing $\alpha>0$, we define the total cost as
\begin{equation}
\label{eq:def_D}
\mathcal{D}(b_k, b_{k+1}):=  \; \mathcal{D}_{\text{travel}}(b_k, b_{k+1})+\alpha \mathcal{D}_{\text{info}}(b_k, b_{k+1}).
\end{equation}
Throughout this paper, the total cost function \eqref{eq:def_D} serves as a distance metric with which the lengths of the belief paths are measured. Before the shortest path problem is formally formulated in the next section, it is worthwhile to note the following key properties of the function \eqref{eq:def_D}:
\begin{enumerate}[label=(\roman*)]
    \item $\mathcal{D}(b_1, b_2)\geq 0 \;\; \forall b_1, b_2 \in \mathbb{B}$;
    \item $\mathcal{D}(b, b)= 0 \;\; \forall b \in \mathbb{B}$; and
    \item $\mathcal{D}(b_1, b_2)\leq \mathcal{D}(b_1, b_3)+\mathcal{D}(b_3, b_2) \;\; \forall b_1, b_2, b_3 \in \mathbb{B}$.
\end{enumerate}
The first two properties are straightforward to verify. The third property (the triangle inequality) has been shown in \cite{pedram2021rationally} for special cases with $d=1$. 
As the first technical result of this paper, we prove the triangle inequality in full generality as follows:
\begin{theorem}
\label{theo:tria}
In obstacle free space, the optimal path cost between $b_1=(x_1, P_1)$ and $b_2=(x_2,P_2)$ is equal to $\mathcal{D}(b_1, b_2)$ or equivalently
\begin{equation}
 \mathcal{D}(b_1, b_2) \leq \mathcal{D}(b_1, b_{int})+\mathcal{D}(b_{int}, b_2)   
\end{equation}
for any intermediate $b_{int}$.
\end{theorem}
\begin{proof}
See Appendix~\ref{ap:B}.
\end{proof}

Theorem~\ref{theo:tria} implies that the shortest path from $(x_1, P_1)$ to $(x_2, P_2)$ is obtained by first making a sensing-free travel from $(x_1, P_1)$ to $(x_2, \hat{P}_2)$ followed by a covariance reduction from $\hat{P}_2$ to $P_2$. In other words, the ``move-and-sense'' strategy is optimal for transitioning in an obstacle-free space.

It is also noteworthy that \eqref{eq:def_D} fails to satisfy symmetry, i.e., $\mathcal{D}(b_1, b_2)\neq \mathcal{D}(b_2, b_1)$ in general. 
Consequently, the notion of the path length we introduce below is direction-dependent. This nature will also be presented in a numarical example in Section~\ref{sec:asymmetry}.
The function \eqref{eq:def_D} also fails to satisfy the identity of indiscernibles since $\mathcal{D}(b_1, b_2)=0$ does not necessarily imply $b_1=b_2$. (Consider $b_1=(x_1, P_1)$ and $b_2=(x_2, P_2)$ with $x_1=x_2$ and $P_1 \preceq P_2$.) 
Due to the lack of symmetry and the identity of indiscernibles, the function \eqref{eq:def_D} fails to be a metric. 
However, with properties (i)-(iii) above, $\mathcal{D}$ is a \emph{quasi-pseudometric} on the belief space $\mathbb{B}$.
A shortest path problem on $\mathbb{B}$ is then well-defined with respect to $\mathcal{D}$, as we discuss in the next section.

Finally, from the perspective of information geometry, the distance function \eqref{eq:def_D} is one among many other alternative choices like the KL-divergence, which is used in the covariance steering problems in  \cite{chen2015optimal, okamoto2018optimal} and references therein. Unfortunately, using KL-divergence from $b_k$ to $b_{k+1}$ as the total cost $\mathcal{D}(b_k, b_{k+1})$ does not lead to a well-defined shortest path problem, since KL-divergence does not satisfy
the triangle inequality even in the space of Gaussian beliefs\cite{zhang2021properties}. 
Other functions including the Fisher-Rao metric and the square root of the Jensen-Shannon divergence are not attractive candidates for $\mathcal{D}(b_k, b_{k+1})$ due to the lack of closed-form expressions for Gaussian distributions\cite{pinele2020fisher, nielsen2020generalization}.
However, one may consider using an alternative information cost function in place of $\mathcal{D}_{\text{info}}$ in \eqref{eq:def_D}.
For instance, if we adopt the Wasserstein distance or the Hellinger distance \cite{pardo2018statistical}, the information cost can be defined as
\begin{align}
\mathcal{D}_{\text{info,W}}(b_k,\! b_{k+1}\!)
\!=\!\big[{\rm{Tr}}(\hat{P}_{k+1}\!+\!P_{k+1}\!-\!2(P_{k+1}^{{1/2}} \hat{P}_{k+1}P_{k+1}^{{1/2}})^{{1/2}})\big]^{1/2}, \label{eq:info_wass}
\end{align}
and
\begin{align}
    \mathcal{D}_{\text{info,H}}(b_k,b_{k+1})=\big[1-\frac{(\det \hat{P}_{k+1})^{{1/4}}  (\det {P}_{k+1})^{{1/4}} }{\left(\det \left(\frac{\hat{P}_{k+1}+P_{k+1} }{2}\right)\right)^{{1/2}}}\big]^{1/2}, \label{eq:info_Hell}
\end{align}
for $P_{k+1}\preceq\hat{P}_{k+1}$, respectively. 
Total costs defined in this manner are different from the direct Wasserstein and Hellinger distances between $b_k$ and $b_{k+1}$, because \eqref{eq:info_wass} and \eqref{eq:info_Hell} quantify the required information for transition $\hat{P}_{k+1}\rightarrow P_{k+1}$, as opposed to the transition $P_k \rightarrow P_{k+1}$. In contrast to the Wasserstein distance and the Hellinger distance themselves, the total costs defined using \eqref{eq:info_wass} and \eqref{eq:info_Hell} do not satisfy the axioms of metrics like symmetry. Further studies on these distance-like functions are left as future work.

The RRT*-based algorithm we present in Section~\ref{sec:algorithm}
can easily be modified to incorporate different choices. In Subsection~\ref{subsec:compare_metric}, we provide a comparison between different options of distance functions in a sample environment.
\vspace{-0.3cm}
\section{Problem formulation}
\label{sec:formulation}
In this section, we define the length of general paths in the belief space $\mathbb{B}$ using the distance function \eqref{eq:def_D} and formally state the shortest path problem.
\vspace{-0.3cm}
\subsection{Belief Chains and Belief Paths}
In the sequel, we use the term \emph{belief chain} to refer to a sequence of transitions from $b_k=(x_k, P_k)\in\mathbb{B}$ to $b_{k+1}=(x_{k+1}, P_{k+1})\in\mathbb{B}$, $k=0, 1, 2, ... , K-1$, where $K$ is a finite integer. 
We also use the term \emph{belief path} to refer to a function $\gamma: [0,T]\rightarrow \mathbb{B}$, $\gamma(t)=b(t)$ with $b(t)=(x(t), P(t))$.

The origin and the end point of the path $\gamma$ are denoted by $\gamma(0)$ and $\gamma(T)$, respectively. 
The parameter $t$ is often referred to as \emph{time}, but we remark that $t$ does not necessarily correspond to the physical time. The time of arrival of the robot at the end point depends on the the length of the path and the travel speed of the robot. 

\subsubsection{Lossless chains and paths}
Recall that a transition from $b_k=(x_k, P_k)$ to $b_{k+1}=(x_{k+1}, P_{k+1})$ is said to be \emph{lossless} if
\begin{equation}
\label{eq:lossless_def}
P_{k+1} \preceq \hat{P}_{k+1}(:=P_k+\|x_{k+1}-x_k\|W).
\end{equation}
If every transition in the belief chain $\{(x_k, P_k)\}_{k=0,1, ... , K}$ is lossless, we say that the belief chain is lossless.

Let $\gamma: [0,T]\rightarrow \mathbb{B}$, $\gamma(t)=(x(t), P(t))$ be a belief path. 
The \emph{travel length} of the path $\gamma$ from time $t=t_a$ to time $t=t_b$ is defined as $\ell(x[t_a, t_b])=\sup_{\mathcal{P}} \sum_{k=0}^{K-1} \|x(t_k)-x(t_{k+1})\|$,
where the supremum is over the space of all partitions $\mathcal{P}=(t_a=t_0<t_1<\cdots < t_K=t_b), K\in\mathbb{N}$. We say that a path $\gamma$ is \emph{lossless} if the condition
\begin{equation}
P(t_b) \preceq P(t_a)+\ell(x[t_a, t_b])W
\end{equation}
holds for all $0\leq t_a < t_b \leq T$. A path $\gamma(t)= (x(t),P(t))$ is said to be \emph{finitely lossless} if there exists a finite partition $\mathcal{P}=(0=t_0<t_1<\dots<t_K=T)$ such that for any refinement $\mathcal{P}'=(0=t'_0<t'_1<\dots<t'_{K'}=T)$ of $\mathcal{P}$ (i.e., $\mathcal{P}' \supseteq \mathcal{P}$), the belief chain $\{(x_{k'},P_{k'})\}_{k'=0, 1,\dots, K'}$ is lossless.

\subsubsection{Collision-free chains and paths}
Let $\mathcal{X}^{l}_{\text{obs}}\subset \mathbb{R}^d$ be a closed convex subset representing the obstacle $l \in \{1, \dots, M\}$. 
Consider a robot moving from a way point $x_k\in \mathbb{R}^d$ to $x_{k+1}\in \mathbb{R}^d$. Using $0\leq \lambda \leq 1$, the line segment connecting $x_k$ and $x_{k+1}$ is parametrized as $x[\lambda]=(1-\lambda)x_k+\lambda x_{k+1}$
Assuming that the robot's initial covariance is $P_k$, the evolution of the covariance matrix subject to the model  \eqref{eq:euler3} is written as
$P[\lambda]=P_k+\lambda\|x_{k+1}-x_k\|W$.
For a fixed confidence level parameter $\chi^2>0$, we say that the transition from $x_k$ to $x_{k+1}$ with initial covariance $P_k$ is \emph{collision-free} if
\begin{align}
&(x[\lambda]-x_{\text{obs}})^\top P[\lambda]^{-1}(x[\lambda]-x_{\text{obs}}) \geq \chi^2 \nonumber \\
&\forall \lambda\in [0,1], \quad \forall x_{\text{obs}}\in \mathcal{X}^l_{\text{obs}}, \quad \forall l \in\{1, \dots, L\}.
\end{align}
\begin{remark}
We say that a collision with obstacle $l$ is detected when 
\begin{equation}
(x[\lambda]-x_{\text{obs}})^\top P[\lambda]^{-1}(x[\lambda]-x_{\text{obs}}) < \chi^2\quad
\end{equation}
for some $ \lambda\in [0,1]$ and $x_{\text{obs}}\in \mathcal{X}^l_{\text{obs}}$. Collision detection can be formulated as a feasibility problem
\begin{align}
\label{eq:collision_checker}
&\begin{bmatrix}
\chi^2 &  (1-\lambda)x_k^\top+\lambda x_{k+1}^\top-x_{\text{obs}}^\top  \\
(1-\lambda)x_k+\lambda x_{k+1}-x_{\text{obs}} & P_k+\lambda\|x_{k+1}-x_k\|W
\end{bmatrix}\succ 0, \nonumber \\
&0\leq \lambda \leq 1,  \quad x_{\text{obs}}\in \mathcal{X}^l_{\text{obs}}, 
\end{align}
which is a convex program for each convex obstacle $\mathcal{X}^l_{\text{obs}}$.
 \qed
\end{remark}
We say that a belief chain $\{(x_k, P_k)\}_{k\!=\!0, 1, ... , K-1}$ is \emph{collision-free} if for each $k=0, 1, ... , K-1$, the transition from $x_k$ to $x_{k+1}$ with the initial covariance $P_k$ is collision-free. We say that a belief path $\gamma\!:\![0,T]\!\rightarrow \!\mathbb{B}$, $\gamma(t)\!=\!(x(t), P(t))$ is \emph{collision-free} if
\begin{equation}
\begin{split}
&(x(t)-x_{\text{obs}})^\top P^{-1}(t)(x(t)-x_{\text{obs}}) \geq \chi^2, \;\;\\
& \quad \forall t\in [0, T], \;\; \forall x_{\text{obs}}\in \mathcal{X}^l_{\text{obs}}, \ \  \forall l \in \{1, \dots, L\}.
\end{split}
\end{equation}
\vspace{-0.8cm}
\subsection{Path Length}
Let $\gamma: [0,T]\rightarrow \mathbb{B}$, $\gamma(t)=(x(t), P(t))$ be a path, and $\mathcal{P}=(0=t_0<t_1<\cdots < t_K=T)$ be a partition.
The length of the path $\gamma$ with respect to the partition $\mathcal{P}$ is defined as
$c(\gamma;\mathcal{P})=\sum_{k=0}^{K-1} \mathcal{D}(\gamma(t_k), \gamma(t_{k+1}))$,
where the function $\mathcal{D}$ is defined by \eqref{eq:def_D}.
The length of a path $\gamma$ is defined as the supremum of $c(\gamma;\mathcal{P})$ over all partitions
\begin{equation}
\label{eq:def_path_length}
c(\gamma):=\sup_\mathcal{P} c(\gamma;\mathcal{P}).
\end{equation}
The definition \eqref{eq:def_path_length} means that for each path with a finite length, there exists a sequence of partitions
$\{\mathcal{P}_i\}_{i\in\mathbb{N}}$ such that
$\lim_{i \rightarrow \infty} c(\gamma; \mathcal{P}_i) =c(\gamma)$.
\begin{remark}
If $\gamma(t)$ is  differentiable, then the losslessness condition  \eqref{eq:lossless_def} is equivalent to
$W\frac{d}{dt}x(t) \succeq \frac{d}{dt} P(t), \forall t \in [0, T]$. In this case, the path length can be expressed as:
 \begin{align*}
     c(\gamma) \!= \!\!\int_{0}^T \left[ \left\|\frac{d}{dt}x(t)\right\|+ \frac{\alpha}{2} Tr\Big( \big(W-\frac{d}{dt}P(t)\big)P^{-1}(t)\Big)\right] dt.
 \end{align*}
\end{remark}
\vspace{-0.5cm}
\subsection{Topology on the Path Space}
The proofs of asymptotic optimality of the original RRT* algorithm \cite{karaman2010incremental,karaman2011sampling} critically depends on the continuity of the path length function $c(\gamma)$. 
In this subsection, we introduce an appropriate topology on the space of belief paths $\gamma: [0,T]\rightarrow \mathbb{B}$ with respect to which the path length function $c(\gamma)$ is shown to be  continuous in Theorem~\ref{theo:continuity2} below.

The space of all belief paths $\gamma: [0,T]\rightarrow \mathbb{B}$ 
 can be thought of as an open subset (convex cone) of the space of \emph{generalized paths} $\gamma: [0,T]\rightarrow \mathbb{R}^d\times \mathbb{S}^d$. The space of generalized paths is a vector space on which addition and scalar multiplication are defined as  $(\gamma_1+\gamma_2)(t)=(x_1(t)+x_2(t), P_1(t)+P_2(t))$ and $a \gamma(t)=(a x(t), a P(t))$ for $a\in\mathbb{R}$, respectively. 
Let $\mathcal{P}=(0=t_0<t_1<\cdots < t_K=T)$ be a partition. The variation of a generalized path $\gamma$ with respect to $\mathcal{P}$ is defined as $V(\gamma; \mathcal{P}):=\|x(0)\|\bar{\sigma}(W)+\bar{\sigma}(P(0)) +\sum_{k=0}^{K-1}\big[\|x(t_{k+1})-x(t_k)\|\bar{\sigma}(W)+\bar{\sigma}(P(t_{k+1})-P(t_k))\big]$.
The total variation of a generalized path $\gamma$  is defined as $|\gamma|_{\text{TV}}:=\sup_{\mathcal{P}}V(\gamma; \mathcal{P})$.
Notice that $|\cdot|_{\text{TV}}$ defines a norm on the space of generalized paths. If we introduce $\|\gamma\|_\infty:=\sup_{t\in[0,T]} \|x(t)\|\bar{\sigma}(W)+\bar{\sigma}(P(t))$
then $\|\gamma\|_\infty \leq |\gamma|_{\text{TV}}$ holds \cite[Lemma 13.2]{carothers2000real}.

In what follows, we assume  the topology of total variation metric $|\gamma_1-\gamma_2|_{\text{TV}}$ on the space of generalized paths $\gamma\!:\! [0,T]\!\rightarrow\!\mathbb{R}^d\!\times\!\mathbb{S}^d$, which is then inherited to the space of belief paths $\gamma\!:\![0,T]\rightarrow\! \mathbb{B}(=\mathbb{R}^d\times \mathbb{S}_{++}^d)$. We denote by $\mathcal{BV}[0, T]$ the space of belief paths $\gamma\!:\![0,T]\!\rightarrow\!\mathbb{B}$ such that $|\gamma|_{\text{TV}}<\infty$.
\vspace{-0.4cm}
\subsection{The Shortest Belief Path Problem}
Let $b_0=(x_0, P_0)\in \mathbb{B}$ be a given initial belief state,  $\mathcal{B}_{\text{target}} \subset \mathbb{B}$ be a given closed subset representing the target belief region, and 
$\mathcal{X}^{l}_{\text{obs}}\subset \mathbb{R}^d$ be the given obstacle $l \in\{1, \dots, L\}$. Given a confidence level parameter $\chi^2>0$, the shortest path problem is formulated as 
\begin{equation}
\label{eq:main_problem}
\begin{split}
\!\!\min_{\gamma \in \mathcal{BV}[0, T]} \;\; & c(\gamma) \\
\text{s.t.}\;\;\;\; \; \;  & \gamma(0)=b_0, \; \gamma(T)\in \mathcal{B}_{\text{target}} \\
& (x(t)-x_{\text{obs}})^\top P^{-1}(t)(x(t)-x_{\text{obs}}) \geq \chi^2 \\
& \forall t\in [0, T], \;\; \forall x_{\text{obs}}\in \mathcal{X}^{l}_{\text{obs}}, \ \  \forall l \in \{1, \dots, L\}.
\end{split}
\end{equation}
We make the following mild assumption which will be needed in the development of Section~\ref{sec:algorithm}.
\begin{assumption}
There exists a feasible path $\gamma(t)=(x(t), P(t))$ for \eqref{eq:main_problem} such that $P(t)\in \mathbb{S}_\rho^d$ and $\rm{Tr}(P(t))\leq R$ for all $t\in[0,T]$, where $R>0$ and $\rho>0$ are constants.
\end{assumption}

The next theorem also plays a key role in the development of our algorithm in Section~\ref{sec:algorithm}.
\begin{theorem}
\label{theo:loss-lessmod}
For any collision-free belief chain $\{b_k = (x_k, P_k)\}_{k=0, 1, ... , K-1}$, there exist a collision-free and lossless chain  $\{b'_k=(x'_k, P'_k )\}_{k=0, 1, ... , K-1}$ with $x_k=x'_k$ and $P'_k \preceq P_k$ for $k=0, \dots, K$ that has a shorter (or equal) length in that $\sum_{k=1}^{K-1} \mathcal{D}(b'_k,b'_{k+1}) \leq \sum_{k=1}^{K-1} \mathcal{D}(b_k, b_{k+1}).$
\end{theorem}
\begin{proof}
See Appendix~\ref{ap:C}.
\end{proof}
\vspace{-0.2cm}
Theorem~\ref{theo:loss-lessmod} implies that 
the shortest path problem \eqref{eq:main_problem} always admits a ``physically meaningful'' path as an optimal solution. We also use Theorem~\ref{theo:loss-lessmod} to restrict the search for an optimal solution to the space of lossless paths in the algorithms we develop in the sequel.
\vspace{-0.45cm}
\subsection{Interpretation}
\label{sec:formulation_interpret}
By tuning the parameter $\alpha>0$, the shortest path problem formulation \eqref{eq:main_problem} is able to incorporate information cost at various degrees. As we will show in Section~\ref{sec:simulation_alpha}, different choices of $\alpha$ lead to qualitatively different optimal paths.
Notice that the proposed problem formulation \eqref{eq:main_problem} is purely geometric and does not use any particular sensor models to estimate sensing costs. This makes our motion planning strategy model-agnostic and widely applicable to scenarios where information cost for path following is concerned but actual models of sensors are not available or too complex to be utilized. 

Another advantage of the proposed approach is that, from the obtained belief path $(x(t), P(t))$, one can \emph{synthesize} a sensing strategy under which the planned covariance $P(t)$ is obtained as an outcome of Bayesian filtering.
To see this, assume that a desired belief path is given by a lossless belief chain $\{(x_k, P_k)\}_{k=0,1, ... , K}$. At every way point, the prior covariance 
\begin{equation}
\label{eq:sensor_reconstruction1}
    \hat{P}_k=P_{k-1}+\|x_k-x_{k-1}\|W
\end{equation}
needs to be updated to a posterior covariance $P_k$.
Such a belief update occurs as a consequence of a linear measurement 
\begin{equation}
\label{eq:sensor_reconstruction2}
    \by_{k}=C_k\bx_{k}+\bv_{k}
\end{equation}
with Gaussian noise $\bv_{k}\sim\mathcal{N}(0, V_k)$, provided that $C_k$ and $V_k$ are chosen to satisfy
\begin{equation}
\label{eq:sensor_reconstruction3}
C_k^\top V_k^{-1}C_k=P_{k}^{-1}-\hat{P}_{k}^{-1}.
\end{equation}
Notice that \eqref{eq:sensor_reconstruction3} together with \eqref{eq:sensor_reconstruction1} are the standard Riccati recursion for Kalman filtering.
Moreover, it can be shown that the linear sensing strategy \eqref{eq:sensor_reconstruction2} incurs the designated information gain (i.e., the equality \eqref{eq:info_gain1} holds), and is information-theoretically optimal in the sense that no other sensing strategy, including nonlinear ones, allows the covariance update from $\hat{P}_k$ to $P_k$ with less information gain.
In other words, \eqref{eq:sensor_reconstruction2} for $k=0,1, ... , K$ provides an optimal sensing strategy that perceives ``minimum yet critical'' information from the environment to perform the path following task.
References \cite{tanaka2015sdp,tanaka2016semidefinite,tanaka2017lqg} elaborate on an information-theoretic interpretation of the sensing mechanism  \eqref{eq:sensor_reconstruction2} as an optimal source-coder (data-compressor) for networked LQG control systems.

While information-theoretically optimal, the sensing strategy \eqref{eq:sensor_reconstruction2} may not be feasible in reality if the robot is not equipped with an adequate set of sensors.
In such cases, the planned sequence of covariance matrices $\{P_k\}_{k=0,1, ... , K}$ cannot be traced exactly. Even if the sensing strategy \eqref{eq:sensor_reconstruction2} is feasible (and thus the planned sequence of covariance matrices $\{P_k\}_{k=0,1, ... , K}$ is traceable), the realization of the mean $\hat{\bx}_k=\mathbb{E}[\bx_k|\by_0, \cdots , \by_k]$ inevitably deviates from the planned trajectory $\{x_k\}_{k=0,1, ... , K}$ because $\hat{\bx}_k$ is a random process whose realization depends on the realization of $\by_k$.

For these reasons, the belief path we obtain by solving \eqref{eq:main_problem} can only serve as a reference trajectory to be tracked in the real-time implementations. Even though the optimal belief path cannot be traced perfectly, various approaches can be taken to design a joint sensing and control policies for trajectory tracking.
In Section~\ref{sec:simulation}, we present simulation results showing that such a joint sensing and control strategy helps the robot to mitigate sensing cost (e.g.,  the frequency of sensing actions, the number of sensors that must be used simultaneously) during the path following phase.
\vspace{-0.4cm}
\section{Algorithm}
\label{sec:algorithm}
We utilize the RRT* algorithm \cite{karaman2011sampling} as a numerical method to solve the shortest path problem \eqref{eq:main_problem}. 
In this section, we develop three different variations of the algorithm. While they operate differently, they are basically the same in that they all incrementally construct directed graphs $G=(B,E)$ with randomly sampled belief nodes $B$ and edges $E$. 
To shed light on the asymptotic optimality of the proposed algorithms, we show that the path length function \eqref{eq:def_path_length} is continuous.

 \begin{figure*}[t!]
    \centering
    \subfloat[]{\includegraphics[clip,width=4.1cm]{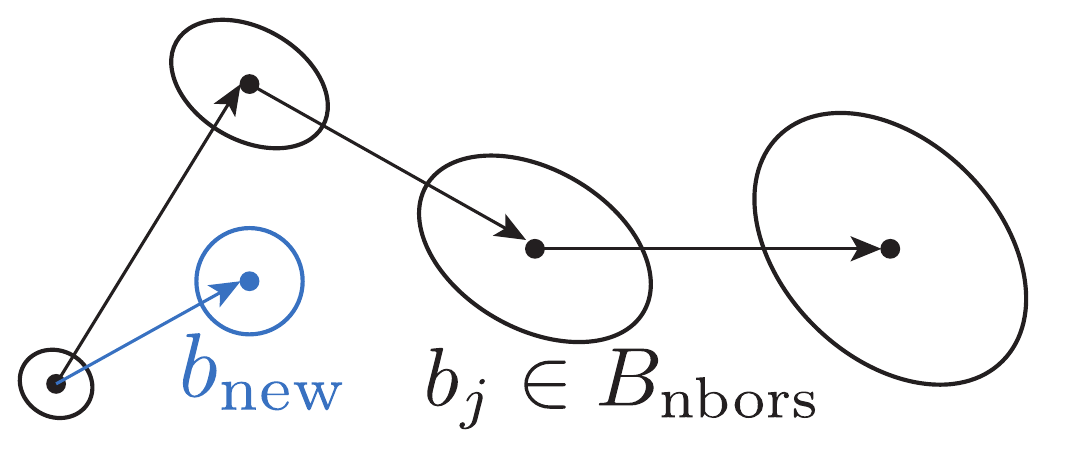}
    \label{fig:LossLess_1}} \quad
    \subfloat[]{\includegraphics[clip,width=4.3cm]{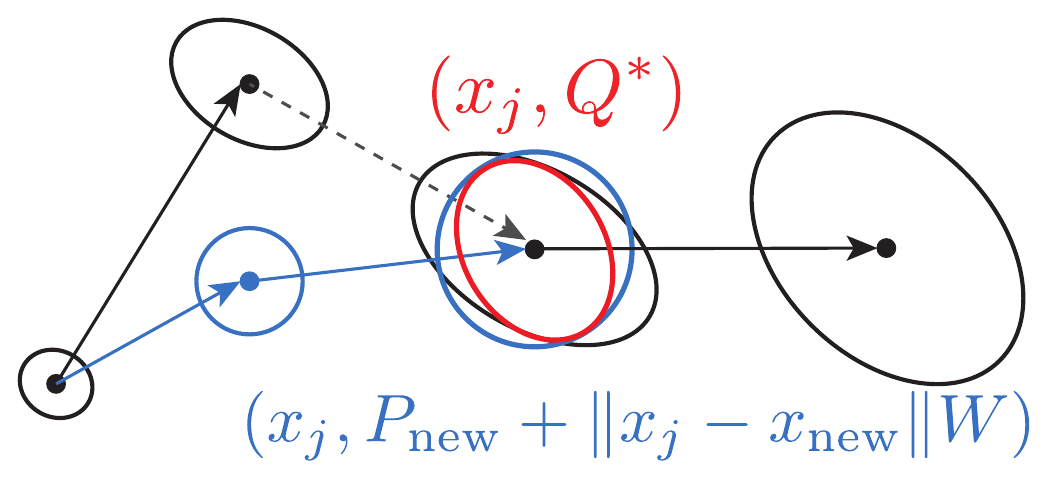}
    \label{fig:LossLess2}} \quad
    \subfloat[]{\includegraphics[clip,width=4.3cm]{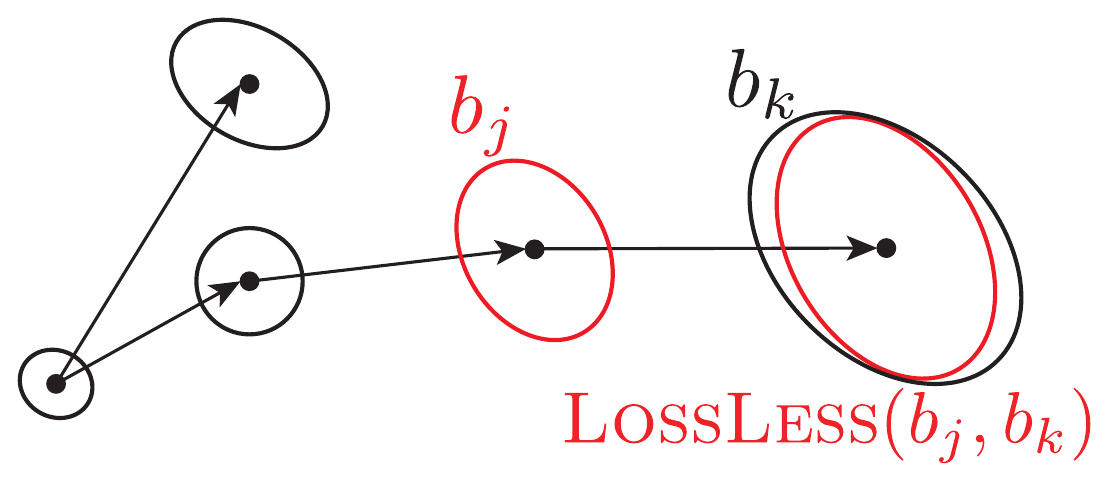}
    \label{fig:LossLess3}} \quad
    \subfloat[]{\includegraphics[clip,width=4cm]{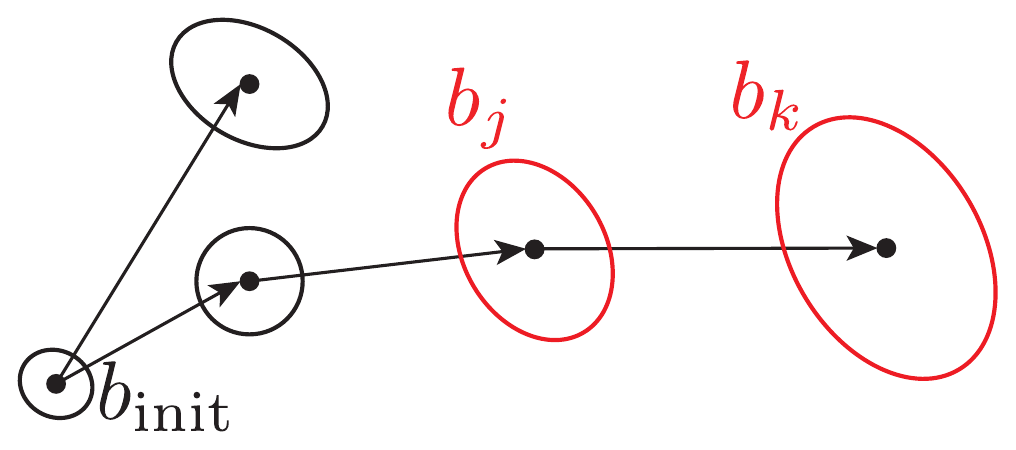}
    \label{fig:LossLess4}} 
    \caption{The lossless modification executed by the $\textsc{\fontfamily{cmss}\selectfont LossLess}$ function and its propagation to the descendants. (a) The $\textsc{\fontfamily{cmss}\selectfont Generate}(i)$ function samples the new node $b_{\rm new}$ (a blue ellipse). In the rewired process, the black ellipse $b_j$ located at the center is selected as an element of $B_{\rm nbors}$. (b) The prior covariance after the travel from $x_{\rm new}$ to $x_{j}$ is depicted as a blue ellipse in the center. The $\textsc{\fontfamily{cmss}\selectfont LossLess}(b_{\rm new}, b_j)$ function calculates the optimal solution of \eqref{eq:d_info_general0} to achieve the lossless transition. (c) $\textsc{\fontfamily{cmss}\selectfont LossLess}$ is propageted to the descendant of $b_j$ to create the lossless chain. (d) As the consequence of the rewiring and $\textsc{\fontfamily{cmss}\selectfont LossLess}$, the algorithm generates the path with lower $\mc D(b_{\rm init}, b_k)$  while achieving smaller covariances for $b_j$ and $b_k$.}
    \label{fig:LossLess_mod}
\end{figure*}
\vspace{-0.5cm}
\subsection{Basic Algorithm}
The basic implementation of RRT* in the belief space for the cost function \eqref{eq:def_path_length} is summarized in Algorithm~\ref{algo:1}.  The source code for Algorithm~\ref{algo:1} is available at \href{https://github.com/AlirezaPedram/IG-RRTstar}{https://github.com/AlirezaPedram/IG-RRTstar}. 
Sampling of a new node (Lines~3-7), an addition of a new edge connecting $b_j\in B$ and $b_{\rm new}$ (Lines~8-14), and the rewiring process (Line~15-20) are performed similarly to those of the original RRT* \cite{karaman2011sampling}. However, the new distance function ${\mc D}$ and its directional dependency necessitate the introduction of new functionalities.

 \begin{algorithm}[t]
 \label{algo:1}
    \footnotesize{
    $B \leftarrow \{b_{\text{init}}\}$; $E \leftarrow \emptyset$; cost($b_{\text{init}}$)$\leftarrow 0$;  $G\leftarrow (B,E)$\;
    \For{$i = 2:N$}{
        $b_i = (x_i,P_i) \leftarrow \textsc{\fontfamily{cmss}\selectfont Generate}(i)$\;
        $b_{\text{near}} \leftarrow \textsc{\fontfamily{cmss}\selectfont Nearest}(B,b_i)$\;
        $b_{\text{new}} \leftarrow \textsc{\fontfamily{cmss}\selectfont Scale} (b_{\text{near}},b_i, D_{\text{min}})$\;
        \If{$\textsc{\fontfamily{cmss}\selectfont FeasCheck} (b_{\text{near}},b_{\text{new}}) = \text{ \textup{True}}$}{
            $B \leftarrow B \cup b_{\text{new}}$\;
            $B_{\text{nbors,in}} \leftarrow \textsc{\fontfamily{cmss}\selectfont Neighbor-in} (B,b_{\text{new}}, D_{\text{min}})$\;
            $\text{cost}(b_{\text{new}}) \leftarrow realmax$\;
            \For{$b_j \in B_{\text{nbors,in}}$}{
                \If{$\textsc{\fontfamily{cmss}\selectfont FeasCheck} (b_j,b_{\text{new}}) = \text{\normalfont True}$ \textbf{\textup{ and}}
                $\textup{cost} (b_j) + \mathcal{D}(b_j,b_{\text{new}}) < \textup{cost}(b_{\text{new}}) $}{
                        $\text{cost}(b_{\text{new}}) \leftarrow \text{cost} (b_{j}) + \mathcal{D}(b_j,b_{\text{new}}) $\;
                        $b_{\text{nbor}}^* \leftarrow b_j$\;
                }
            }
            $E \leftarrow E \cup \left[b_{\text{nbor}}^*,b_{\text{new}} \right] $\;
             $B_{\text{nbors,out} \leftarrow \textsc{\fontfamily{cmss}\selectfont Neighbor-out} (B,b_{\text{new}},{D}_{\text{min}})}$\;
            \For{$b_j \in B_{\text{nbors,out}} \: \backslash \: b_{\text{nbor}}^*$}{
                \If{$\textsc{\fontfamily{cmss}\selectfont FeasCheck} (b_{\text{new}}, b_j) = \text{ \normalfont True} $ \textbf{\textup{ and}} $ \textup{cost} (b_{\text{new}}) + \mathcal{D}(b_{\text{new}},b_j) < \textup{cost} (b_j)$}{
                        $value \leftarrow \text{cost} (b_{\text{new}}) + \mathcal{D}(b_{\text{new}}, b_j) - \text{cost} (b_j)$\;
                        $E \leftarrow E \cup \left[ b_{\text{new}}, b_j \right] \backslash \left[ \text{parent}(b_j), b_j \right]$\;
                        $\textsc{\fontfamily{cmss}\selectfont UpdateDes}(G, b_j, value)$\;
                }
            }
        }\Return $G = (B, E)$
    }
    }
\caption{Information-Geometric RRT* Algorithm }
\end{algorithm}

Algorithm~\ref{algo:1} begins with the graph $G$ containing the initial node $b_{\rm init}$ and an empty edge set.
At each iteration, the $\textsc{\fontfamily{cmss}\selectfont Generate}(i)$ function creates a new node $b_i$ in the obstacle-free space $\mathcal{B}_{\text{free}}$ by randomly sampling an obstacle-free spatial location ($x\in \mathbb{R}^d$) and a covariance ($P \in \mathbb{S}^d_{++}$).
The $\textsc{\fontfamily{cmss}\selectfont Nearest}$ function finds the nearest node $b_{\text{near}}$ in the set $B$ from the node $b_i$ , which obtains the minimum $\mathcal{D}(b_{\text{near}},b')$.

The $\textsc{\fontfamily{cmss}\selectfont Scale}(b_{\text{near}},b_i, D_{\text{min}})$ function linearly shifts the generated point $b_i$ to a new location as:
\begin{equation*} 
\label{eq:steer}
b_{\text{new}} \!=\!\!
\begin{cases}
    b_{\text{near}} \!+\! \frac{D_{\text{min}}}{\mathcal{D}(b_i,b_\text{near})}\left(b_i-b_{\text{near}}\right)~ \text{if}~\mathcal{D}(b_i,b_{\text{near}}) > D_\text{min}, \\
    b_i \hspace{3.8cm} \text{otherwise,}
\end{cases}
\end{equation*}
where $b_i\pm b_{j}=(x_i \pm x_j, P_i \pm P_j)$, and for scalar $\alpha$, $\alpha b_i= (\alpha x_i, \alpha P_i)$. In addition, $D_\text{min} := \textup{min}\{ED_{\text{min}},\  r \left( \frac{\log n}{n} \right)^{1/d}\}$. $ED_{\text{min}}$ is a user-defined constant, $r$ is the connection radius, and $n$ is the number of nodes  in the graph $G$.
The $\textsc{\fontfamily{cmss}\selectfont Scale}$ function also checks that the $\chi^2$ confidence region of $b_{\rm new}$ does not interfere with any obstacle.

The $\textsc{\fontfamily{cmss}\selectfont FeasCheck}(b_{\text{near}}, b_{\text{new}}) = \textsc{\fontfamily{cmss}\selectfont IsLossless}(b_{i}, b_{j}) \ \textbf{\textup{and}} \ $ $  \textsc{\fontfamily{cmss}\selectfont ObsCheck}(b_{i}, b_{j})$ is a logical function. 
It ensures the transition $b_{\text{near}} \rightarrow b_{\text{new}}$ is a lossless transition i.e., it checks if 
$P_{\text{new}} \preceq P_{\text{near}}+ \|x_{\text{new}}-x_{\text{near}}\|W$. It also ensures that the $\chi^2$ confidence bound in this transition does not intersect with any obstacle $\ell\in \{1, \dots, M\}$ by solving problem~\eqref{eq:collision_checker}. The function $\textsc{\fontfamily{cmss}\selectfont Neighbor-in}(B, b_{\text{new}}, D_{\text{min}})$  returns the subset of nodes described as $ B_{\text{nbors,in}} =  \{ b_i \in B: \mathcal{D}(b_i,b_\text{new}) \leq D_{\text{min}} \}$

To find the parent node for the sampled node $b_{\text{new}}$, Lines 11-13 of Algorithm~\ref{algo:1} attempt connections from the neighboring nodes $B_{\text{nbors}}$ to  $b_{\text{new}}$. 
Among the nodes in $B_{\text{nbors}}$ from which there exists a collision-free and lossless path, the  node that results in minimum $\textup{cost}({b_{\text{new}}})$ is selected as the parent of $b_{\text{new}}$, where $\textup{cost}({b})$ denotes the cost of the path from the $b_{init}$ to node $b$. Note that the transitions are lossless and thus $\mathcal{D}(b_j, b_{\text{new}}) =\log\det(P_j+\|x_{\text{new}}-x_j\|W)- \log\det (P_{\text{new}})$. Line 14 establishes a new edge between the sought parent and $b_{\text{new}}$.

In the rewiring step (Lines 15-20), the function $\textsc{\fontfamily{cmss}\selectfont Neighbor-out}(B, b_{\text{new}}, {D}_{\text{min}})$ returns the neighboring nodes of $b_{\text{new}}$ defined  as $ B_{\text{nbors,out}}=\{ b_i \in B: {\mathcal{D}}(b_\text{new},b_i) \leq {D}_{\text{min}}\}$, which is different from $ B_{\text{nbors,in}}$ due to the asymmetric nature of $\mathcal{D}$. 
The algorithm replaces the parents of nodes $b_j$ in $B_{\text{nbors,out}}$ with $b_{\text{new}}$ if it results in lower $\textup{cost}(b_j)$. In line 16, the $\textsc{\fontfamily{cmss}\selectfont FeasCheck}$ function is called to check if $b_{\text{new}}\rightarrow b_j$ is lossless and collision-free.
Finally, for each rewired node $b_j$, its cost (i.e., $ \textup{cost}({b_j})$) and the cost of its descendants are updated via $\textsc{\fontfamily{cmss}\selectfont UpdateDes}(G, b_j, value)$ function in Line 20 as $\textup{cost}(.) \leftarrow \textup{cost}(.) + value $.

\vspace{-0.4cm}
\subsection{Improvement of Algorithm~\ref{algo:1} }
While Algorithm~\ref{algo:1} is simple to implement and easy to analyze, it can be modified in at least  two aspects to improve its computational efficiency. Algorithm~\ref{algo:2} shows the modified algorithm.

\subsubsection{Branch-and-Bound}
As the first modification, we deploy a branch-and-bound technique as detailed in {\color{blue} \cite{karaman2011anytime, ferguson2006anytime, otte2013c,gammell2014informed}}. For a  given tree $G$, let $b_\text{min} $ be the node that has the lowest cost along the nodes of $G$ within $\mathcal{B}_{\text{target}}$. It follows from the triangle inequality  (Theorem~\ref{theo:tria}) that the cost $\mathcal{D}(b,b_\text{goal})$  of traversing from $b$ to the goal region ignoring obstacles is a lower-bound for the cost of transitioning from $b$ to $b_\text{goal}$. The $\textsc{\fontfamily{cmss}\selectfont BranchAndBound(G)}$ function, Line 22 in Algorithm~\ref{algo:2}, periodically deletes the nodes $B'' = \{b\in B: \textup{cost}(b) +\mathcal{D}(b, b_\text{goal}) \geq \textup{cost}({b_\text{min}})\}$. This elimination of the non-optimal nodes speeds up the RRT* algorithm.

\begin{algorithm}[ht]\label{algo:2}
    \footnotesize{
    $B \leftarrow \{b_{\text{init}}\}$; $E \leftarrow \emptyset$; $\textup{cost}(b_{\text{init}}) \leftarrow 0$ $G\leftarrow (B,E)$\;
    \For{$i = 2:N$}{
        $b_i = (x_i,P_i) \leftarrow \textsc{\fontfamily{cmss}\selectfont Generate}(i)$\;
        $b_{\text{near}} \leftarrow \textsc{\fontfamily{cmss}\selectfont Nearest}(B,b_i)$\;
        $b_{\text{new}} \leftarrow \textsc{\fontfamily{cmss}\selectfont Scale} (b_{\text{near}},b_i,\hat{D}_{\text{min}})$\;
        \If{$\textsc{\fontfamily{cmss}\selectfont ObsCheck} (b_{\text{near}},b_{\text{new}}) = \text{ \normalfont True}$}{
            $B_{\text{nbors}} \leftarrow \textsc{\fontfamily{cmss}\selectfont Neighbor} (B,b_{\text{new}},\hat{D}_{\text{min}})$\;
            $\text{cost}(b_{\text{new}}) \leftarrow realmax$\;
            \For{$b_j \in B_{\text{nbors}}$}{
                \If{$\textsc{\fontfamily{cmss}\selectfont ObsCheck} (b_j,b_{\text{new}}) = \text{\normalfont True} $ \textbf{\textup{and}} $  \text{cost}(b_j) + \mathcal{D}(b_j,b_{\text{new}}) < \text{cost}(b_{\text{new}}) $}{
                        $\text{cost}(b_{\text{new}}) \leftarrow \text{cost} (b_{j}) + \mathcal{D}(b_j,b_{\text{new}}) $\;
                        $b_{\text{nbor}}^* \leftarrow b_j$\;
                }
            }
            $b_{\text{new}} \leftarrow \textsc{\fontfamily{cmss}\selectfont LossLess} (b_{\text{nbor}}^*, b_{\text{new}})$\;
            $B \leftarrow B \cup B_{\text{new}}$\;
            $E \leftarrow  E \cup \left[b_{\text{nbor}}^*,b_{\text{new}} \right]$\;
            \For{$b_j \in B_{\text{nbors}} \: \backslash \: b_{\text{nbor}}^*$}{
                \If{$\textsc{\fontfamily{cmss}\selectfont ObsCheck} (b_{\text{new}},b_j) = \text{\normalfont True} $ \text{\textup{ and}} $ \text{cost} (b_{\text{new}}) + \mathcal{D}(b_{\text{new}},b_j) < \text{cost} (b_j)$}{
                    
                        $b_{j} \leftarrow \textsc{\fontfamily{cmss}\selectfont LossLess} ( b_{\text{new}}, b_j)$\; 
                        $E \leftarrow E \cup \left[ b_{\text{new}}, b_j \right] \backslash \left[ \text{parent}(b_j), b_j \right]$\;
                        $\text{cost}(b_j) \leftarrow \text{cost} (b_{\text{new}}) + \mathcal{D}(b_{\text{new}},b_j) $\;
                        $\textsc{\fontfamily{cmss}\selectfont UpdateDes}(G, b_j)$\;
                }
            }
        } $\textsc{\fontfamily{cmss}\selectfont BranchAndBound} (G)$\;
        \Return $G = (B, E)$
    }
    }
\caption{Improved Information-Geometric RRT* Algorithm}
\end{algorithm}
\subsubsection{Lossless Modification}
Simulation studies with Algorithm~\ref{algo:1} show that the $\textsc{\fontfamily{cmss}\selectfont IsLossless}(b_{i}, b_{j})$ check often returns $\textup{False}$, meaning that  Lines 11-13 and Lines 16-19 are skipped frequently. Consequently, an extremely large $N$ may be required for Algorithm~\ref{algo:1} to produce meaningful results.
To resolve this issue, we adopt an extra step called \emph{lossless modification} which, as shown graphically in Fig.~\ref{fig:LossLess_mod}, ensures that the existing links are all lossless. By including lossless modification, Algorithm~2 no longer needs  
to call $\textsc{\fontfamily{cmss}\selectfont IsLossless}(b_i,b_j)$ to verify the connected link $b_i \rightarrow b_j$ is lossless. Therefore, in Algorithm~\ref{algo:2},  we employ $\textsc{\fontfamily{cmss}\selectfont ObsCheck}(b_i, b_j)$ instead of $\textsc{\fontfamily{cmss}\selectfont FeasCheck}(b_i, b_j)$. This alteration of functions mitigates the computational burden. 

In Algorithm~\ref{algo:2}, we first find the parent $b_{\text{nbor}}^*$ for $b_{new}$ in Lines 6-12. Then in Line 13, the covariance component in $b_{new}$ is modified so that the transition from its parent becomes lossless.  Specifically, the $\textsc{\fontfamily{cmss}\selectfont LossLess}$ function computes $\textsc{\fontfamily{cmss}\selectfont LossLess} ( b_j, b_{\text{new}})= (x_{\text{new}}, Q^*)$, where $Q^*$ is the minimizer of \eqref{eq:d_info_general0} in computing $\mathcal{D}_{\text{info}} (b_j, b_{\text{new}})$. In rewiring step, a similar lossless modification is performed for the rewired node $b_j$ in Line 18 to assure the transition $b_{\text{new}} \rightarrow b_j$ is lossless. After modifying the rewired node $b_j$, all its descendant belief nodes are modified sequentially so that all transitions  become lossless as detailed in Algorithm~\ref{algo:update_des}.  Modifying the descendant nodes is a common practice, originally introduced in \cite{arslan2013use}, to increase the convergence rate of sampling-based methods. The process of lossless modification is similar to the method introduced in Appendix~\ref{ap:C} for constructing a lossless collision-free chain that has a lower cost from a given collision-free chain.  The source code for Algorithm~\ref{algo:2} is accessible at \href{https://github.com/AlirezaPedram/IG-RRTstar-rapid}{https://github.com/AlirezaPedram/IG-RRTstar-rapid}.

\begin{algorithm}[t]\label{algo:update_des}
    \footnotesize{
    $Parent\_list \leftarrow \{b_{\text{rewired}}\}$; $Child\_list \leftarrow \emptyset $\;  
    \While {$Parent\_list \neq \emptyset$}{
        \For{$b_i \in Parent\_list $}{
        $Child\_list.append(\text{Children}(b_i))$ \;
        }
        \For {$b_j \in Child\_list$}{
        $b_{j} \leftarrow \textsc{\fontfamily{cmss}\selectfont LossLess} ( \text{parent}(b_j), b_j)$\;
        $\text{cost}(b_j) \leftarrow \text{cost}( \text{parent}((b_j)) + \mathcal{D}(\text{parent}((b_j), b_j) $\;
        }
        $Parent\_list \leftarrow Child\_list$\;
        $Child\_list \leftarrow \emptyset$\;
    } 
}
\caption{$\textsc{\fontfamily{cmss}\selectfont UpdateDes}(G, b_{\text{rewired}})$}
\end{algorithm}
\vspace{-0.2cm}

\begin{remark}
In Algorithm~\ref{algo:2},
we use
$\hat{\mathcal{D}}(b_1,b_2)= \|x_1-x_2\|+ \|P_1-P_2\|_F$ to define functions $\textsc{\fontfamily{cmss}\selectfont Nearest}$, $\textsc{\fontfamily{cmss}\selectfont Scale}$, and $\textsc{\fontfamily{cmss}\selectfont Neighbor}$, to reduce the execution time for these functions. 
Since $\hat{\mathcal{D}}$  is symmetric, the neighbors of the sampled  node
need to be computed only once (instead of twice as in Algorithm~\ref{algo:1}), which helps increase the speed of the algorithm.
However, $\textsc{\fontfamily{cmss}\selectfont ObsCheck}$ is called twice (Lines 10 and 17) because $\textsc{\fontfamily{cmss}\selectfont ObsCheck}(b_j, b_{\text{new}}) = \textup{True}$ does not necessarily imply  $\textsc{\fontfamily{cmss}\selectfont ObsCheck}(b_{\text{new}}, b_j) = \textup{True}$. Note that $\mathcal{D}$ (not the simplified $\hat{\mathcal{D}}$) is used elsewhere in the algorithm, and thus the algorithm still attempts to find the shortest path with respect to the original distance metric $\mathcal{D}$.
\end{remark}

\begin{figure*}[t!]
    \centering
    \subfloat[$\alpha = 0$]
    {\includegraphics[trim = 0.3cm 0cm 1.46cm 0.70cm, clip=true, width=0.55\columnwidth]{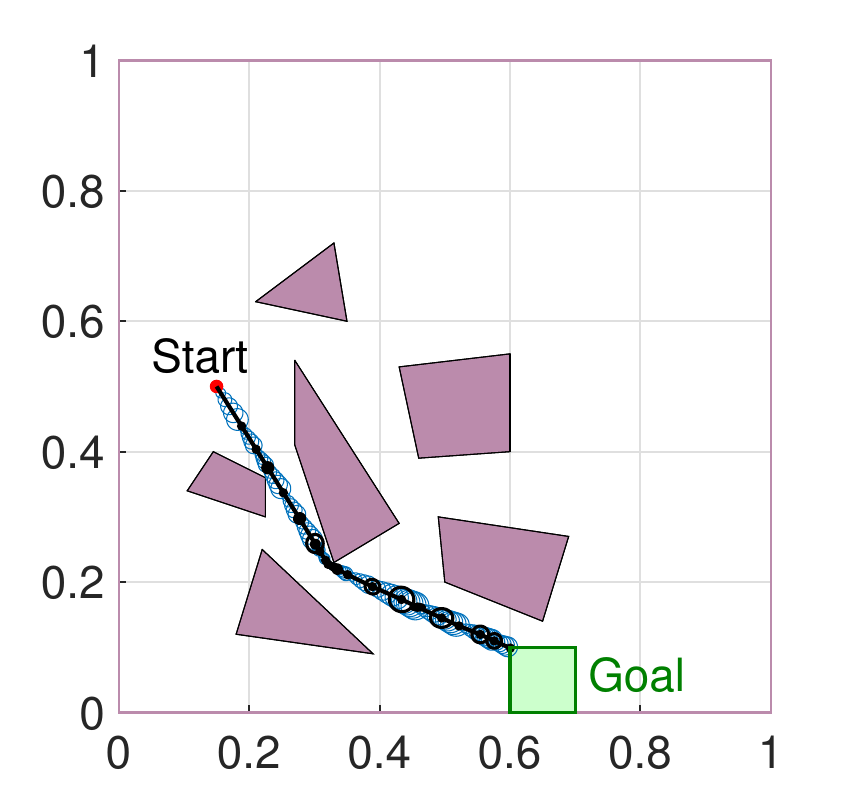}
    \label{fig:sim_snap_nocbf}} \quad 
    \subfloat[$\alpha = 0.3$]
    {\includegraphics[trim = 0.3cm 0cm 1.46cm 0.70cm, clip=true, width=0.55\columnwidth]{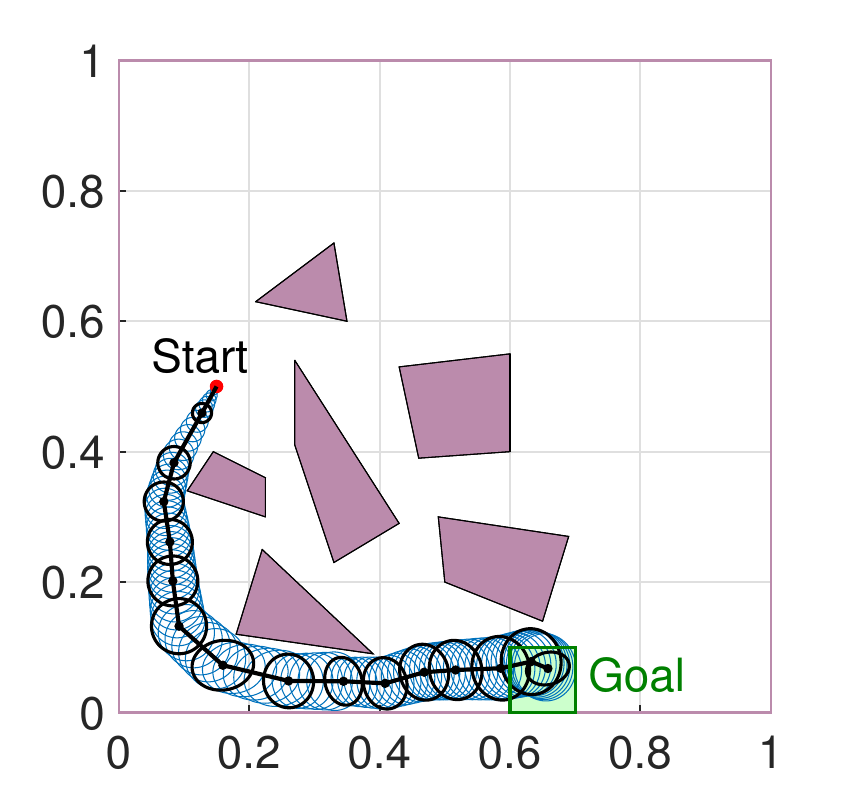}
    \label{fig:sim_snap_cbf19}} \quad
    \subfloat[$\alpha = 0.7$]
    {\includegraphics[trim = 0.3cm 0cm 1.46cm 0.70cm, clip=true, width=0.55\columnwidth]{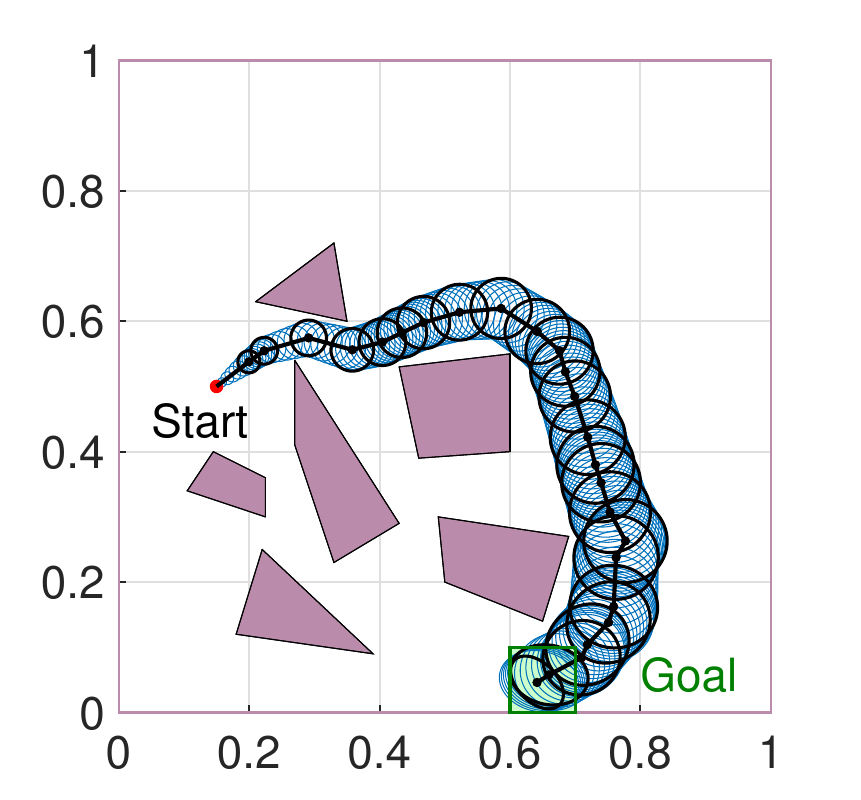}
    \label{fig:sim_snap_cbf19_1}} \quad
  
    \caption{Simulation results with $\alpha = 0, 0.3, 0.7$ under the existence of multiple obstacles. Disturbance noise intensity is set to $W = 10^{-3} I_2$ and confidence ellipses representing ${\textup{Pr}}=90 \%$ certainty regions. The boundaries of the plots are considered as obstacles.}
    \label{fig:2D_mult}
\end{figure*}
\vspace{-0.5cm}
\subsection{Incorporating Sensor Constraints}
\label{sec:sensor_constraint}
Algorithms~\ref{algo:1} and \ref{algo:2}  are purely geometric in the sense that they ignore physical constraints of the robot's hardware, including sensors. While this is an advantage in the sense we discussed in Section~\ref{sec:proposed_approach}, it is also a limitation. 
In particular, a possible glitch of a belief chain synthesized by our algorithms is that it can be ``physically unrealizable'' if the robot is not equipped with a sensor to perform necessary measurements. 
For example, a robot only equipped with a camera with a bounded field of view (FOV) is prohibited from obtaining information outside the FOV, or the situation where sensors provide valid measurements only in a specific range. Drones relying on GPS signals for localization may not have an access to sensor measurements in GPS-denied regions.
Such sensor constraints restrict the way in which belief states are updated. 

Fortunately, there is a simple remedy for this issue. Consider the transition from $b_i=(x_i, P_i)$ to $b_j=(x_j, P_j)$, and assume that the only sensor available at the belief state $b_j$ is $\by_{j}=C_j\bx_{j}+\bv_{j}$ with $\bv_{j}\sim\mathcal{N}(0, V_j)$.
In this transition, the covariance $P_i$ first 
grows into $\hat{P}_j = P_i+\|x_j-x_i\| W $ in the prediction step, which is then reduced to $\tilde{P}_j := (\hat{P}_j^{-1}+ C_j ^\top V_j^{-1} C_j)^{-1}$ in the update step. Thus, the transition from $b_i$ to $b_j$ is clearly feasible if $P_j \succeq \tilde{P}_j$. 
Therefore, to incorporate the sensor constraint, the function $\textsc{\fontfamily{cmss}\selectfont FeasCheck}$ in Algorithm~\ref{algo:1} simply needs to be replaced with $\textsc{\fontfamily{cmss}\selectfont FeasCheck2}(b_i, b_j):= \textsc{\fontfamily{cmss}\selectfont FeasCheck}(b_i,b_j) \  \textbf{\textup{and}} \  (P_j \succeq \tilde{P}_j)$.
\vspace{-0.5cm}
\subsection{Asymptotic Optimality}
The RRT* algorithm \cite{karaman2011sampling} is an improvement of the RRT algorithm \cite{lavalle2001randomized} to achieve asymptotic optimality (the cost of the best path discovered converges to the optimal one almost surely as the number of samples increases).
Since the algorithms we introduced in this section are RRT*-based, their asymptotic optimality can be naturally conjectured. 
Unfortunately, it is not straightforward to prove such a property in our setting due to a number of differences between the problem formulation \eqref{eq:main_problem} and the premises utilized in the original proof of asymptotic optimality \cite{karaman2010incremental,karaman2011sampling} (see also \cite{solovey2020revisiting}).

One of the premises that the original proof critically relies on is the continutity of the path cost function. Thus, to understand the asymptotic optimality of the algorithms we introduced in this section, it is essential to understand the continuity  of the path length function $c(\gamma)$ we introduced by \eqref{eq:def_path_length}.
The next theorem shows that the function $c(\gamma)$ is continuous in the space of finitely lossless paths with respect to the topology of total variation:
\begin{theorem}
\label{theo:continuity2}
Let $\gamma: [0,T]\rightarrow \mathbb{R}^d\times \mathbb{S}_\rho^d$ and $\gamma': [0,T]\rightarrow \mathbb{R}^d\times \mathbb{S}_\rho^d$ be paths. Suppose  $\gamma \in \mathcal{BV}[0, T]$ and $\gamma' \in \mathcal{BV}[0, T]$ and they are both finitely lossless. Then, for each $\epsilon > 0$, there exists $\delta > 0$ such that
\[
|\gamma'-\gamma|_{\text{TV}}\leq \delta \quad \Rightarrow \quad  |c(\gamma')-c(\gamma)| \leq \epsilon.
\]
\end{theorem}
\begin{proof}
See Appendix~\ref{ap:A}, where without loss of generality we assume $T=1$.
\end{proof}
\vspace{-0.1cm}
A complete proof of asymptotic optimality requires much additional work and must be postponed as future work.
\vspace{-0.3cm}
\section{Numerical Experiments}
\label{sec:simulation}

\subsection{Impact of Changing \texorpdfstring{$\alpha$}{a} }
\label{sec:simulation_alpha}
In this experiment, Algorithm~\ref{algo:2} is tested  with $\alpha = 0.0,\ 0.3,$ and $0.7$ in a two-dimensional configuration space containing multiple obstacles shown in Fig.~\ref{fig:2D_mult}.
The paths shown in Fig.~\ref{fig:2D_mult} are generated by sampling 20,000 nodes. Sampled covariance ellipses are shown in black and the propagation between samples are shown in blue.

As shown in Fig.~\ref{fig:2D_mult}~(a), the algorithm yields a path with the shortest Euclidean length when $\alpha = 0$. If the weight is increased to $\alpha=0.7$, the algorithm finds a long path depicted in Fig.~\ref{fig:2D_mult}~(c). 
When $\alpha = 0.3$, the path illustrated in Fig.~\ref{fig:2D_mult}~(b) is obtained.
Numerical experiment shows that the optimal paths are homotopic to Fig.~\ref{fig:2D_mult}~(a), (b), and (c) when $\alpha \leq 0.3$,  $0.3 < \alpha \leq 0.5$, and $0.5 < \alpha$, respectively. In the sequel,  red, purple, and blue colors are used to refer to the paths that are homotopic to the paths shown in Fig.~\ref{fig:2D_mult}~(a), (b), and (c). 
Fig.~\ref{fig:cost}~(a) and (b) display the travel and information costs as  functions of $\alpha \in \{0.1, \dots, 1 \}$ for the environment shown in Fig.~\ref{fig:2D_mult}. 
Fig.~\ref{fig:cost}~(b) 
shows a decreasing trend of the information cost as a function of $\alpha$. It is not monotonically decreasing (even within the same homotopy class) because of probabilistic nature of RRT*; for each $\alpha$, the generated path itself is a random variable.  

From these simulation results,
it is evident that more clearance around the path tends to imply less perception cost. However, perception cost reduction cannot be achieved merely by adopting the maximum clearance planners (e.g., \cite{sakcak2021complete} and the references therein).
This is because these planners only account for the minimum clearance incurred along the path, whereas the perception cost is a function of the clearance of all way-points along the path.

\begin{figure}[t!]
    \vspace{-0.3cm}
    \centering
    \subfloat[Travel Cost]
    {\includegraphics[trim = 0.1cm 0cm 1cm 0.4cm, clip=true, width=0.4\columnwidth]{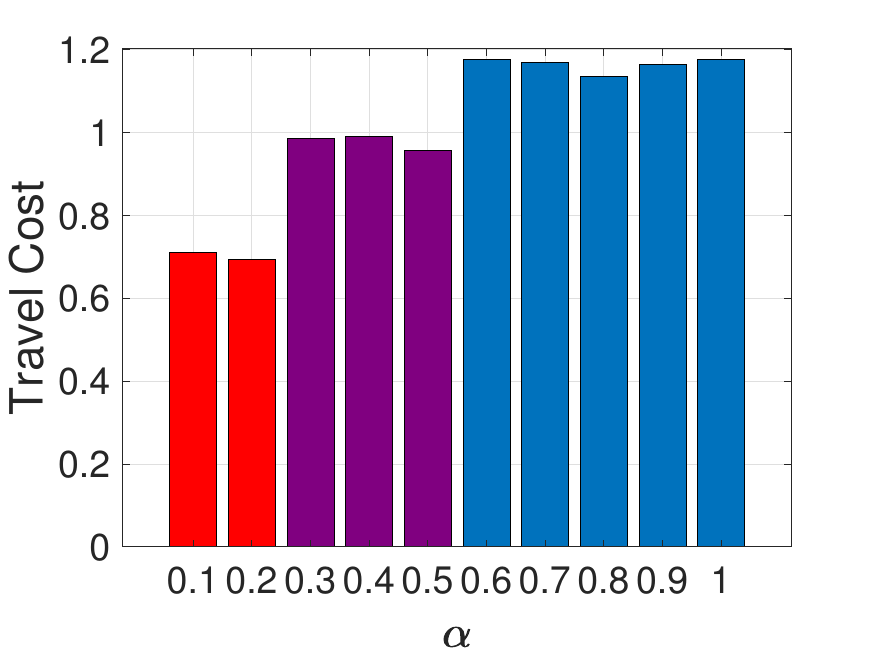}
    \label{fig:travel_cost}} \quad 
    \subfloat[Information cost.]
    {\includegraphics[trim = 0.1cm 0cm 1cm 0.4cm, clip=true, width=0.4\columnwidth]{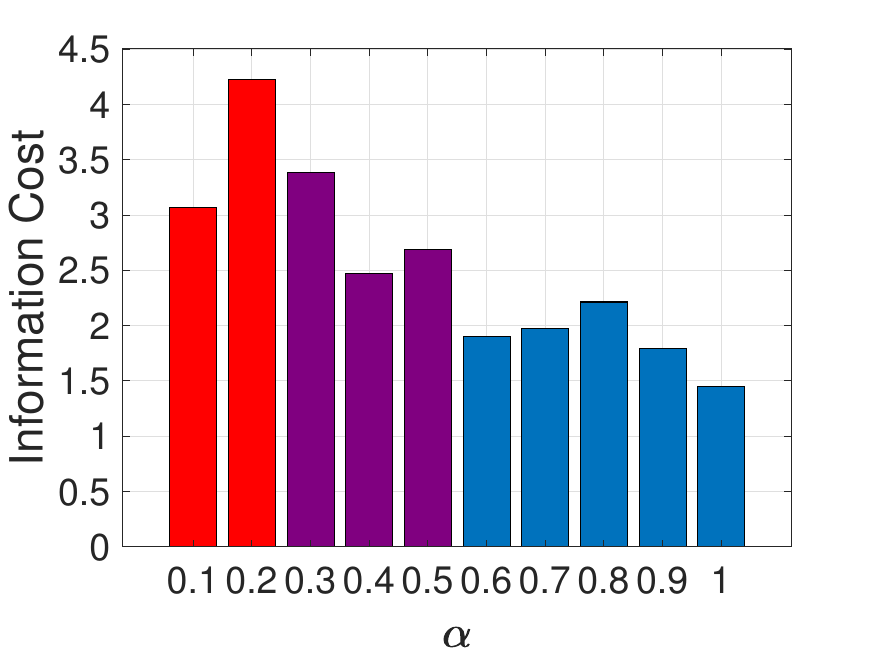}
    \label{fig:info_cost}}
    \caption{Travel and information costs as functions of $\alpha$ for the simulations tested in the environment shown in Fig.~\ref{fig:2D_mult}.}
    \label{fig:cost}
\end{figure}
\vspace{-0.4cm}
\subsection{Asymmetry of Path Length}
\label{sec:asymmetry}
As seen in Section~\ref{sec:simulation_alpha}, there is a direct connection between the clearance of a path and perception cost required to follow it. However, the path clearance is not the only factor that determines the perception cost. 
To see this fact, consider a ground robot moving into (resp. moving out of) a funnel-shaped safe region shown in Fig.~\ref{fig:Qbot}\subref{fig:Qbot_conv} (resp. Fig.~\ref{fig:Qbot}\subref{fig:Qbot_div}). 
For this ground robot, we assume Quanser's Qbot2e system (see \href{https://www.quanser.com/products/qbot-3/ for full specification }{https://www.quanser.com/products/qbot-3/} for specification of this robot). This differential wheeled mobile platform is equipped with two wheel encoders and an IMU unit. This robot uses the Microsoft Kinect sensor with a $57\,{\rm deg}$ field of view to obtain color image frames (RGB) as well as depth information. 
We use a particle filter for visual-inertial odometry with $5000$ particles. 
For this simulation, we simulate the navigation of the robot along the blue dashed lines shown in the Fig.~\ref{fig:Qbot}, where the robot should avoid entering the light red regions. More precisely, the ${\textup{Pr}}=90 \%$ confidence ellipse should not overlap with the red regions. 
\begin{figure}[t!]
    \centering
    \subfloat[Qbot navigation through converging passage]
    {\includegraphics[trim = 0.1cm 0.7cm 1cm 2.5cm, clip=true, width=0.75\columnwidth]{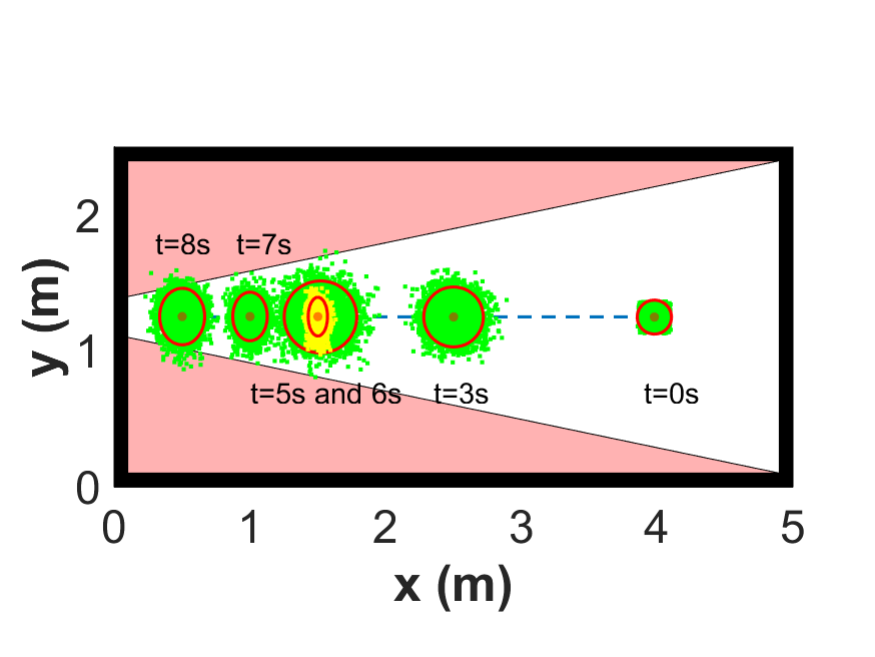}
    \label{fig:Qbot_conv}} \\
    \subfloat[Qbot navigation through diverging passage]
    {\includegraphics[trim = 0.1cm 0.7cm 1cm 2.5cm, clip=true, width=0.75\columnwidth]{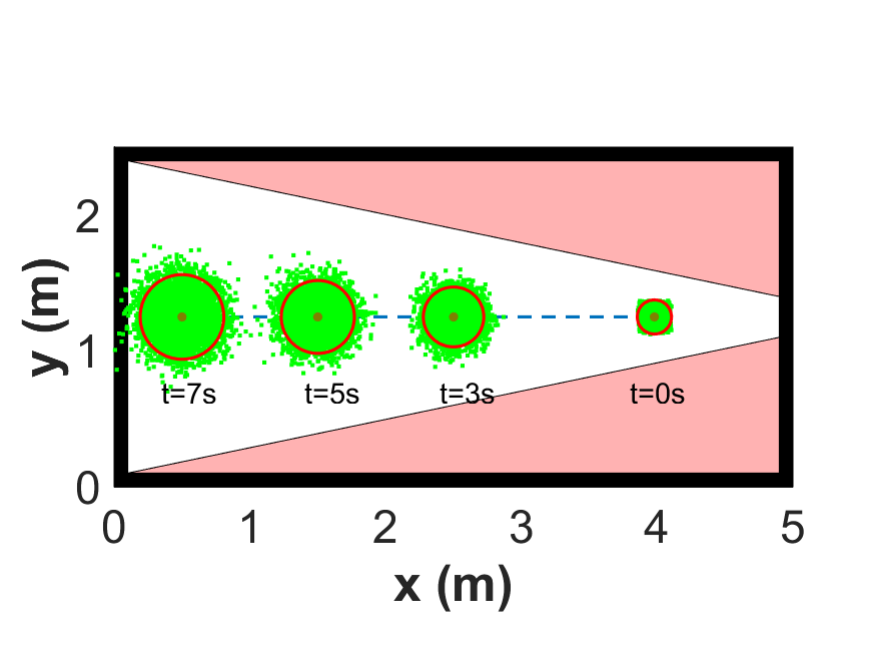}
    \label{fig:Qbot_div}}
    \caption{Navigation of Qbot through converging and diverging passages. The reference trajectory is shown in the blue line. Green dots show the position of particles in the particle filter, where the red ellipses are ${\textup{Pr}}=90 \%$ confidence ellipses. As shown in Fig.~\ref{fig:Qbot}(a), Qbot needs to stop at $t=5$\,s for $1$\,s and rotate to obtain visual information to reduce the uncertainty while going through a converging passage. The position of particles after the stop is shown in yellow. However, Qbot does not need to stop while traveling a diverging passage as shown in Fig.~\ref{fig:Qbot}(b).}
    \label{fig:Qbot}
\end{figure}

As shown in Fig.~\ref{fig:Qbot}, the particles start diverging during the navigation due to the drift in odometry data. As demonstrated in Fig.~\ref{fig:Qbot}\subref{fig:Qbot_conv}, for safe navigation in the converging passage, the robot should temporarily stop following the path at $t=5\,{\rm s}$ for $1\,{\rm s}$, and instead start rotating to localize itself using visual data obtained from the camera. This stop-and-localization is an example of perception effort. However, as shown in Fig.~\ref{fig:Qbot}\subref{fig:Qbot_div}, it is not required to stop for visual localization while navigating through the diverging passage. As a result, following  the path in Fig.~\ref{fig:Qbot}\subref{fig:Qbot_conv} takes $1\,{\rm s}$ longer than the path in Fig.~\ref{fig:Qbot}\subref{fig:Qbot_div}, although the blue paths in both settings have the same clearance from red regions. This example shows that  max-clearance planners (e.g., \cite{sakcak2021complete}) might not be able to find path with low perception effort. However, our proposed path length function $c(\gamma)$ is direction-dependant, and it is able to distinguish the sensing effort incurred by two paths in Fig.~\ref{fig:Qbot}.
Here, we demonstrate this ability using another sample configuration space (Fig.~\ref{fig:2D_asym}) in which there exists two paths having the same Euclidean length but different $c(\gamma)$.
\begin{figure}[t]
\vspace{-0.2cm}
\centering
{\includegraphics[trim = 0.9cm 0.7cm 0cm 0.7cm, clip=true, width=0.9\columnwidth]{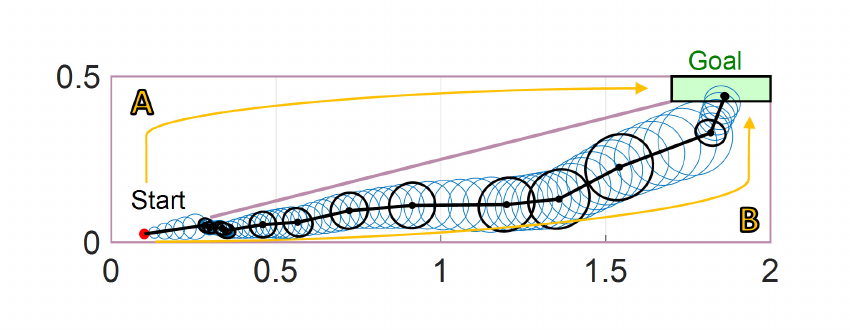}}
\caption{Results of Algorithm~\ref{algo:2} with $10,000$ nodes in the two-dimensional space containing roughly two paths, A and B, separated by a diagonal wall. The black line is the shortest path with the associated covariance ellipses. The blue ellipses illustrate the propagation of covariance between nodes. The simulation was completed with $W = 10^{-3} I_2$ and $\chi^2$ covariance ellipses representing ${\textup{Pr}}=90 \%$ certainty regions. The boundaries of the plots are considered as obstacles.}
\label{fig:2D_asym}
\end{figure}
\begin{figure*}[t!]
    \centering
    \subfloat[$\alpha = 0.1$]
    {\includegraphics[trim = 0.3cm 0cm 1.46cm 0.70cm, clip=true, width=0.55\columnwidth]{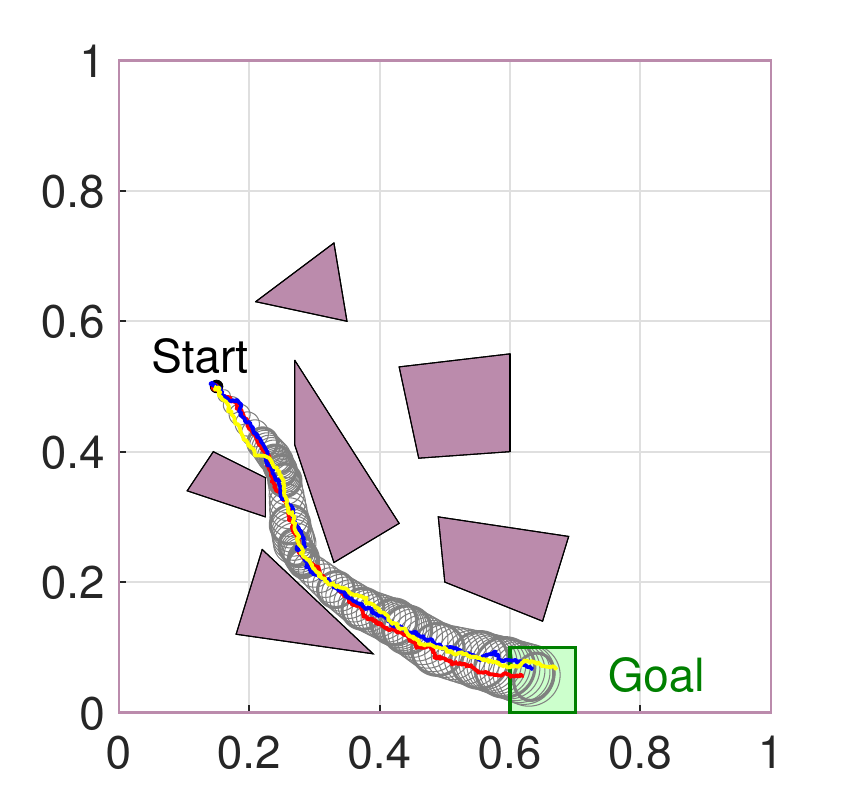}
    \label{fig:EV_alpha_01}} \quad 
    \subfloat[$\alpha = 0.3$]
    {\includegraphics[trim = 0.3cm 0cm 1.46cm 0.70cm, clip=true, width=0.55\columnwidth]{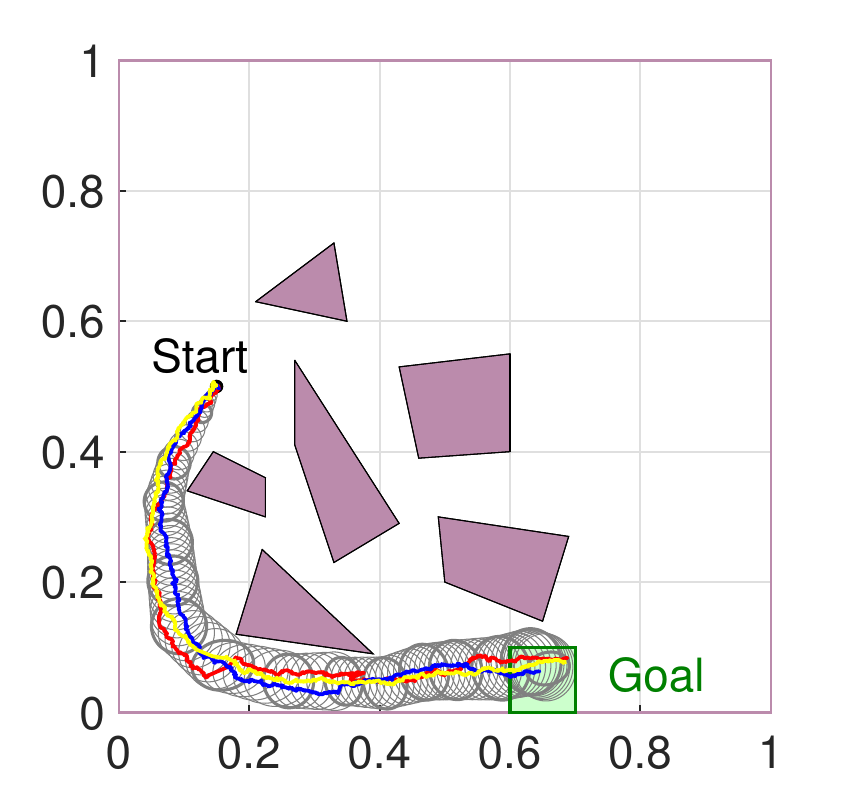}
    \label{fig:EV_alpha_03}} \quad
    \subfloat[$\alpha = 0.7$]
    {\includegraphics[trim = 0.3cm 0cm 1.46cm 0.70cm, clip=true, width=0.55\columnwidth]{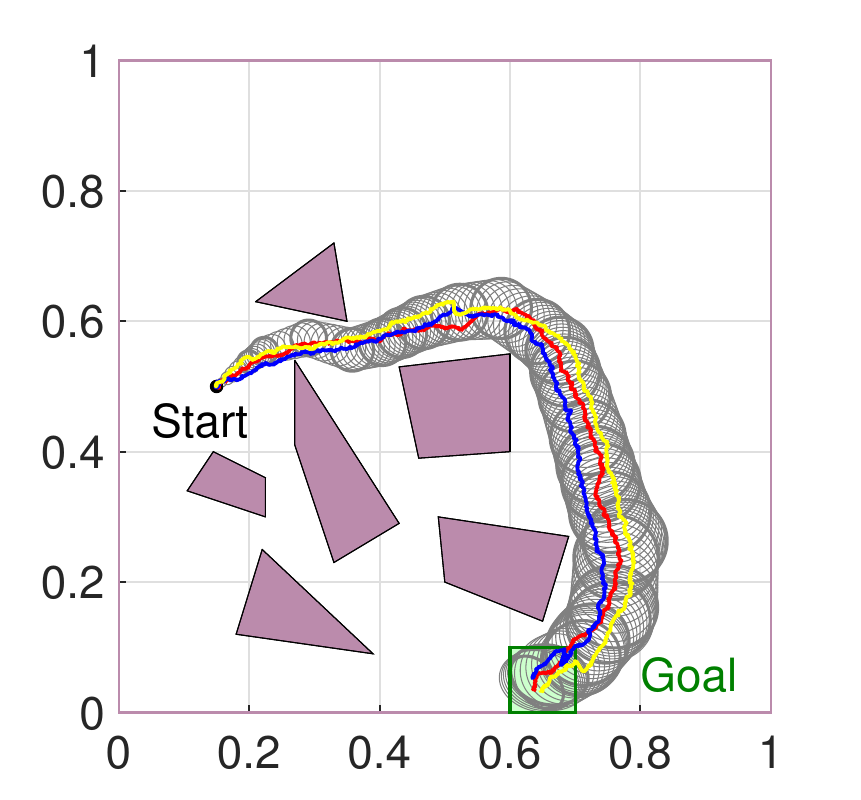}
    \label{fig:EV_alpha_07}} \quad
    \caption{
    The reference paths generated with $\alpha=0.1, 0.3$ and $0.7$ are followed by an event-based LQG controller using the high-precision sensor with $Z=10^{-4} I_2$. Three sample trajectories (shown in red, blue, and yellow) are plotted in each case. 
     }
    \label{fig:Event_based}
\end{figure*}

Fig.~\ref{fig:2D_asym} shows a two-dimensional configuration space with a diagonal wall. The initial state of the robot is marked by the red dot and the target region is shown as the green rectangle at the upper-right corner. The path is obtained by running Algorithm~\ref{algo:2} for $N=10,000$ nodes.
Notice that in Fig.~\ref{fig:2D_asym}, there exist two homotopy classes of paths from the initial state to the goal, as shown by A and B; each path resembles the Qbot navigation scenario in Fig.~\ref{fig:Qbot}.

The simulation results reveal that the proposed methods prefer Path B, which has a similar Euclidean length to Path A, but yields less sensing effort as seen in the Qbot simulation. More specifically, Path B allows the covariance to grow freely until the robot approaches the goal region, where a one-time covariance reduction is performed. In contrast, Path A requires covariance reduction multiple times as the passage narrows. 
Note that this strategy can be understood by invoking the optimality of the ``move-and-sense'' strategy for transitioning between two points (Recalling the comment after Theorem~\ref{theo:tria}).

\vspace{-0.5cm}
\subsection{Event-based Control with Acceleration Input}
So far, we have considered $\mathcal{D}_{\text{info}}$ (i.e., the entropy reduction) as the cost of perception without demonstrating the connections between $\mathcal{D}_{\text{info}}$ and more concrete metrics of perception costs (e.g., sensing power or sensing frequency).
In this section, we consider a mobile robot with a noisy location sensor, and demonstrate that the belief path with a small $\mathcal{D}_{\text{info}}$ helps the robot to navigate with less frequent measurements. Once again, consider the environment shown in Fig.~\ref{fig:2D_mult}. As in \eqref{eq:euler3}, assume 
the robot is a point mass with acceleration input. Denoting by $[x_{1,k} \; x_{2,k}]^\top$ and  $[v_{1,k} \; v_{2,k}]^\top$ the position and the velocity of the robot, the dynamics are described as 
\begin{equation*}
   \begin{bmatrix}
    \bx_{1,k+1} \\ \bx_{2,k+1} \\ {\bv}_{1,k+1} \\ {\bv}_{2,k+1}
    \end{bmatrix} \!=\! 
    \begin{bmatrix}
    I_2 & \Delta t  I_2 \\
    0_2 & I_2
    \end{bmatrix}
    \begin{bmatrix}
    \bx_{1,k} \\ \bx_{2,k} \\  {\bv}_{1,k} \\ {\bv}_{2,k}
    \end{bmatrix} \!+\!
    \begin{bmatrix}
    0 \\ 0 \\ a_{1,k} \\ a_{2,k}
    \end{bmatrix} \Delta t +  \bw_k,~\bw_k \sim \mc N(0, \|\Delta t\| W),
\end{equation*}
where acceleration $a_k = [a_{1,k} \; a_{2,k}]^\top$ is the control input with $\Delta t = \frac{1}{30}$ and  $W = {\rm diag} (10^{-4}, 10^{-4},0,0)$. The robot can observe its position by making noisy measurements of its position as 
\begin{equation}
\label{eq:event_measurement}
\by_k = \begin{bmatrix}
\by_{1,k} \\ \by_{2,k}
\end{bmatrix} = \begin{bmatrix}
\bx_{1,k} \\ \bx_{2,k}
\end{bmatrix} + \bz_t, \quad \bz_t \sim \mathcal{N} (0, Z).
\end{equation}
Two values of $Z=10^{-3} I_2$ and $Z=10^{-4} I_2$ are considered to model  moderate and high precision measurements, respectively.
The robot uses a Kalman filter (KF) to obtain the  estimation $\hat{\bx}_k \sim (\hat{x}_k, P_{k}^{\text{KF}})$. To follow the reference covariance $\{P_{k}^{\text{ref}}\}_{k=0}^{N}$ with infrequent measurements we adopt an event-based sensing strategy. In particular, the robot performs a measurement \eqref{eq:event_measurement} only when the confidence ellipse corresponding to $(\hat{x}_k, P_{k}^{\text{KF}})$ is not contained in the planned ellipse corresponding to $(x_k^{\text{ref}}, P^{\text{ref}}_k)$ under a fixed confidence parameter $\chi^2$.
The reference control input to follow the reference trajectory is generated using a linear quadratic tracker for the nominal speed of 0.1\,m/s. An LQG controller is used in the path following. 
Fig.~\ref{fig:Event_based} shows three sample trajectories of $[{x}_{1,k} \; {x}_{2,k}]^\top$ obtained for each of the scenarios with $\alpha = 0.1, 0.3$ and $0.7$. Fig.~\ref{fig:num_meas_di} shows the number of measurements  with different values of $\alpha \in \{0.1, \dots, 1\}$, averaged over $500$ sampled trajectories. Fig!\ref{fig:num_meas_di} illustrates that the robot following the path generated for higher $\alpha$ performs fewer measurements. Note that the non-monotonicity observed in Fig.~\ref{fig:num_meas_di} is primarily due to the stochastic nature of the RRT* algorithm. Fig.~\ref{fig:num_meas_di} also shows that the number of measurements can be reduced by using sensors with a higher precision.
\begin{figure}[t!]
\vspace{-0.5cm}
    \centering
    \subfloat[The number of measurements using moderate precision sensor $Z=10^{-3} I_2$.   ]
    {\includegraphics[trim = 0.1cm 1cm 1cm 0.0cm, clip=true, width=0.4\columnwidth]{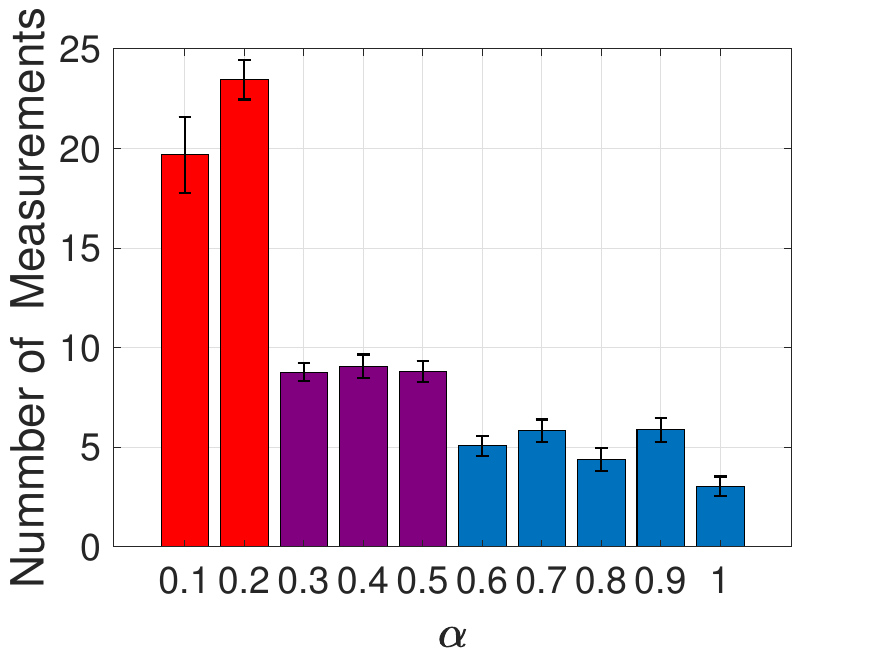}
    \label{fig:low_prec_di}} \quad 
    \subfloat[The number of measurements using high precision sensor $Z=10^{-4} I_2$.]
    {\includegraphics[trim = 0.1cm 1cm 1cm 0.0cm, clip=true, width=0.4\columnwidth]{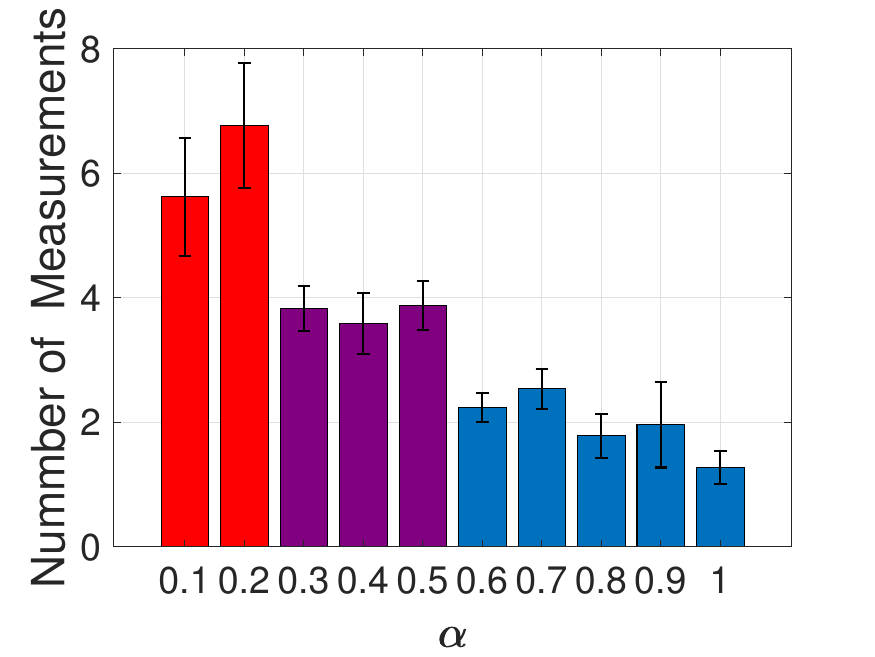}
    \label{fig:high_prec_di}}
    \caption{The  number of required measurements for a double integrator robot using event-based LQG controller. The results are averaged over $500$ randomly generated paths.}
    \label{fig:num_meas_di}
\end{figure}
\vspace{-0.4cm}
\subsection{Path Following with Landmark Selections}
Previous subsection demonstrated a correlation between $\alpha$ and the frequency of sensing actions. In this subsection,  we consider a scenario where multiple sensors are available to the robot. 
We show the proposed planning strategy is effective to reduce the number of sensors that must be activated simultaneously within a single time step.
Specifically, we consider a scenario in which a robot follows a path generated by Algorithm~\ref{algo:2} while localizing its position by an omnidirectional camera that provides the relative angle between itself and obstacles (which also serve as landmarks).
We show the number of measured obstacles during the path following phase decreases by increasing $\alpha$ in the path planning phase.

The state of the robot at time step $k$ comprises the 2-D position $[x_k~y_k]^\top$ and the orientation $\theta_k$.
The dynamics of the robot are governed by the unicycle model perturbed with a Gaussian i.i.d. noise
\begin{equation} \label{eq:uni_cycle}
   \begin{bmatrix}
    \bx_{k+1} \\ \by_{k+1} \\ {\bm{\theta}}_{k+1}
    \end{bmatrix} \!=\! 
    \begin{bmatrix}
    \bx_k \\ \by_k \\  \bm{\theta}_{k}
    \end{bmatrix} \!+\!
    \begin{bmatrix}
    v_k \cos{{\bm \theta}_k} \\ v_k \sin{{\bm{\theta}}_k} \\ \omega_k
    \end{bmatrix} \Delta t +  \bw_k,~\bw_k \sim \mc N(0, W),
\end{equation}
with the velocity and angular velocity input $u_k = [v_k~\omega_k]^\top$.
In this simulation, we set $\Delta t = \frac{1}{30}$s and $W = {\rm diag}(\Delta t\times10^{-4}, \Delta t\times10^{-4}, \frac{1}{100}\times\frac{\pi}{180})$.
The relative angles between the center of obstacles and the robot are extracted via the computer vision techniques \cite{Lowry16_visual_place} and a camera model \cite{Kawai11_panorama_cam}.
The measurement model can be expressed as
\begin{equation} \label{eq:sens_model_sim}
    \bm{y}_{k} =  
    \begin{bmatrix}
    \arctan{\left( \frac{m_{1,y}-\bm{y}_k}{m_{1,x}-\bm{x}_k} \right)} - \bm{\theta}_k
    \\ \vdots \\
    \arctan{\left( \frac{m_{M,y}-\bm{y}_k}{m_{M,x}-\bm{x}_t} \right)} - \bm{\theta}_k
    \end{bmatrix}
    + \bv_k,~ \bv_k \sim \mc N(0, \hat V),
\end{equation}
with positions of known obstacles with known positions $m_j = [m_{j,x}~m_{j,y}]^\top$ for  $j \in \{1,\ldots,M\}$, 
where $\hat V =  {\rm diag}(\{\hat{V}_{j}\}_{j \in \{1,\ldots,M\}})$ is the noise level of the sensors in case the robot decides to measure all obstacles.
$\hat V_i = 0.305$ for all landmarks, for which the standard variance is 10\,deg.
\begin{figure}[t]
\vspace{-0.5cm}
    \centering
    \subfloat[Planned path with $\alpha = 0.4$]
    {\includegraphics[trim = 0.1cm 0cm 1cm 0.3cm, clip=true, width=0.4\columnwidth]{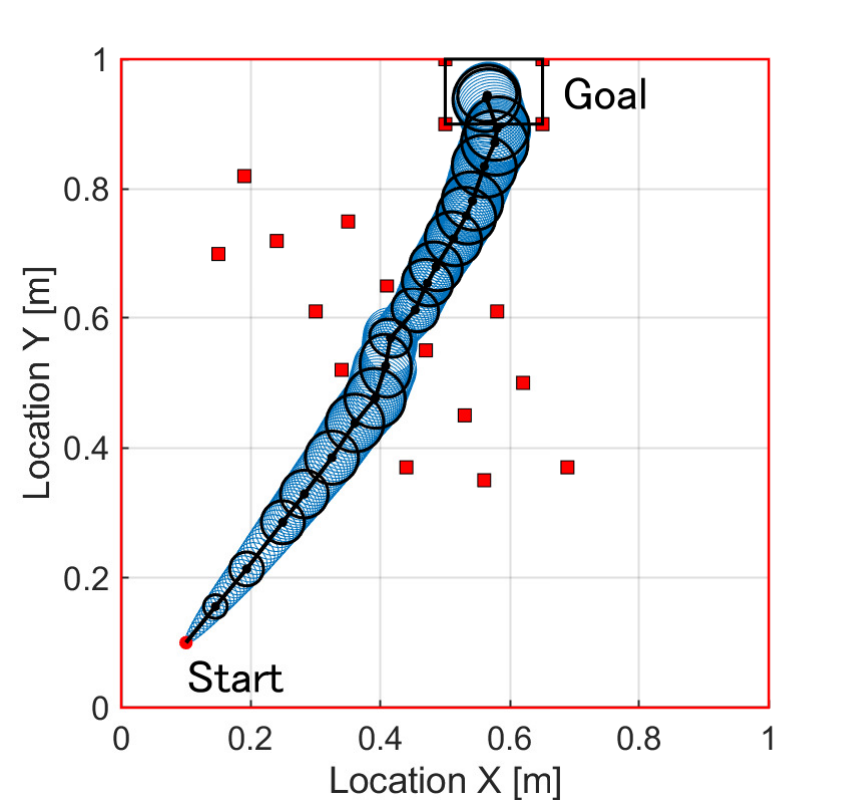}
    \label{fig:plan_al04}} \quad 
    \subfloat[Planned path with $\alpha = 1.6$]
    {\includegraphics[trim = 0.1cm 0cm 1cm 0.3cm, clip=true, width=0.4\columnwidth]{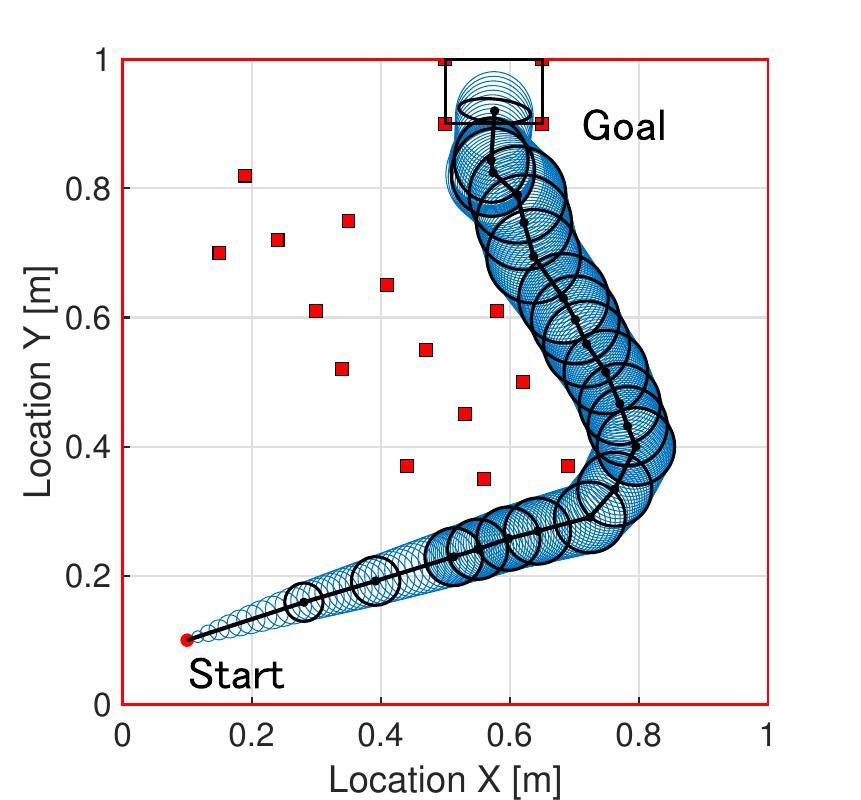}
    \label{fig:plan_al16}} \\
    \subfloat[Average number of measured landmarks when following the path shown in Fig.~\ref{fig:path_follow}(a)]
    {\includegraphics[trim = 0.1cm 0cm 1cm 0.2cm, clip=true, width=0.4\columnwidth]{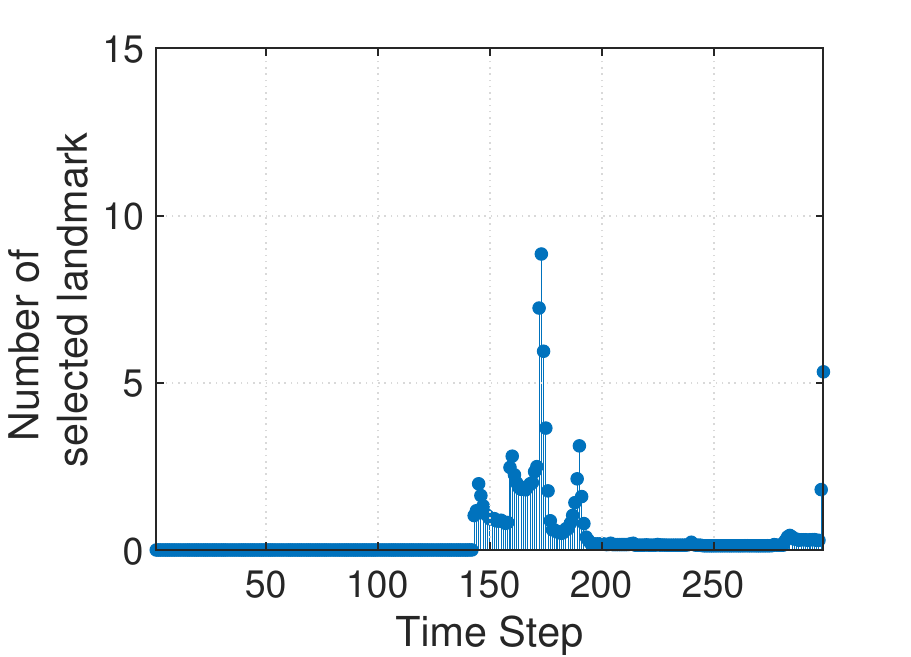}
    \label{fig:num_lm_al04}} \quad 
    \subfloat[Average number of measured landmarks when following the path shown in Fig.~\ref{fig:path_follow}(b)]
    {\includegraphics[trim = 0.1cm 0cm 1cm 0.2cm, clip=true, width=0.4\columnwidth]{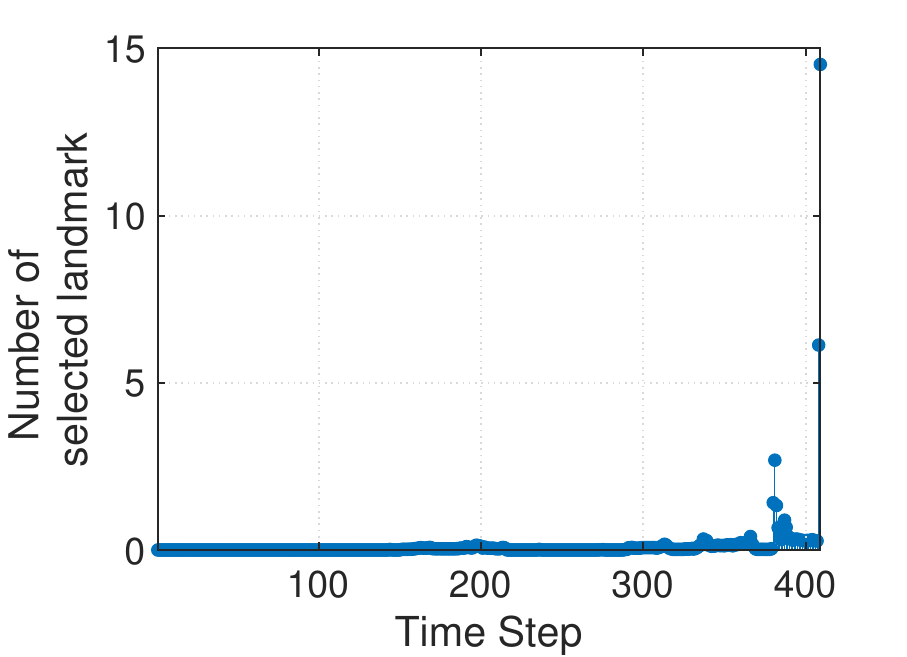}
    \label{fig:num_lm_al16}} 
    \caption{Simulation results in the path following and self-localization scenario. (a) and (b) show the path generated by Algorithm~\ref{algo:2} with $N=30,000$, $W = 10^{-3} I_2$ and $\alpha = 0.4, 1.6$, respectively. The $18$~obstacles, also utilized as landmarks in the self-localization during the path following, are shown in red rectangles. (c) and (d) illustrate the average number of measured landmarks in the self-localization at each time step for 1000 runs.}
    \label{fig:path_follow}
\end{figure}

The reference trajectories generated by Algorithm~\ref{algo:2} with $N = 30,000$, $W=10^{-3}I_2$, and $\alpha = 0.4, 1.6$ for ${\textup{Pr}}=90\%$ safety are depicted in Fig.~\ref{fig:path_follow}(a) and (b), respectively.
Similarly to the results in Section\,\ref{sec:simulation_alpha}, the smaller $\alpha$ forces the robot to shrink the covariance to go through the narrower region surrounded by obstacles shown in red rectangles.

We  assume that measurements are performed only when the covariance ellipse is not contained in the planned one. 
When measurements are needed, the robot is allowed to observe multiple landmarks within a single time step. Landmarks are selected greedily -- the one that reduces the determinant of the covariance matrix the most is selected one after another until the confidence ellipse is contained in the planned one. The robot follows the reference trajectory with the nominal speed of $0.1$\,m/s.
As the control and estimation schemes, an LQG controller and extended Kalman filter are employed to mitigate the deviation from the reference trajectory.

Figure~\ref{fig:path_follow}(c) and (d) show how many obstacles are measured at each time step of the path following averaged over 1,000 runs with the paths in Fig.~\ref{fig:path_follow}(a) and (b), respectively.
Fig.~\ref{fig:path_follow}(c) reveals that a robot measures many obstacles from time step $k\!=\!140$ to $k\!=\!190$, which corresponds with the period that the robot is required to follow a narrow path passing between obstacles.
Namely, a robot is forced to localize its own position accurately to avoid a collision with obstacles, resulting in navigation with many selected landmarks.
In the last few time steps, a robot conducts measurements again to obtain a smaller covariance that can fit in the goal region.
In contrast, Fig.~\ref{fig:path_follow}(d) illustrates the robot following the path with $\alpha \!=\! 1.6$ performs fewer measurements than the one with $\alpha=0.4$. 
\begin{figure}[t]
\vspace{-0.2cm}
\centering
\includegraphics[trim = 0cm 0cm 0cm 0cm, clip=true, width = 0.60\columnwidth]{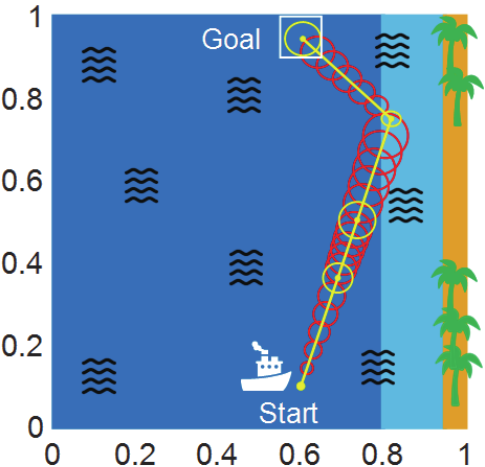}
\caption{Simulation results generated by Algorithm~\ref{algo:2} for $N=10,000$ nodes and $W=5\times 10^{-4} I_2$ for coastal navigation. Measurement noise in dark blue and light blue are $V=10^{-3} I_2$ and $V=10^{-5}I_2$, respectively. Confidence ellipses representing ${\textup{Pr}}=90 \%$ certainty regions.}
\label{fig:sens_cons}
\end{figure}

\vspace{-0.5cm}
\subsection{Coastal Navigation}
As discussed in Section~\ref{sec:sensor_constraint}, sensor constraints can be incorporated into the proposed algorithms. In this subsection, we demonstrate that the proposed algorithm with appropriate sensor constraints can reproduce a solution that is widely known as coastal navigation \cite{roy1999coastal}. Consider a ship that is able to localize itself more accurately when the coastline is observable and less accurately otherwise. A sample environment is depicted in Fig.~\ref{fig:sens_cons}, where the coastline is only observable in the light blue areas, not in the dark blue area. The dynamic and measurement models of the ship are assumed to be 
\begin{subequations}
\begin{align}
\bx_{k+1} &=\bx_{k}+u_k+ \bw_k, \quad \bw_k\sim \mathcal{N}(0,\|u_k\|W),\\
\by_{k} &= \bx_k+ \bv_k, \quad \bv_k \sim \mathcal{N}(0, V),
\end{align}
\end{subequations}
where the state $\bx$ is the 2-D position of the ship, the control input $u$ is its velocity, and $\by$ is the measured value. Here, $V= 10^{-3} I_2$ in the dark blue region and $V= 10^{-5} I_2$ in the light blue region. The simulations are carried out with $W = 5\times 10^{-4} I_2$, where the confidence ellipses correspond to ${\textup{Pr}}=90\%$ safety.

Fig.~\ref{fig:sens_cons} visualizes the result of the simulation using Algorithm~\ref{algo:1} with function $\textsc{\fontfamily{cmss}\selectfont FeasCheck2}$.
Since accurate location data is unavailable in the dark region, the direct move from the start point to the goal region (which entirely lies in the dark region) results in an unacceptably large covariance at the end. The result plotted in Fig.~\ref{fig:sens_cons} shows a feasible solution; the ship visits the light blue region to make a high-precision measurement shortly before moving toward the goal.
\begin{figure}[t]
\vspace{-0.5cm}
    \centering
    \subfloat[$\alpha=0.2$.]
    {\includegraphics[trim = 0.1cm 1cm 1cm 0.3cm, clip=true, width=0.4\columnwidth]{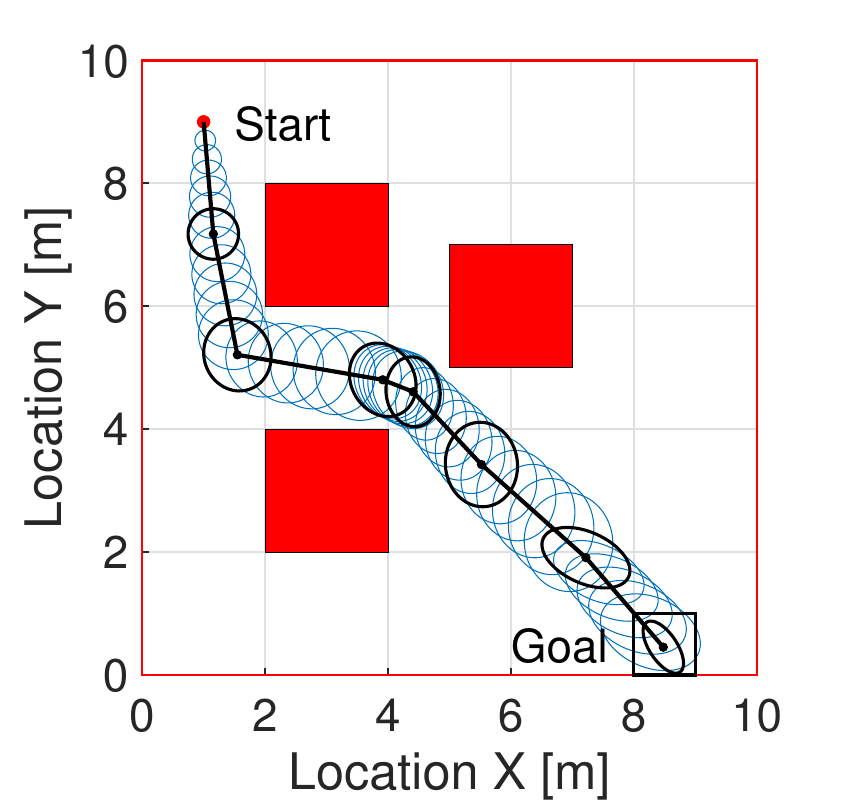}
    \label{fig:quad_traj_low}} \quad 
    \subfloat[ $\alpha=2.0$.]
    {\includegraphics[trim = 0.1cm 1cm 1cm 0.3cm, clip=true, width=0.4\columnwidth]{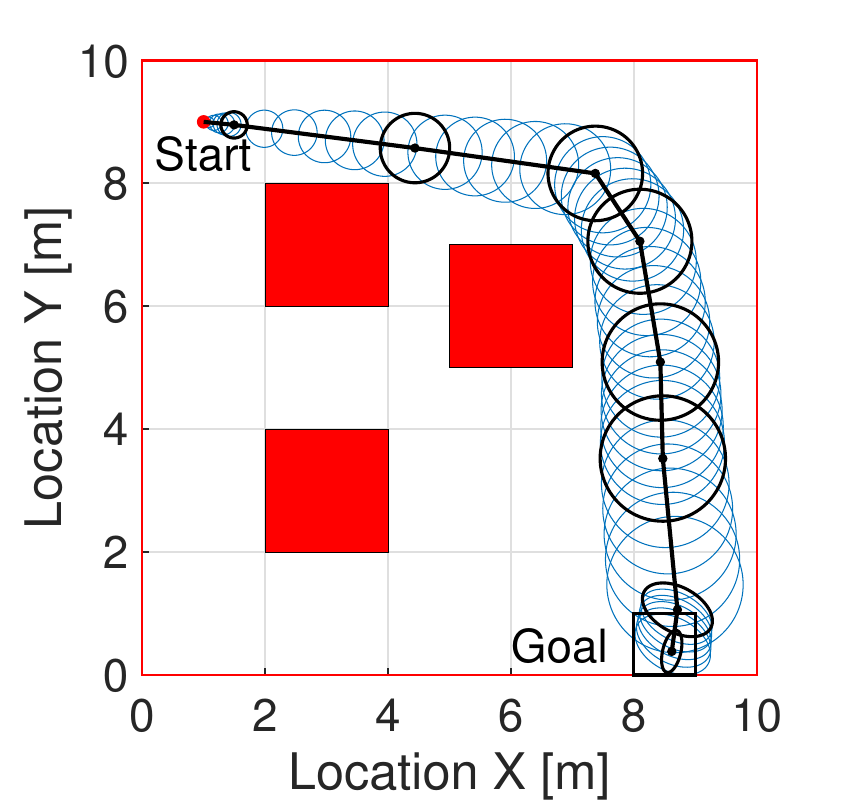}
    \label{fig:quad_traj_high}}
    \caption{ The path generated by Algorithm~\ref{algo:2} for $N=10,000$ and  $W= 0.1I_2$ in an obstacle-cluttered environment used for quadrotor simulation in Section~\ref{sec:quad}. Confidence ellipses representing ${\textup{Pr}}=90 \%$ certainty regions.}
    \label{fig:quad_traj}
\end{figure}
\begin{figure*}[t!]
    \centering
    \subfloat[ Horizontal position of CM (X-Y plane) of the quadrotor following the path generated by $\alpha=0.2$.]
    {\includegraphics[trim = 0.1cm 0cm 0.91cm 0.3cm, clip=true, width=0.55\columnwidth]{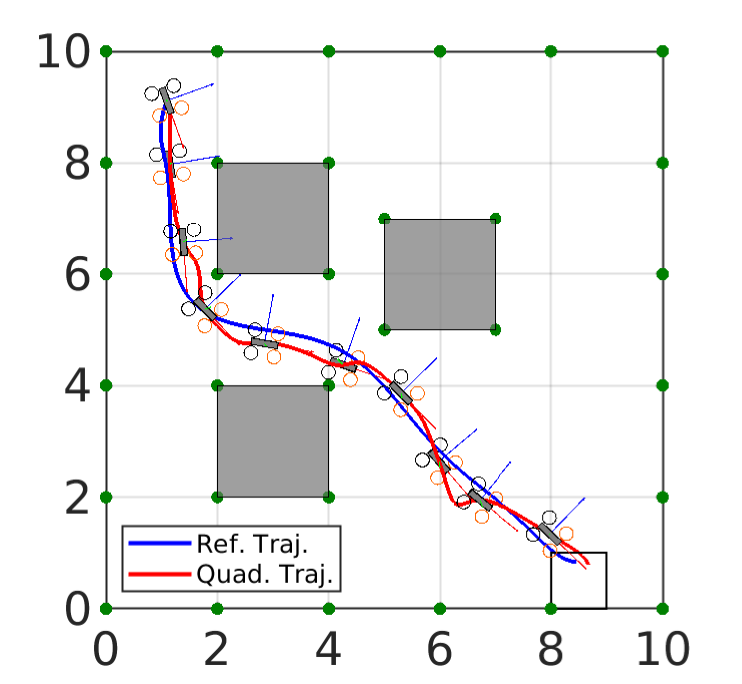}
    \label{fig:sample_XY_02}}  
    \quad
    \subfloat[ Horizontal position of CM (X-Y plane) of the quadrotor following the path generated by $\alpha=2.0$.]
    {\includegraphics[trim = 0.1cm 0cm 0.91cm 0.3cm, clip=true, width=0.55\columnwidth]{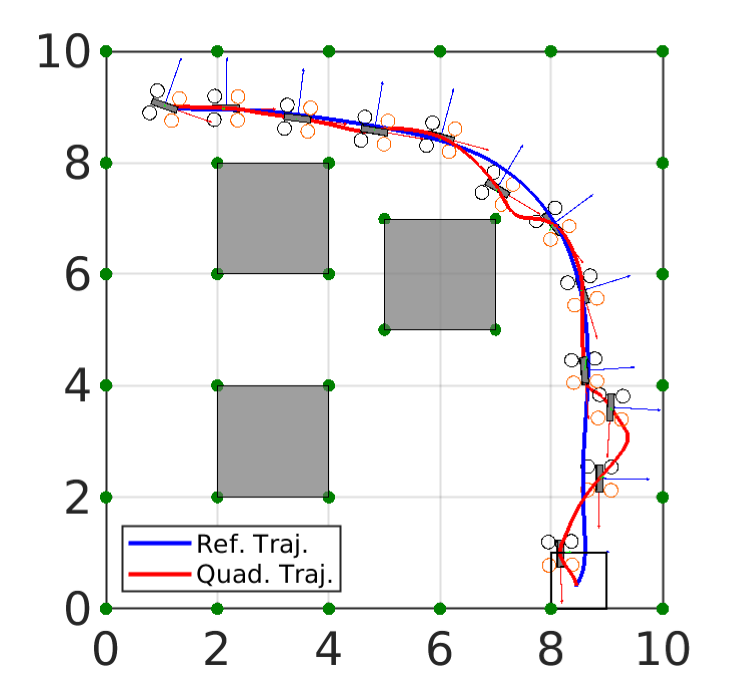}
    \label{fig:sample_XY_2}}
    \quad
    \subfloat[Snapshots of the quadrotor simulation conducted in the Gazebo simulator. The small green box shows the goal region.]
    {\raisebox{6ex} 
     {\includegraphics[trim = 0.0cm 0cm 0cm 0cm, clip=true, width=0.6\columnwidth]{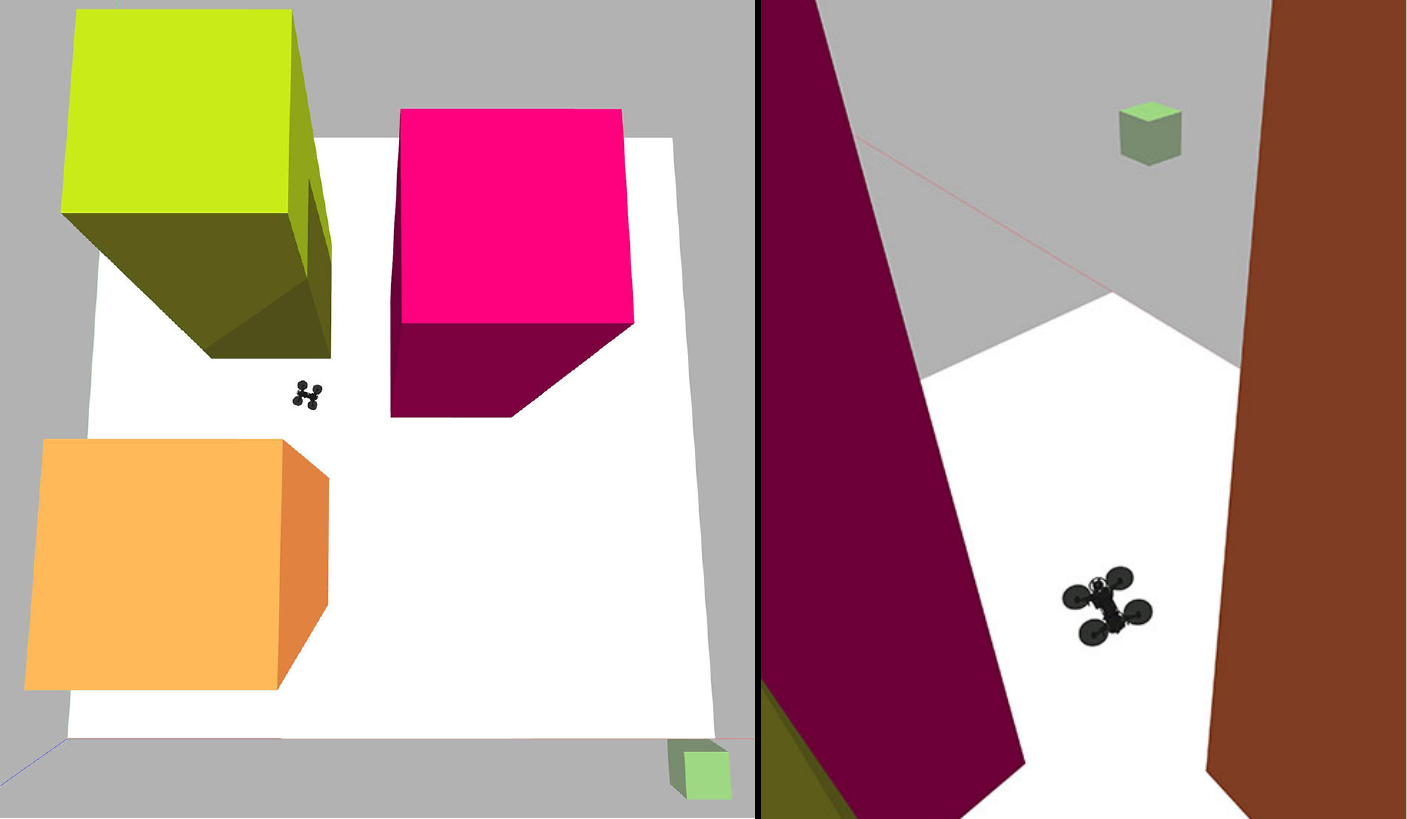}}
    \label{fig:sample_Zt}}
     \caption{ Smoothed reference trajectories (blue), sampled paths (red) of the quadrotor, and snapshots of the simulation in Gazebo.}
    \label{fig:quad_sample}
\end{figure*}
\vspace{-0.3cm}
\subsection{Quadrotor}
\label{sec:quad}
To further demonstrate the effectiveness of the proposed path planning strategy for reducing sensing costs, this subsection considers the problem of navigating a 6 DoF quadrotor. The dynamic models of the quadrotor, rotors' thrusts and response times, motor dynamics, and the aerodynamic effects are adopted from \cite{hoffmann2007quadrotor}. The quadrotor is equipped with an IMU accelerometer with navigation grade that measures acceleration and gyro rate. In addition, a 4K camera mounted on the quadrotor allows it to identify the 3-D bearing of the visual features captured in the camera frame. 
The unscented Kalman filter is deployed for state estimation. The source codes used for simulating the quadrotor, the IMU unit, the 4K camera, and the state estimation unit are accessible at \href{https://github.com/PeteLealiieeJ/EBSHQSim}{https://github.com/PeteLealiieeJ/EBSHQSim}.

The IMU suffers from a drift (i.e., the estimation error accumulates over time), and hence the quadrotor relies on camera data to successfully follow a path. However, due to the limited memory and computational resources, the quadrotor cannot process the captured image on board. This limitation makes the quadrotor send the visual data to the base station. Then, at the base station, the 3-D bearing information of the visual features is extracted from the transmitted image while utilizing the stored map information and sent back to the quadrotor. The quadrotor seeks to minimize the frequency of communication with the ground station without compromising safety. We demonstrate this objective can be achieved by the proposed framework. Fig.~\ref{fig:quad_traj} shows a $10\, \rm m \times 10\, \rm m$ environment with obstacles and the belief paths generated by Algorithm~\ref{algo:2} for $\alpha =0.2$ and $\alpha=2$, $W=0.1 I_2$, and ${\textup{Pr}}=90\%$ safety level.

To obtain a nominal reference, the paths generated by Algorithm~\ref{algo:2} are smoothed by fitting a 9-th degree polynomial to them. The quadrotor follows the reference with the nominal velocity of $1\,\rm{m/s}$ and it is commanded at $200\,\rm{Hz}$. The quadrotor we consider is an under-actuated system . In this simulation, the 3D position of the center of the mass (CM) and the yaw angle are controlled with the aid of a PID controller. The quadrotor communicates with the ground station if and only if the error covariance does not fully reside inside the confidence ellipse planned by Algorithm~\ref{algo:2}. Fig.~\ref{fig:quad_sample}\subref{fig:sample_XY_02} and \subref{fig:sample_XY_2} demonstrate the smoothed paths and sample trajectories, where the position of the visual features are shown by green circles. Fig.~\ref{fig:quad_sample}\subref{fig:sample_Zt} depicts snapshots of the simulation in Gazebo environment. A video of the simulation is uploaded to \href{https://youtube.com/shorts/Y8VbUfKAElU?feature=share}{https://youtube.com/shorts/Y8VbUfKAElU?feature=share} to provide further details. Fig.~\ref{fig:meas_dist} depicts the distribution of the required number of communications for 500 runs. It shows a significant drop (from approximately $230$ to $120$ on average) in communications during the navigation between $\alpha=2.0$ and $\alpha=0.2$. 
\begin{figure}[t!]
\vspace{-0.7cm}
    \centering
    \subfloat[$\alpha=0.2$]
    {\includegraphics[trim = 0.1cm 0cm 0.7cm 0.4cm, clip=true, width=0.45\columnwidth]{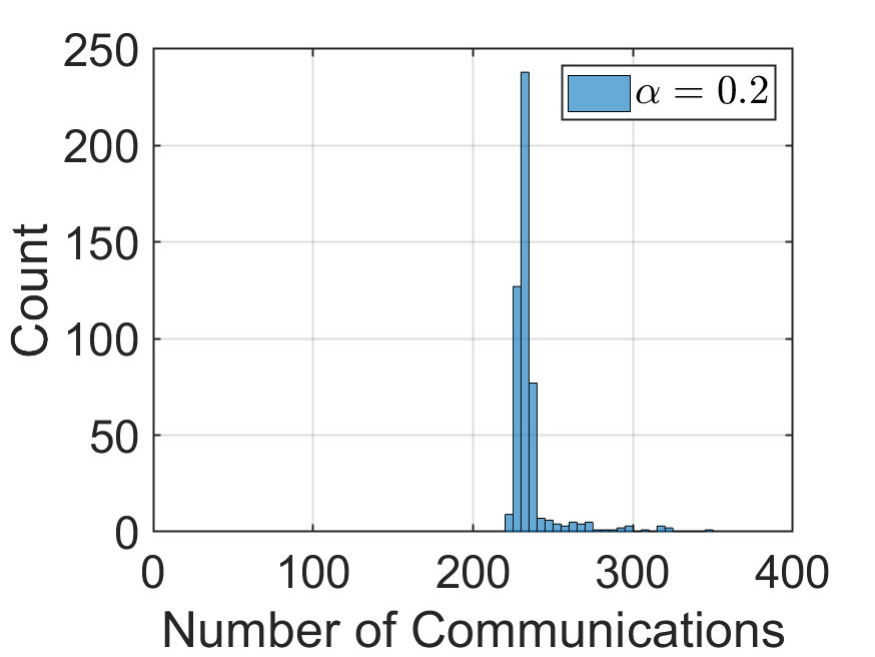}
    \label{fig:meas_dist_02}} \quad 
    \subfloat[$\alpha=2.0$]
    {\includegraphics[trim = 0.1cm 0cm 0.7cm 0.4cm, clip=true, width=0.45\columnwidth]{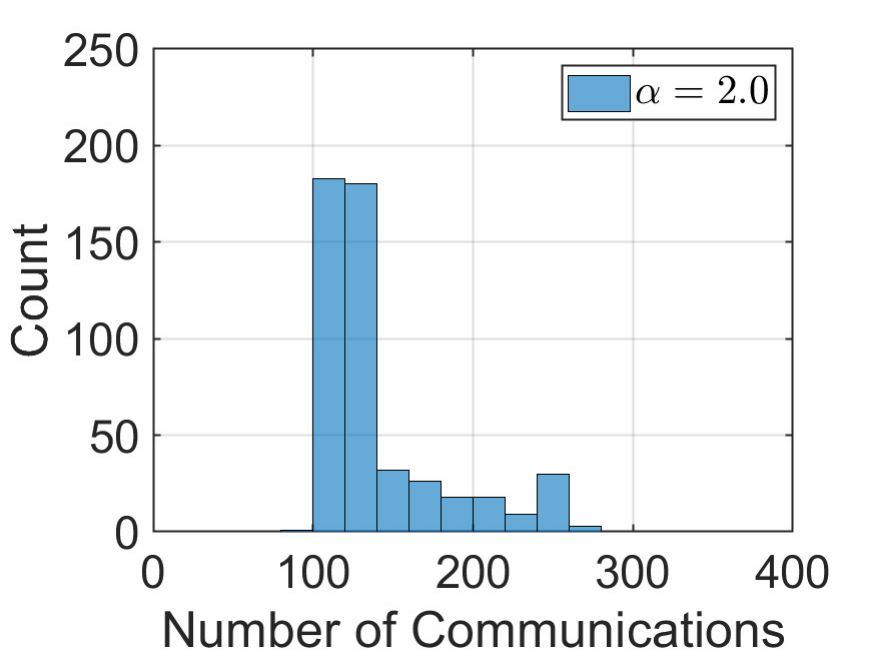}
    \label{fig:meas_dist_2}}
    \caption{Histograms of required number of communications with the ground station for the reference trajectories for $\alpha =0.2$ and $\alpha = 2.0$ shown in Fig.~\ref{fig:quad_traj}.}
    \label{fig:meas_dist}
\end{figure}

\begin{figure}[ht!]
\vspace{-0.7cm}
    \centering
    \subfloat[Entropy with $\alpha\!=\!0.05$.]
    {\includegraphics[trim = 0.3cm 0cm 1.20cm 0.30cm, clip=true, width=0.4\columnwidth]{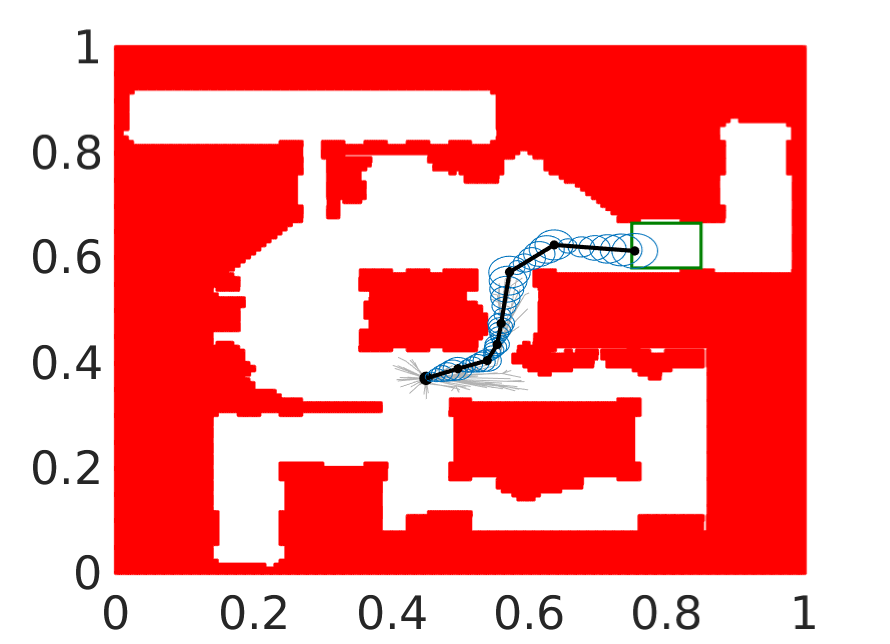}
    \label{fig:Ent_low_alpha}} \quad 
    \subfloat[Entropy with $\alpha\!=\!0.2$.]
    {\includegraphics[trim = 0.3cm 0cm 1.20cm 0.30cm, clip=true, width=0.4\columnwidth]{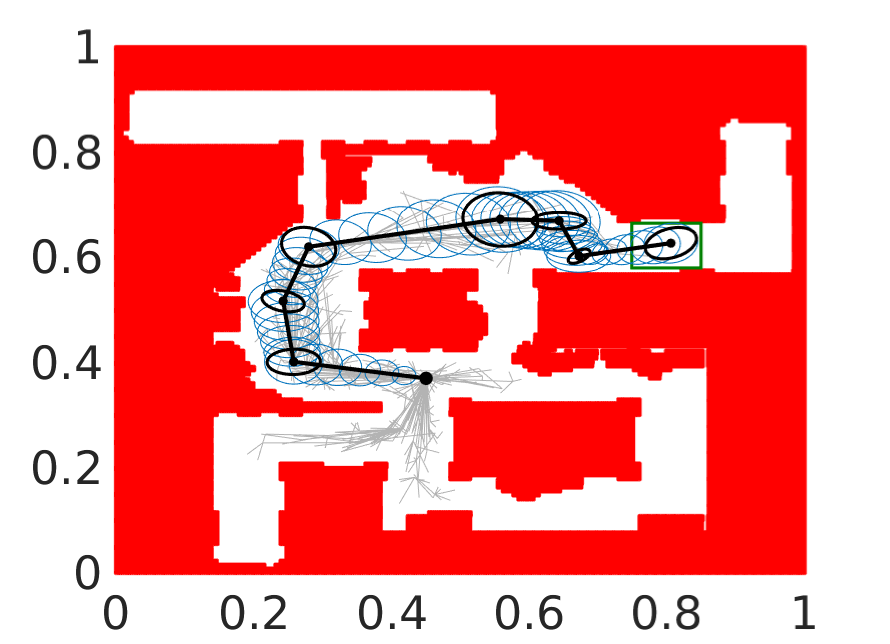}
    \label{fig:Ent_high_alpha}} \quad
    \\
    \subfloat[Wasserstein distance with $\alpha\!=\!200$.]
    {\includegraphics[trim = 0.3cm 0cm 1.20cm 0.30cm, clip=true, width=0.4\columnwidth]{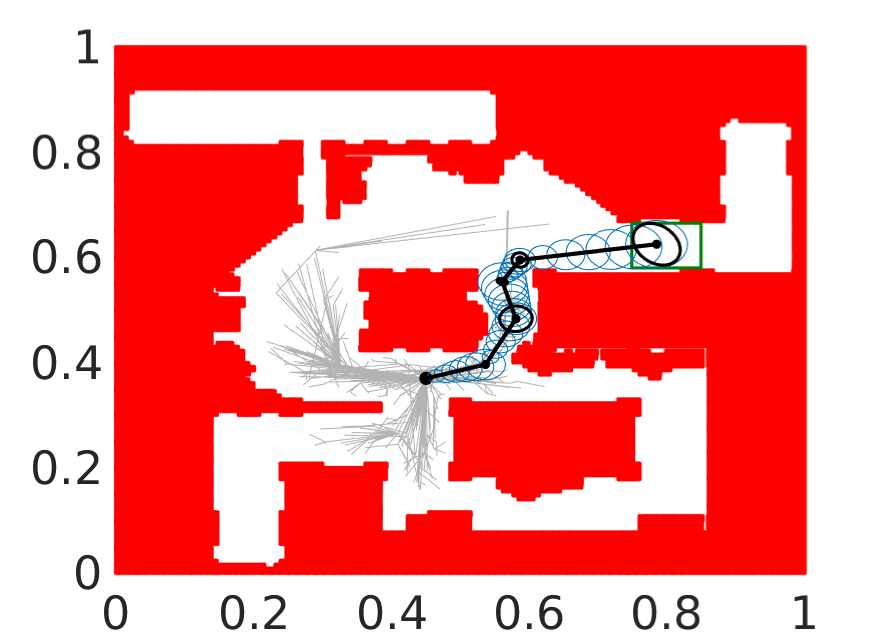}
    \label{fig:Wass_low_alpha}} \quad
    \subfloat[Wasserstein distance with $\alpha\!=
    \!2000$.]
    {\includegraphics[trim = 0.3cm 0cm 1.20cm 0.30cm, clip=true, width=0.4\columnwidth]{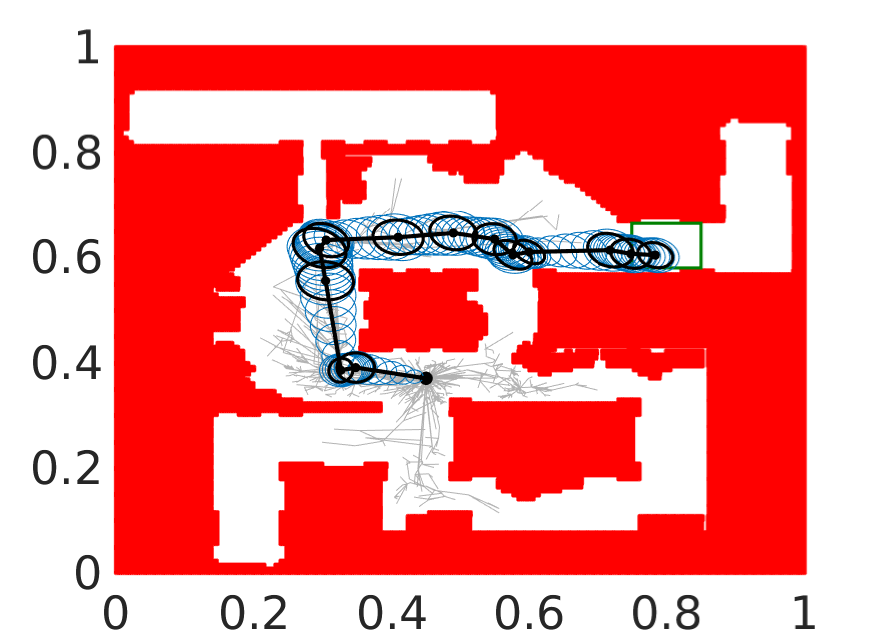}
    \label{fig:Hell_high_alpha}}\\
    \subfloat[Hellinger distance with $\alpha\!=\!0.3$.]
    {\includegraphics[trim = 0.1cm 0cm 1.20cm 0.30cm, clip=true, width=0.4\columnwidth]{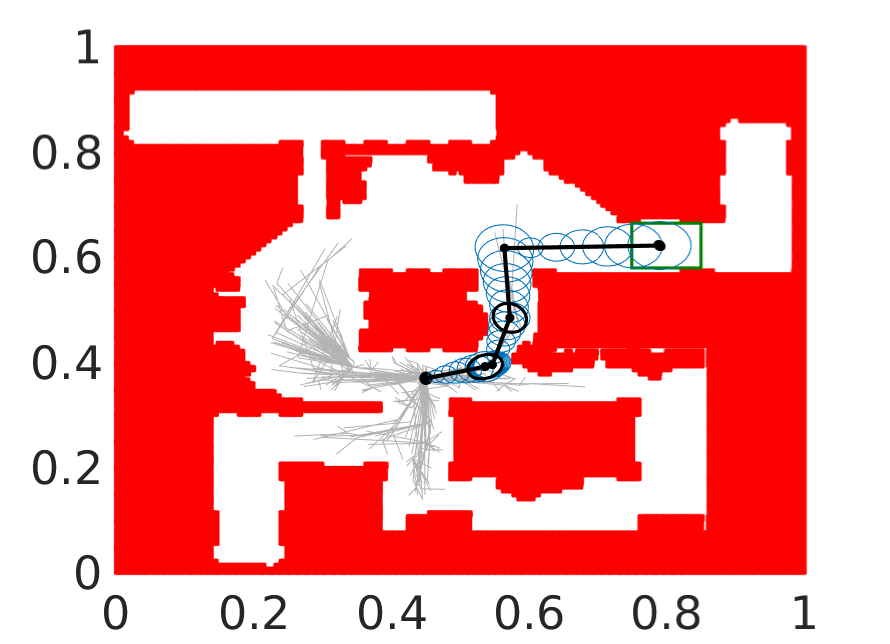}
    \label{fig:Hell_low_alpha}} \quad
    \subfloat[Hellinger distance with $\alpha\!=\!35$.]
    {\includegraphics[trim = 0.1cm 0cm 1.20cm 0.30cm, clip=true, width=0.4\columnwidth]{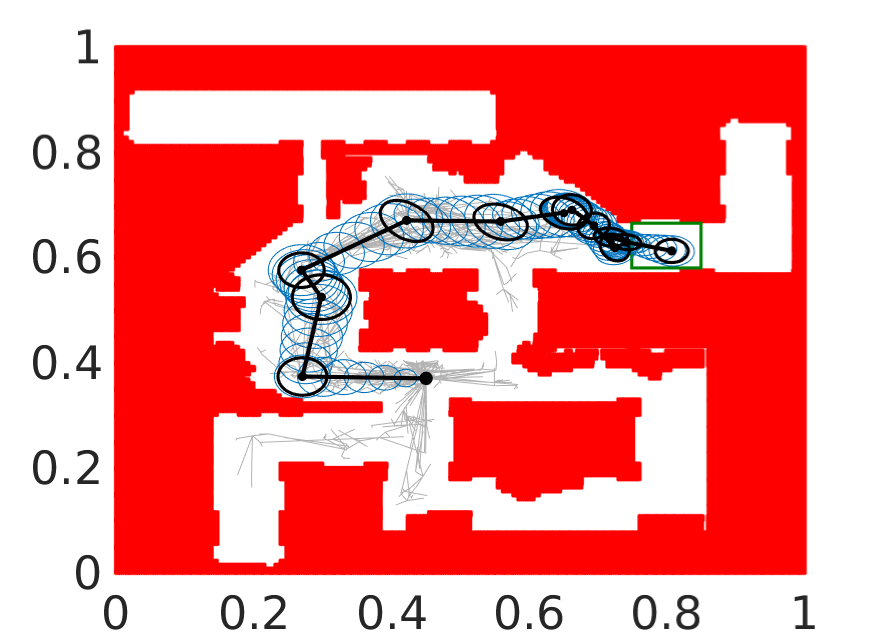}
    \label{fig:Wass_high_alpha}}
    \caption{
    The reference paths generated by different choices of information costs in Algorithm~\ref{algo:1} for different values of $\alpha$ with $N=1000$ and $W= 2.0\times 10^{-3} I_2$. Green rectangle shows the target region, and the ellipses represent $\textup{Pr}=90\%$ certainty regions.} 
\label{fig:metric_compare}
\end{figure}
\vspace{-0.3cm}
\subsection{Comparison between Different Information Metrics}
\label{subsec:compare_metric}
As discussed at the end of Subsection~\ref{subsec:metric}, the proposed methodology and algorithms are capable of incorporating different choices of information metrics. Fig.~\ref{fig:metric_compare} shows the results of using Algorithm~\ref{algo:1} for information cost defined as our proposed metric \eqref{eq:info_gain1}, Wasserstein distance \eqref{eq:info_wass}, and Hellinger distance \eqref{eq:info_Hell} and different values of $\alpha$, in ``It\_backalley\_g" map from DA2 maps
in \cite{sturtevant2012benchmarks}.
In all three cases, we observed that the parameter $\alpha$ plays similar roles, and in this regard we did not observe qualitative differences due to the choice of different metrics in this particular simulation environment. More extensive comparisons (both theoretical and experimental) between these metrics are left for future work.  
\vspace{-0.2cm}
\section{Conclusion and Future Work}
\label{sec:conclusion}
In this paper, we proposed an information-geometric method to generate a reference path that is traceable with moderate sensing costs by a mobile robot navigating through an obstacle-filled environment.
In Section~\ref{sec:prelim}, we introduced a novel distance metric on the Gaussian belief manifold which captures the cost of steering the belief state. Based on this distance concept, in Section~\ref{sec:formulation}, we formulated a shortest path problem that characterizes the desired belief path. In Section~\ref{sec:algorithm}, an RRT*-based algorithm is proposed to solve the shortest path problem. A few variations of the algorithm were also proposed to improve computational efficiency and to accommodate various scenarios.
The continuity of the path length function with respect to the topology of total variation was proved, which is expected to be a key step towards the proof of asymptotic optimality of the proposed RRT*-based algorithms.
Section~\ref{sec:simulation} presented simulation results that confirmed the effectiveness of the proposed planning strategy to mitigate sensing costs for self-navigating mobile robots in several practical scenarios. 

There are several directions to explore in the future:
\begin{itemize}
    \item Computational efficiency of the proposed algorithms can be improved further by incorporating the exiting methodologies, such as informed RRT* or k-d trees.
    \item Asymptotic optimality of Algorithm~\ref{algo:1} is conjectured by the continuity of the path length function (Theorem~\ref{theo:continuity2}) and should be investigated further.
    \item  More extensive comparisons (both theoretical and experimental) should be considered between different choices of the information cost function.
    \item Path ``smoothing'' algorithms similar to \cite{cubuktepe2020scalable} need to be developed to fine-tune the path obtained by  RRT*-based algorithms.
    \item Connections between information theory and sensing costs in practical contexts need to be further investigated. Although our simulation studies empirically confirmed that minimizing the information gain is an effective strategy to mitigate the expected sensing cost in some practical scenarios, theoretical reasoning for these results needs to be developed to fully understand the application domain where the proposed method yields practically useful results. 
\end{itemize}
\vspace{-0.6cm}
\appendices
\section{Explicit expression for \texorpdfstring{$\mathcal{D}_{\text{info}}(k)$}{D(k)}}
\label{ap:zero}
\begin{lemma}
\label{lemma:explicit}
Let $[U,\Sigma]$ be the eigen-decomposition of $P_{k+1}^{-1/2}\hat{P}_{k+1} P_{k+1}^{-1/2}$ i.e. $U \Sigma U^\top=P_{k+1}^{-1/2}\hat{P}_{k+1} P_{k+1}^{-{1}/{2}}$, where $\Sigma= \textup{diag} (\sigma_1, \dots, \sigma_n) \succeq 0$  and $U$ is unitary matrix.  Then, $Q^*= P_{k+1}^{{1}/{2}} U S^* U^\top  P_{k+1}^{{1}/{2}}$ is the optimal solution of \eqref{eq:d_info_general}, where $S^*:=\rm diag (\min \{1, \sigma_1\}, \dots, min\{1, \sigma_n\})$.
\end{lemma}
\begin{proof}
The term $(1/2) \log\det \hat{P}_{k+1}$ is a constant in \eqref{eq:d_info_general} and thus $Q^*$, the optimal solution for \eqref{eq:d_info_general0},  can be computed as 
\begin{equation}
\label{eq:explicit_one}
\begin{split}
Q^*=\argmin_{Q_{k+1}\succeq 0} & \quad -(1/2)\log\det Q_{k+1} \\
\text{s.t. } &\quad Q_{k+1} \preceq P_{k+1}, \;\; Q_{k+1} \preceq \hat{P}_{k+1}.
\end{split}
\end{equation}
If we  define a new variable $  R_{k+1} := P_{k+1}^{-{1}/{2}} Q_{k+1} P_{k+1}^{-{1}/{2}}$, problem ~\eqref{eq:explicit_one} can be rewritten as
\begin{equation}
\label{eq:explicit_two}
\begin{split}
R^*=\argmin_{R_{k+1}\succeq 0} & \quad -(1/2)\log\det R_{k+1} \\
\text{s.t. } &\quad R_{k+1} \preceq I, \;\; R_{k+1} \preceq \bar{P}_{k+1},
\end{split}
\end{equation}
where $\bar{P}_{k+1} := P_{k+1}^{-{1}/{2}} \hat{P}_{k+1} P_{k+1}^{-{1}/{2}}$. Using eigen-decomposition, $\bar{P}_{k+1}$ can be written in the canonical form $\bar{P}_{k+1}= U \Sigma U^\top$.
By defining  $S_{k+1}:= U^\top R_{k+1} U$, problem~\eqref{eq:explicit_two} can be cast as 
\begin{equation}
\label{eq:explicit_three}
\begin{split}
S^*=\argmin_{S_{k+1}\succeq 0} & \quad -(1/2)\log\det S_{k+1} \\
\text{s.t. } &\quad S_{k+1} \preceq I, \;\; S_{k+1} \preceq  \Sigma = \rm diag (\sigma_1, \dots,  \sigma_n).
\end{split}
\end{equation}
It is easy to verify that $S^*=\rm diag(\min\{1, \sigma_1\}, \dots, min\{1, \sigma_n\})$ which completes the proof.
\end{proof}
\vspace{-0.5cm}
\section{Proof of triangle inequality}
\label{ap:B}
In what follows, we will establish the following chain of inequalities:
\begin{subequations}
\begin{align}
\nonumber
 &\mathcal{D}(b_1,b_{int})+\mathcal{D}(b_{int},b_{2})
 = \|x_{int}-x_1\|+\|x_2-x_{int}\|\\ \nonumber
         +& \frac{\alpha}{2} \left(\!\begin{array}{cc}
         &\!\!\!\!\!\!\!\!\min_{Q_{1}\succeq 0} \log \det (P_1+\|x_{int}-x_1\|W)-\log \det Q_{1}  \\
              & \text{s.t.} \quad  Q_{1}\preceq P_1+\|x_{int}-x_1\|W, \quad Q_{1} \preceq P_{int} 
         \end{array}\!\!\right)\\ \nonumber 
+& \frac{\alpha}{2} \left(\!\begin{array}{cc}
              &\!\!\!\!\!\!\!\! \min_{Q_{2}\succeq 0} \log \det (P_{int}+\|x_{2}-x_{int}\|W)-\log \det Q_{2}  \\
              & \text{s.t.} \quad  Q_{2}\preceq P_{int}+\|x_2-x_{int}\|W, \quad Q_{2} \preceq P_{2} 
         \end{array}\!\!\!\right) \\ \nonumber
& \overset{(A)}{\geq}  \|x_2-x_1\| \\ \nonumber 
         +& \frac{\alpha}{2} \left(\!\begin{array}{cc}
              & \!\!\!\!\!\!\!\!\!\!\!\!\min_{Q\succeq 0} \log \det (P_1+(\|x_{int}-x_1\|+ \|x_2-x_{int}\|)W)\\
              &-\log \det Q \\
              & \!\!\!\!\!\!\!\text{s.t.} \quad  Q\preceq P_1+(\|x_{int}-x_1\|+ \|x_2-x_{int}\|)W, Q \preceq P_2
         \end{array}\!\!\!\right)\\ \nonumber
& \overset{(B)}{\geq}  \|x_2-x_1\| \\ \nonumber 
         +& \frac{\alpha}{2}\!\left(\!\begin{array}{cc}
              & \!\!\!\!\!\!\!\!\min_{Q\succeq 0} \log \det (P_1\!+\!(\|x_2-x_1\!\|W)\!-\!\log \det Q \\
              & \text{s.t.} \quad  Q\preceq P_1+\|x_2-x_1\!\|W, \quad Q \preceq P_2
         \end{array}\!\!\!\right)\!
         =\!\mathcal{D}(b_1, b_2)
\end{align}
We will prove the inequalities (B) and (A) in Appendix~\ref{subsec:2} and Appendix~\ref{subsec:3}, respectively.   
\end{subequations}
\vspace{-0.4cm}
\subsection{Proof of Inequality (B)}
\label{subsec:2}
Define a function $f: \mathbb{S}^d_{++} \rightarrow \mathbb{R}$ for  $Y \in \mathbb{S}^d_{++}$  as
\begin{equation}
\label{def:f_1}
     \begin{split}
         &f(X):= \min_{Q\succeq 0} \log \det X -\log \det Q\\
         & \quad \quad \quad \text{s.t.} \quad  Q\preceq X, \quad Q\preceq Y.
     \end{split}
 \end{equation}
 Equivalently, $f(X)$ can also be defined as
\begin{equation}
\label{def:f_2}
     \begin{split}
         & f(X):= \min_{R\succeq 0} \log \det X +\log \det R\\
         & \quad \quad \quad \text{s.t.} \quad  R\succeq X^{-1}, \quad R \succeq Y^{-1}.
     \end{split}
 \end{equation}
\vspace{-0.5cm} 
\begin{lemma}
\label{lem:app}
$f(X)$ is a monotone function meaning that if $0 \preceq X_1 \preceq X_2$, then $f(X_1) \leq f(X_2)$. 
\end{lemma}

\begin{proof}
Set $X=X_2$ and let $R^*$ be the minimizer of right-hand-side of (\ref{def:f_2}), i.e., 
$f(X_2)= \log \det X_2 +\log \det R^*$. Introducing $F:=R^*-X_2^{-1}$, we have $R^*=X_2^{-1}+F$ and $f(X_2)=\log \det X_2 +\log \det (X_2^{-1}+F)= \log \det(I+ F^{\frac{1}{2}}X_2F^{\frac{1}{2}})$.
On the other hand, set $X=X_1$  in (\ref{def:f_2}) then $R':= X_1^{-1}+F$ is a feasible point for (\ref{def:f_2}). Namely, $R'= X_1^{-1}+F \succeq X_1^{-1}$ and
\vspace{-0.2cm}
\begin{align*}
  R'& = X_1^{-1}+F
    = X_1^{-1}+R^*-X_2^{-1}
     \overset{(\text{I})}{\geq} R^* 
      \overset{(\text{II})}{\geq} Y^{-1}, 
\end{align*}
where (I) follows from $X_1^{-1} \preceq X_2^{-1}$, and (II) holds since  $R^*$ is a feasible point for (\ref{def:f_2}) for $X=X_1$.
Therefore, $f(X_1) = \log \det X_1 +\log \det R^* 
= \log \det X_1 +\log \det (X_1^{-1}+F)=\log \det(I+ F^{1/2}X_1F^{1/2})
= \log \det(I+ F^{1/2}X_2F^{1/2}) \!=\! f(X_2)$.
\end{proof}
\vspace{-0.1cm}
The inequality (B) is an application of Lemma~\ref{lem:app} with $X_1= P_1+\|x_2-x_1\|W$ and $X_2= P_1+(\|x_{int}-x_1\|+ \|x_2-x_{int}\|)W$, which clearly satisfy $0 \preceq X_1 \preceq X_2$.
\vspace{-0.5cm}
\subsection{Proof of Inequality (A)}
\label{subsec:3}
Let $P_1\succ 0$, $W_1\succeq 0$, $P_2\succ 0$, and $W_2\succeq 0$ be given matrix-valued constants. To complete the proof of (A), we consider
\begin{equation}
\label{def:F_12}
\min_{P_{int}\succeq 0} F_1(P_{int})+F_2(P_{int}),
\end{equation}
where 
\begin{equation}
\label{def:F_1}
     \begin{split}
         &F_1(P_{int}):= \min_{Q_1\succeq 0} \log \det (P_1+W_1) -\log \det Q_1\\
         & \quad \quad \quad \text{s.t.} \quad  Q_1\preceq P_1+W_1, \quad Q_1\preceq P_{int}.
     \end{split}
 \end{equation}
\begin{equation}
\label{def:F_2}
     \begin{split}
         &F_2(P_{int}):= \min_{Q_2\succeq 0} \log \det (P_{int}+W_2) -\log \det Q_2\\
         & \quad \quad \quad \text{s.t.} \quad  Q_2\preceq P_{int}+W_2, \quad Q_2\preceq P_2.
     \end{split}
 \end{equation}
We show optimality is attained by $P^*_{int}= P_1+W_1$. First, we show that the optimality is attained by $P^*_{int} \preceq P_1+W_1$.
\begin{proposition}
\label{prop_a}
There exists an optimal solution for (\ref{def:F_12}) that belongs to the set $\mathbb{P}^{ineq}_{int}:= \{P_{int} \succeq 0: P_{int} \preceq P_1+W_1\}$.
\end{proposition}

\begin{proof}
We show for any optimal solution candidate $P_{int}^* \succeq 0$, there exist an element $ P'_{int}\in \mathbb{P}^{ineq}_{int}$ such that
$F_1(P'_{int})+F_2(P'_{int}) \leq F_1(P_{int}^*)+F_2(P_{int}^*)$.
Set $P_{int}=P^*_{int}$ in (\ref{def:F_12}), and let $Q^*$ be the unique optimal solution, i.e., $F_1(P^*_{int})= \min_{Q_1\succeq 0} \log \det (P_1+W_1) -\log \det Q^*$. Take $P'_{int}=Q^*$ as a new solution candidate. Note that $Q^* \in \mathbb{P}^{ineq}_{int}$ since it is feasible solution for (\ref{def:F_12}) with $P_{int}=P^*_{int}$. Additionally, $Q_1=Q^*$ is a feasible a point for (\ref{def:F_12}) with $P_{int}=P'_{int}$. Therefore, $F_1(P'_{int}) \leq F_1(P^*_{int})$.
Moreover, since $P'_{int}\succeq P^*_{int}$, by Lemma~\ref{lem:app}, we have
$F_2(P'_{int}) \leq F_2(P^*_{int})$.
Finally, these two inequalities yield $F_1(P'_{int})+F_2(P'_{int}) \leq F_1(P_{int}^*)+F_2(P_{int}^*)$.
\end{proof}

\begin{proposition}
\label{prob_b}
$P'_{int}=P_1+W_1$ is the minimizer of (\ref{def:F_12}). 
\end{proposition}
\begin{proof}
Let $P^*_{int} \in \mathbb{P}^{ineq}_{int}$ (i.e., $P^*_{int} \preceq P_1+W_1$) be any solution candidate from Proposition~\ref{prop_a}. We will show that if we pick a new solution $P'_{int}:=P_1+W_1$, then 
\begin{equation}
\label{equ:appb_final}
 F_1(P'_{int})+F_2(P'_{int}) \leq F_1(P_{int}^*)+F_2(P_{int}^*). 
\end{equation}
The constraints in (\ref{def:F_1}) for $P_{int} \!\preceq\!P^*_{init}\!\preceq\! P_1+W_1$ reduces to $Q_1 \preceq P^*_{init}  \preceq P_1+W_1$. Thus, $P^*_{int}$ is the unique minimizer and
\begin{align}
\nonumber
F_1(P^*_{int})&=\log \det (P_1+W_1) -\log \det P^*_{int}\\ \label{equ:appb_final_1}
&= \log \det (P'_{int}) -\log \det P^*_{int}.
\end{align}
Let $Q^*_2$ be the unique minimizer of (\ref{def:F_2}) for $P_{int}=P^*_{int}$. Then,
\begin{align}
\label{equ:appb_final_3}
    F_2(P^*_{int})&=\log \det (P^{*}_{int}+W_2) -\log \det Q_2^*
\end{align}
On the other hand, since $P^*_{int} \preceq P'_{int}$, $Q^*_2$ is a feasible point for (\ref{def:F_2}) with $P_{int}=P^*_{int}$. Hence,
\begin{align}
\label{equ:appb_final_4}
    F_2(P'_{init})\leq \log \det (P'_{int}+W_2) -\log \det Q_2^*.
\end{align}
\end{proof}
\vspace{-0.5cm}
Now, from  $F_1(P'_{int})=0$, (\ref{equ:appb_final_1}), 
(\ref{equ:appb_final_3}), and (\ref{equ:appb_final_4}), we have
\begin{align*}
    &F_1(P^*_{int})+F_2(P^*_{int})- F_1(P'_{int})-F_2(P'_{int})\\
    & \geq \log \det (P'_{int}) -\log \det P^*_{init} \\
    & \quad + \log \det (P^{*}_{int}+W_2) -\log \det Q_2^*\\
    & \quad -\log \det (P'_{int}+W_2) +\log \det Q_2^*\\
    & = \log \det (P^{*}_{int}+W_2)-\log \det P^*_{int}\\
    & \quad  -\log \det (P'_{int}+W_2) + \log \det (P'_{int})\\
    &= \log \det(I+W_2^{\frac{1}{2}} P^{*^{-1}}_{int}W_2^{\frac{1}{2}})- P'^{-1}_{int}W_2^{\frac{1}{2}}) \geq 0,
\end{align*}
since $P^{*^{-1}} \succeq P'^{-1}$. Therefore, (\ref{equ:appb_final}) holds. Inequality (A) is the application of Proposition~\ref{prob_b} with $W_1= \|x_{int}-x_1\|W$ and $W_2=\|x_2-x_{int}\|W$, and standard triangle inequality: $ \|x_2-x_1\|\leq \|x_{int}-x_1\|+\|x_2-x_{int}\|$. 
\vspace{-0.2cm}
\section{Proof of Theorem~\ref{theo:loss-lessmod}}
\label{ap:C}
We prove the existence by  constructing a collision-free lossless chain $\{b'_k\}_{k=0, 1, \dots, K-1}$ from the initial chain $\{b_k\}_{k=0, 1, \dots, K-1}$. The construction is performed in $K-1$ steps, where in the $k$-th step, the belief $b_{k}=(x_k,P_k)$ is shrunk to $b'_k=(x_k,P'_k)$, where $ P'_k \preceq P_k$ and the transition $b_{k-1}$ to $b'_k$ becomes lossless. More precisely, $P'_k$ is selected as the minimizer of  \eqref{eq:def_D} for computing $\mathcal{D}(b_{k-1},b_k)$, where from \eqref{eq:d_info_general1} we have $P'_k \preceq P_k$. The fact that  $P_k$ does not increase after performing a step automatically guarantees that the transitions $b_{k-1}\rightarrow b'_k$ and $b'_{k}\rightarrow b_{k+1}$ are collision-free. (They reside completely inside $b_{k-1}\rightarrow b_k$ and $b_k \rightarrow b_{k+1}$, respectively.) This means that the chain stays collision-free after each step, and in particular the final chain is collision-free.

Next, we show after step $k \in \{1, \dots, K-1\}$, the length of the chain does not increase. Note that at step $k$, the transitions do not change except the transitions to and from the $k$-th belief.
For transition to the $k$-th belief, it is trivial to see $\mathcal{D}(b_{k-1},b_k)=\mathcal{D}(b_{k-1},b'_k)$  as $P'_k$ is the minimizer of \eqref{eq:def_D}. For transition from the $k$-th belief,  we have $\mathcal{D}(b'_k, b_{k+1}) \leq \mathcal{D}(b'_k, b_{k}) + \mathcal{D}(b_k, b_{k+1})$ from the triangle inequality we showed in Theorem~\ref{theo:tria}. From $x'_k=x_k$ and $P'_k \preceq P_k$, it is easy to verify that $\mathcal{D}(b'_k, b_{k})=0$ which yields   $\mathcal{D}(b'_k, b_{k+1}) \leq  \mathcal{D}(b_k, b_{k+1})$. This relation leads to the conclusion that the length of the chain does not increase after step $k$, which completes the proof.
\vspace{-0.4cm}
\section{Proof of Theorem~\ref{theo:continuity2}}
\label{ap:A}
\subsection{Preparation}
\begin{lemma}
\label{lem1}
{For arbitrary $M\in \mathbb{S}^d$ and $N\in \mathbb{S}^d$, if $M \succeq \kappa I$ and  $N \succeq \kappa  I$, then $|\log\det M - \log\det N| \leq (d/\kappa)
\bar{\sigma}(M-N)$.}
\end{lemma}
\begin{proof}
Set $\mathcal{X}=\{X\in \mathbb{S}^d: X \succeq \kappa  I\}$ and define $f: \mathcal{X}\rightarrow \mathbb{R}$ by $f(X)=\log\det X$.
The directional derivative $\nabla_Y f(X)$ of $f(X)$ in the direction $Y$ is given by $\nabla_Y f(X)=\text{Tr}(X^{-1}Y)$.
Suppose $M, N \in \mathcal{X}$ and define $X(t)=tM+(1-t)N, \;\; t \in [0, 1]$.
Since $\nabla_{M-N} f(X(t))=\text{Tr}(X(t)^{-1}(M-N))$, by the mean value theorem, there exists $t\in [0, 1]$ such that $f(M)\!-\!f(N)\!=\!\nabla_{M-N} f(X(t)) \cdot (1-0)
=\text{Tr}(X(t)^{-1}(M-N))$.
Now, $|f(M)-f(N)|\!=\!|\text{Tr}(X(t)^{-1}(M-N))|\!\leq\! \|X^{-1}(t)\|_F \|M-N\|_F \leq d  \bar{\sigma}(X^{-1}(t)) \bar{\sigma}(M-N) \leq d  \bar{\sigma}(\frac{1}{\kappa }I) \bar{\sigma}(M-N) =\frac{d}{\kappa } \bar{\sigma}(M-N)$.
\end{proof}

\begin{lemma}
\label{lem2}
Let $X, Y$ and $\Theta$ be symmetric matrices such that $0\preceq X \preceq \frac{1}{\epsilon}I$ and $0\preceq Y \preceq \frac{1}{\epsilon}I$. Then $\bar{\sigma}(X\Theta X-Y\Theta Y)\leq \frac{3\bar{\sigma}(\Theta)}{\epsilon}\bar{\sigma}(X-Y)$. 
\end{lemma}
\begin{proof}
By assumption, we have $\bar{\sigma}(X-Y)\leq \frac{1}{\epsilon}$. Notice that
\begin{align*}
X\Theta X-Y\Theta Y&=(Y+X-Y)\Theta (Y+X-Y) -Y\Theta Y \\
&=(X-Y)\Theta Y + Y \Theta (X-Y) + (X-Y) \Theta (X-Y).
\end{align*}
Thus, $\bar{\sigma}(X\Theta X\!-\!Y\Theta Y) \!\!\leq\!\! 2\bar{\sigma}(X-Y)\bar{\sigma}(\Theta)\bar{\sigma}(Y)+\bar{\sigma}(X-Y)^2 \bar{\sigma}(\Theta) \!\leq\! \frac{2}{\epsilon} \bar{\sigma}(X-Y)\bar{\sigma}(\Theta)+\frac{1}{\epsilon}\bar{\sigma}(X-Y)\bar{\sigma}(\Theta) \!=\!\frac{3\bar{\sigma}(\Theta)}{\epsilon} \bar{\sigma}(X-Y)$.
\end{proof}
\begin{lemma}
\label{lem3}
Let $X\in \mathbb{S}^d_{\epsilon}$ and $Y\in \mathbb{S}^d_{\epsilon}$ 
for some $\epsilon\!>\!0$, then $\bar{\sigma}(X^{-1}\!-\!Y^{-1})\leq \!\frac{1}{\epsilon^2} \bar{\sigma}(X\!-\!Y)$. 
\end{lemma}
\begin{proof}
By assumption, we have $\bar{\sigma}(X^{-1})\leq \frac{1}{\epsilon}$ and $\bar{\sigma}(Y^{-1})\leq \frac{1}{\epsilon}$. Therefore, $\bar{\sigma}(X^{-1}-Y^{-1})=\bar{\sigma}(X^{-1}(Y-X)Y^{-1}) \leq \bar{\sigma}(X^{-1})\bar{\sigma}(Y^{-1})\bar{\sigma}(X-Y) \leq \frac{1}{\epsilon^2}\bar{\sigma}(X-Y)$.
\end{proof}
\vspace{-0.4cm}
\begin{lemma}
\label{lem4}
Let $X$ and $Y$ be symmetric matrices. Suppose $X \succeq \epsilon_1 I$ and $Y \succeq \epsilon_2 I$ hold for some $\epsilon_1>0$ and $\epsilon_2>0$. Then, $\bar{\sigma}(X^{\frac{1}{2}}-X^{\frac{1}{2}})\leq \frac{1}{\sqrt{\epsilon_1}+\sqrt{\epsilon_2}} \bar{\sigma}(X-Y)$.
\end{lemma}
\begin{proof}
See \cite[Lemma 2.2]{schmitt1992perturbation}.
\end{proof}
\vspace{-0.5cm}
\subsection{Continuity of \texorpdfstring{$c(\gamma)$}{c(g)}  with respect to the topology of total variation}
Consider transitions from $(x_k, P_k)$ to $(x_{k+1}, P_{k+1})$ and from $(x'_k, P'_k)$ to $(x'_{k+1}, P'_{k+1})$.
Assume the following:
\begin{itemize}
\item There exists a positive constant $\rho$ such that $\rho I \preceq P_k,\; \rho I \preceq P_{k+1}, \;  \rho I \preceq P'_k, \  $, and $\rho I \preceq P'_{k+1}$.
\item Perturbations $\Delta x_k:= x'_k-x_k$,  $\Delta x_{k+1}:= x'_{k+1}-x_{k+1}$, $\Delta P_k:= P'_k-P_k$,  $\Delta P_{k+1}:= P'_{k+1}-P_{k+1}$ are bounded by a constant $\delta < \frac{\rho}{4}$ as
\begin{align}
&\|\Delta x_{j}\|\leq \frac{\delta}{\bar{\sigma}(W)}, \; \bar{\sigma}(\Delta P_{j})\leq \delta, \; j\in\{k,k+1\} . \label{eq:deviation_delta}
\end{align}
\item Transition from $(x_k, P_k)$ to $(x_{k+1}, P_{k+1})$ is lossless.
\item Transition from $(x'_k, P'_k)$ to $(x'_{k+1}, P'_{k+1})$ is lossless.
\end{itemize} 
Based on these assumptions, we have 
\begin{align}
&\Big|\mathcal{D}(x'_k, x'_{k+1}, P'_k, P'_{k+1})-\mathcal{D}(x_k, x_{k+1}, P_k, P_{k+1})\Big| \nonumber \\
&\leq \Big|\|x'_{k+1}-x'_k\|\bar{\sigma}(W) - \|x_{k+1}-x_k\|\bar{\sigma}(W) \nonumber \\
&\qquad  +\frac{1}{2}\log \det (P'_k+ \|x'_{k+1}-x'_k\|W) - \frac{1}{2}\log \det P'_{k+1} \nonumber\\
&\qquad  -\frac{1}{2}\log \det (P_k+ \|x_{k+1}-x_k\|W) + \frac{1}{2}\log \det P_{k+1} \Big| \nonumber\\
&\leq \big|  \|x'_{k+1}-x'_k\|- \|x_{k+1}-x_k\| \big| \bar{\sigma}(W) \nonumber\\
&\quad +\frac{1}{2}\Big| \log \det(P'_k+\|x'_{k+1}-x'_k\|W)- \log \det P'_{k+1} \nonumber \\
&\quad\qquad - \log \det(P_k+\|x_{k+1}-x_k\|W)+ \log \det P_{k+1} \Big| \label{eq:continuous1}
\end{align}
Using the triangle inequality $\big| \|x_{k+1}+\Delta x_{k+1}\!-\!x_k-\!\Delta x_k \| - \|x_{k+1}-x_k \| \big| \leq \| \Delta x_{k+1}-\Delta x_k \| \leq \frac{2\delta}{\bar{\sigma}(W)},$
the first term of \eqref{eq:continuous1} can be upper bounded by $\|\Delta x_{k+1}-\Delta x_k\|\bar{\sigma}(W)$.
Writing $\hat{P}_k=P_k+\|x_{k+1}-x_k\|W$, the second term of \eqref{eq:continuous1} can be expressed as
\begin{subequations}
\begin{align}
&\frac{1}{2}\big|\log \det \left(\hat{P}_k+\Delta P_k+(\|x'_{k+1}-x'_k\|-\|x_{k+1}-x_k\|)W\right) \nonumber  \\
&\qquad -\log \det (P_{k+1}+\Delta P_{k+1}) -  \log \det  \hat{P}_k + \log \det  P_{k+1} \big| \nonumber \\
&=\frac{1}{2}\big| \log \det \big(I+\hat{P}_k^{-1/2}\Delta P_k \hat{P}_k^{-1/2}  \nonumber \\
&\hspace{10ex}+(\|x'_{k+1}-x'_k\|-\|x_{k+1}-x_k\|)\hat{P}_k^{-1/2}W\hat{P}_k^{-1/2}\big) \nonumber  \\
&\qquad  - \log \det (I+P_{k+1}^{-1/2}\Delta P_{k+1} P_{k+1}^{-1/2}) \big| \nonumber \\
&\leq \frac{1}{2}\cdot 4d \bar{\sigma}
\big( \hat{P}_k^{-1/2}\Delta P_k \hat{P}_k^{-1/2}-P_{k+1}^{-1/2}\Delta P_{k+1} P_{k+1}^{-1/2} \nonumber \\
&\hspace{8ex}+\left(\|x'_{k+1}-x'_k\|-\|x_{k+1}-x_k\|\right)\hat{P}_k^{-1/2}W\hat{P}_k^{-1/2} \big)  \label{eq:continuous2_a}\\
&\leq 2d \bar{\sigma}\big(\hat{P}_k^{-1/2}\Delta P_k \hat{P}_k^{-1/2}-P_{k+1}^{-1/2}\Delta P_{k+1} P_{k+1}^{-1/2}\big) \nonumber \\
&\qquad +2d \|\Delta x_{k+1}-\Delta x_k\|\bar{\sigma}(\hat{P}_k^{-1/2}W\hat{P}_k^{-1/2}) \nonumber \\
&\leq 2d \bar{\sigma}\big( \hat{P}_k^{-1/2}(\Delta P_k-\Delta P_{k+1}) \hat{P}_k^{-1/2} \nonumber \\
&\hspace{10ex}+\hat{P}_{k}^{-1/2}\Delta P_{k+1} \hat{P}_{k}^{-1/2}-P_{k+1}^{-1/2}\Delta P_{k+1} P_{k+1}^{-1/2}\big) \nonumber \\
&\qquad +2d \|\Delta x_{k+1}-\Delta x_k\|\bar{\sigma}(\hat{P}_k^{-1/2}W\hat{P}_k^{-1/2}) \nonumber \\
&\leq \frac{2d}{\rho}\|\Delta x_{k+1}-\Delta x_k\|\bar{\sigma}(W) + \frac{2d}{\rho}\bar{\sigma}(\Delta P_k-\Delta P_{k+1}) \nonumber \\
& \qquad  +2d \bar{\sigma}\big(\hat{P}_{k}^{-1/2}\Delta P_{k+1} \hat{P}_{k}^{-1/2}-P_{k+1}^{-1/2}\Delta P_{k+1} P_{k+1}^{-1/2}\big).  \label{eq:continuous2_b}
\end{align}
\end{subequations}
To see \eqref{eq:continuous2_a}, notice the following inequalities hold from \eqref{eq:deviation_delta}
a) $\bar{\sigma}(\hat{P}_k^{-1/2}\Delta P_k \hat{P}_k^{-1/2})
\leq \bar{\sigma}(\hat{P}_k^{-1/2})^2 \bar{\sigma}(\Delta P_k) \leq \frac{\delta}{\rho} < \frac{1}{4}$, b) $\bar{\sigma}(P_{k+1}^{-1/2}\Delta P_{k+1} P_{k+1}^{-1/2})
\leq \bar{\sigma}(P_{k+1}^{-1/2})^2 \bar{\sigma}(\Delta P_{k+1}) \leq \frac{\delta}{\rho} < \frac{1}{4}$, and c) $\bar{\sigma}\big((\|x'_{k+1}-x'_k\|-\|x_{k+1}-x_k\|) \hat{P}_k^{-1/2} W \hat{P}_k^{-1/2}\big) \leq \| \Delta x_{k+1}-\Delta x_k \| \bar{\sigma}(P_{k+1}^{-1/2})^2 \bar{\sigma}(W) \leq \frac{2\delta}{\rho}<\frac{1}{2}$. 
Therefore, we have $I\!+\!\hat{P}_k^{-1/2}\Delta P_k \hat{P}_k^{-1/2} \!+(\|x'_{k+1}\!-\!x'_k\|\!-\!\|x_{k+1}\!-\!x_k\|) \hat{P}_k^{-1/2} W \hat{P}_k^{-1/2} \succeq \frac{1}{4}I$ and $I+P_{k+1}^{-1/2}\Delta P_{k+1} P_{k+1}^{-1/2} \succeq \frac{1}{4}I$,
and thus Lemma~\ref{lem1} with $\kappa=\frac{1}{4}$ is applicable.
The last term in \eqref{eq:continuous2_b} is upper bounded as 
\begin{subequations}
\begin{align}
&2d  \bar{\sigma}\big( \hat{P}_k^{-1/2}\Delta P_{k+1} \hat{P}_k^{-1/2} - P_{k+1}^{-1/2} \Delta P_{k+1} P_{k+1}^{-1/2}\big) \nonumber \\
&\leq 2d  \cdot \frac{3\bar{\sigma}(\Delta P_{k+1})}{\sqrt{\rho}}
\bar{\sigma}\big(\hat{P}_k^{-1/2}-P_{k+1}^{-1/2} \big) \label{eq:continuous3_a} \\
&\leq 2d \cdot \frac{3\delta}{\sqrt{\rho}}\cdot \frac{1}{\rho}\bar{\sigma}\big(  \hat{P}_k^{1/2}-P_{k+1}^{1/2} \big)  \label{eq:continuous3_b} \\
&\leq 2d \cdot \frac{3\delta}{\sqrt{\rho}}\cdot \frac{1}{\rho} \cdot \frac{1}{2\sqrt{\rho}}\bar{\sigma}\left(  \hat{P}_k-P_{k+1} \right) \label{eq:continuous3_c}  \\
&= \frac{3\delta d }{\rho^2}\bar{\sigma}\left(P_k-P_{k+1}+\|x_{k+1}-x_k\|W\right) \nonumber \\
&\leq \frac{3\delta d}{\rho^2} \big\{ \bar{\sigma}(P_k-P_{k+1})+\|x_{k+1}-x_k\|\bar{\sigma}(W) \big\}. \nonumber
\end{align}
\end{subequations}
Lemmas \ref{lem2}, \ref{lem3} and \ref{lem4} were used in steps \eqref{eq:continuous3_a}, \eqref{eq:continuous3_b} and \eqref{eq:continuous3_c}.
Combining the results so far, we obtain an upper bound for $\big|\mathcal{D}(x'_k, x'_{k+1}, P'_k, P'_{k+1})-\mathcal{D}(x_k, x_{k+1}, P_k, P_{k+1})\big|$ as follows:
\begin{align}
&\big|\mathcal{D}(x'_k, x'_{k+1}, P'_k, P'_{k+1})-\mathcal{D}(x_k, x_{k+1}, P_k, P_{k+1}) \big| \nonumber \\
& \leq (1+\frac{2n}{\rho})\|\Delta x_{k+1}-\Delta x_k\|\bar{\sigma}(W)
+\frac{2n}{\rho}\bar{\sigma}(\Delta P_{k+1}-\Delta P_k)  \nonumber \\
& \qquad +\frac{3\delta n }{\rho^2} \big\{\bar{\sigma}(P_{k+1}-P_k) +\|x_{k+1}-x_k\|\bar{\sigma}(W) \big\}. \label{continuous4}
\end{align}
The result in this subsection is summarized as follows:
\begin{lemma}
\label{lem:main2}
Let $\delta$ and $\rho$ be any constants satisfying $0< \delta <\frac{\rho}{4}$.
Suppose that $(x_k, P_k)$, $(x_{k+1}, P_{k+1})$,  $(x'_k, P'_k)=(x_k+\Delta x_k, P_k+\Delta P_k)$ and $(x'_{k+1}, P'_{k+1})=(x_{k+1}+\Delta x_{k+1}, P_{k+1}+\Delta P_{k+1})$ are points in 
$\mathbb{R}^d\times \mathbb{S}_\rho^d$.
If the transitions from $(x_k, P_k)$ to $(x_{k+1}, P_{k+1})$ and from $(x'_k, P'_k)$ to $(x'_{k+1}, P'_{k+1})$ are both lossless, then there exists a positive constant $L_\rho$ such that 
\begin{align*}
&|\mathcal{D}(x'_k, x'_{k+1}, P'_k, P'_{k+1})-\mathcal{D}(x_k, x_{k+1}, P_k, P_{k+1})| \\
& \leq L_\rho \big[ \|\Delta x_{k+1}-\Delta x_k\|\bar{\sigma}(W) + \bar{\sigma}(\Delta P_{k+1}-\Delta P_k) \\
&\qquad + \delta \big\{ \|x_{k+1}-x_k\|\bar{\sigma}(W)+\bar{\sigma}(P_{k+1}-P_k)  \big\}
\big].
\end{align*}
\end{lemma}
\begin{proof}
The result follows from \eqref{continuous4} by setting $L_\rho=\max \left\{1+\frac{2d}{\rho},  \frac{3d}{\rho^2} \right\}$.
\end{proof}
\vspace{-0.6cm}
\subsection{Proof of Theorem~\ref{theo:continuity2}}
We prove that the choice $\delta = \frac{\epsilon}{L_\rho (1+|\gamma|_{\text{TV}})}$ suffices, where $L_\rho$ is defined in Lemma~\ref{lem:main2}.
Suppose $\gamma(t)=(x(t), P(t))$, $\gamma'(t)=(x'(t), P'(t))$ and both $\gamma$ and $\gamma'$ are finitely lossless with respect to a partition $\mathcal{P}_0$.
Let $\mathcal{P}=(0=t_0<t_1<\cdots < t_K=1)$ be any partition such that $\mathcal{P}\supseteq \mathcal{P}_0$. Then the following chain of inequalities holds:
\begin{subequations}
\label{eq:chain}
\begin{align}
&\big|c(\gamma'; \mathcal{P})-c(\gamma; \mathcal{P})\big| \nonumber \\
&\leq \sum_{k=0}^{K-1}\big|\mathcal{D}(x'(t_k), x'(t_{k+1}), P'(t_k), P'(t_{k+1})) \nonumber \\
&\qquad \qquad -\mathcal{D}(x(t_k), x(t_{k+1}), P(t_k), P(t_{k+1})) \big| \\
&\leq L_\rho  \sum_{k=0}^{K-1} \Big[
\|x'(t_{k+1})-x(t_{k+1})-x'(t_k)+x(t_k)\|\bar{\sigma}(W) \nonumber \\
&\quad +
\bar{\sigma} \big(P'(t_{k+1})-P(t_{k+1})-P'(t_k)+P(t_k)\big) \nonumber \\
&\quad + \delta \big\{ \|x(t_{k+1})\!-\!x(t_k)\| \bar{\sigma}(W)+\bar{\sigma}\big(P(t_{k+1})-P(t_k)\big) \big\}
\Big] \label{eq:15d} \\ \nonumber
&= L_\rho \big( V(\gamma'-\gamma; \mathcal{P}) + \delta V(\gamma; \mathcal{P})\big) \leq L_\rho \big(|\gamma'-\gamma|_{\text{TV}} + \delta |\gamma|_{\text{TV}} \big)\\ \nonumber
&\leq L_\rho (1 + |\gamma|_{\text{TV}}) \delta =\epsilon
\end{align}
\end{subequations}
The inequality \eqref{eq:15d} follows from Lemma~\ref{lem:main2}, noticing that both $\gamma$ and $\gamma'$ are finitely lossless with respect to $\mathcal{P}$.

Let $\{\mathcal{P}_i\}_{i\in\mathbb{N}}$ and $\{\mathcal{P}'_i\}_{i\in\mathbb{N}}$ be sequences of partitions such that $\mathcal{P}_i \supseteq \mathcal{P}_0$ and $\mathcal{P}'_i \supseteq \mathcal{P}_0$ for each  $i\in\mathbb{N}$, and
\begin{equation}
\label{eq:p_sequence}
\lim_{i \rightarrow \infty} c(\gamma; \mathcal{P}_i) =c(\gamma), \;\;
\lim_{i \rightarrow \infty} c(\gamma'; \mathcal{P}'_i) =c(\gamma').
\end{equation}
Let $\{\mathcal{P}''_i\}_{i\in\mathbb{N}}$ be the sequence of partitions such that for each $i\in\mathbb{N}$, $\mathcal{P}''_i$ is a common refinement of $\mathcal{P}_i$ and $\mathcal{P}'_i$. Since
$c(\gamma;\mathcal{P}_i) \leq c(\gamma;\mathcal{P}''_i)  \leq c(\gamma)$ and 
$c(\gamma';\mathcal{P}'_i) \leq c(\gamma';\mathcal{P}''_i)  \leq c(\gamma')$
hold for each $i\in\mathbb{N}$, \eqref{eq:p_sequence} implies $\lim_{i \rightarrow \infty} c(\gamma; \mathcal{P}''_i) =c(\gamma)$ and $\lim_{i \rightarrow \infty} c(\gamma'; \mathcal{P}''_i) =c(\gamma')$.
Now, since the chain of inequalities \eqref{eq:chain} holds for any partition $\mathcal{P}\supseteq \mathcal{P}_0$ and since $\mathcal{P}''_i \supseteq \mathcal{P}_0$ for each $i\in\mathbb{N}$, $\big| c(\gamma; \mathcal{P}''_i) - c(\gamma'; \mathcal{P}''_i) \big|   \leq \epsilon$
holds for all $i\in\mathbb{N}$. Therefore, we obtain
$|c(\gamma)-c(\gamma')|=\lim_{i\rightarrow \infty} |c(\gamma; \mathcal{P}''_i) - c(\gamma'; \mathcal{P}''_i)|   \leq \epsilon$.





\ifCLASSOPTIONcaptionsoff
  \newpage
\fi



%

\bibliographystyle{IEEEtran}
\bibliography{IEEEabrv,main.bib}
\end{document}